\documentclass{article}

\usepackage[main, final]{neurips_2026}

\usepackage[utf8]{inputenc} 
\usepackage[T1]{fontenc}    
\usepackage{url}            
\usepackage{booktabs}       
\usepackage{amsfonts}       
\usepackage{nicefrac}       
\usepackage{microtype}      
\usepackage{xcolor}     
\usepackage{array}

\usepackage{enumitem}
\usepackage{amsmath}
\usepackage{amsthm}
\usepackage{amssymb}
\usepackage{amsbsy}
\usepackage{mathtools}
\usepackage{graphicx}
\usepackage{color}
\usepackage{mdframed}
\usepackage{url}

\usepackage{algorithm,algorithmic}


\usepackage[pdfencoding=auto, psdextra, pagebackref=false,linktoc=page]{hyperref}
\hypersetup{
    colorlinks,
    linkcolor={blue!50!black},
    citecolor={blue!50!black},
    urlcolor={blue!80!black}
}

\usepackage{tcolorbox}

\usepackage{thmtools, thm-restate}
\declaretheorem[name=Theorem]{theorem}
\declaretheorem[name=Lemma]{lemma}
\declaretheorem[name=Proposition]{proposition}

\declaretheorem[name=Remark]{remark}
\declaretheorem[name=Assumption]{assumption}
\declaretheorem[name=Definition]{definition}

\declaretheorem[name=Information Setting]{protocol}

\usepackage[toc,page,header]{appendix}

\usepackage{enumitem}
\usepackage{notation}

\usepackage{tcolorbox}

\usepackage{xcolor}

\title{Online Learning in Stabilized Linear Dynamical Games with Adversarial Disturbances}

\author{%
  Anas Barakat\\
  SUTD\\
  \texttt{anas\_barakat@sutd.edu.sg} \\
  \And
  John Lazarsfeld \\
  Georgia Tech\\
  \texttt{jlazarsfeld@gatech.edu} \\
  \AND
  Georgios Piliouras\\
  SUTD\\
   \texttt{georgios@sutd.edu.sg} \\
  \And
  Antonios Varvitsiotis \\
  SUTD, CQT NUS, Archimedes\\
  \texttt{antonios@sutd.edu.sg} \\
}

\begin{document}

\maketitle

\begin{abstract}
We study online learning in linear dynamical games where multiple strategic agents act on a shared state evolving according to a linear dynamical system subject to adversarial disturbances. 
This setting lies beyond both single-agent nonstochastic online control and classical linear-quadratic games, which typically focus on quadratic objectives and noiseless or stochastic dynamics. Each agent seeks to minimize its own sequence of convex losses under full-state observability.
Following the stabilizing-baseline paradigm in nonstochastic online control, we assume access to linear controllers that stabilize the noiseless system and focus on online adaptation by learning disturbance-action corrections. This extends adversarial online control to strategic multi-agent shared-state systems, where each agent's actions shape the state trajectory and hence the realized losses of all learners. 
Under state-only and aggregate-input feedback models, we analyze agents running online gradient descent with memory to update their own disturbance-action policies. 
We prove per-agent regret bounds that are sublinear and near-optimal in the time horizon, with population dependence that varies under the specified feedback, opponent policy, and stability assumptions. 
In the common-interest case, where all agents have identical cost functions, we show that the induced online problem forms a time-varying potential game and derive equilibrium tracking guarantees. 
Together, our results provide a theoretical framework for adversarial online learning in stabilized linear dynamical games, connecting online control with learning in games.  
\end{abstract}

\section{Introduction}
\label{sec:intro}

Linear dynamical games are a classical model for strategic interaction in dynamical systems~\citep{basar-olsder98book}. 
In the linear-quadratic case, multiple agents act on a shared linear system and optimize individual quadratic costs; this setting has long served as a benchmark for equilibrium analysis in control and, more recently, for studying learning dynamics and policy-gradient methods in multi-agent reinforcement learning~\citep{mazumdar-et-al20,hambly-et-al23}. 
By contrast, little is known about linear dynamical games in the online nonstochastic regime, where multiple strategic agents must act causally on a shared state while both disturbances and costs may vary adversarially over time. Adversarial disturbances provide a worst-case model for misspecification, unmodeled dynamics, abrupt shocks, or strategic perturbations, without requiring distributional assumptions on the noise. This perspective is standard in nonstochastic online control \cite{hazan-singh22introduction} and is especially relevant in multi-agent systems, where the environment faced by each learner is shaped by both exogenous perturbations and the actions of other agents. 

We study an online learning variant of linear dynamical games. The shared state evolves according to:
\begin{equation}
x_{t+1} = A x_t + B_1 u^1_t + \dots + B_N u^N_t + w_t ,
\tag{LDS}
\end{equation}
where $N$ agents simultaneously choose controls~$(u_t^i)_{i \in \{1, \cdots, N\}}$, $A$ and $(B_i)_{i \in \{1, \cdots, N\}}$ are time-invariant transition matrices, and $w_t$ is an arbitrary bounded disturbance. 
We call this an online linear dynamical game: agents act on a common linear system, but each agent evaluates the resulting trajectory through its own time-varying convex cost.
At each time step $t$, each agent~$i \in \{1, \cdots, N\}$ observes the full state $x_t$, selects its own control input $u_t^i$ according to a causal policy~$\pi^i$ mapping the observation history to controls, and incurs an individual time-varying convex cost $c_t^i(x_t,u_t^i)$. 
Thus, the agents are coupled through the shared state: one agent's action affects the state trajectory, and hence the realized costs incurred by the other agents.
This shared-state coupling makes the online problem inherently nonstationary from each agent's perspective: exogenous disturbances perturb the dynamics, while the adaptive controls of other agents change the state trajectory on which costs are evaluated. 

Our formulation follows the stabilizing-baseline paradigm in nonstochastic online control~\citep{hazan-singh22introduction}. 
We separate stabilization from online adaptation: agents are given nominal controllers that stabilize the noiseless system ($w_t=0$), and the learning problem is to adapt around this stable backbone through disturbance-action corrections. 
As in the single-agent nonstochastic control literature, these stabilizing controllers are treated as part of the problem input; our focus is on the online adaptation problem induced by adversarial disturbances, changing costs, and shared-state coupling. 
We refer to this setting as a \emph{stabilized linear dynamical game}: the baseline closed-loop dynamics are stable, while agents learn online correction policies in response to the realized adversarial environment.

Relative to single-agent nonstochastic control, the main challenge is analyzing counterfactual performance under shared-state coupling. Each agent controls only one component of the joint input, but the state trajectory depends on all agents' controls.
Thus, even when evaluating unilateral regret against a fixed realized sequence of other agents' controls, the counterfactual trajectory induced by one agent's policy must be analyzed through the coupled dynamics. Different feedback models determine how much of this coupling each learner can account for, leading to different dependences on the number of agents. In the common-interest case, the same shared-state structure raises a collective question: whether individual online adaptation tracks equilibria of the induced time-varying game.

Thus, our setting sits between classical linear-quadratic games and single-agent nonstochastic online control. Unlike the former, we allow adversarial disturbances and time-varying convex costs; unlike the latter, multiple learners act on the same state. Consequently, even with a stabilizing baseline, it is unclear whether one can obtain individual regret guarantees or collective guarantees from simultaneous gradient updates under shared-state coupling. This motivates our central question:
\begin{quote}
\center{
\textit{
In stabilized linear dynamical games with adversarial disturbances, can agents update their own policies online so as to achieve sublinear individual regret and collective equilibrium guarantees in the common-interest setting? 
}
}
\end{quote}

\noindent\textbf{Contributions.}
We obtain the first regret and equilibrium-tracking guarantees for adversarial online learning in shared-state linear dynamical games.
Our contributions are as follows.
\begin{itemize}[leftmargin=*]
\item \textbf{Online learning in stabilized linear dynamical games.}
We formulate an online nonstochastic variant of linear dynamical games in which agents act on a shared linear dynamical system, face individual time-varying convex costs, and adapt through disturbance-action corrections on top of stabilizing baseline controllers (Section~\ref{sec:pb-formulation}). This framework extends single-agent nonstochastic online control to strategic shared-state systems.  

\item \textbf{Individual policy regret under state-only feedback.}
In the state-only feedback model, agents observe the shared state but not the aggregate control contribution of the other agents. We prove a per-agent policy regret bound of $\widetilde{\mathcal O}(N^2\sqrt{T})$ (Theorem~\ref{thm:local-gpc-regret-indep}) for online gradient descent (Algorithm~\ref{algo:MAGPC}) over disturbance-action policies. The proof reduces each agent's problem to a single-agent nonstochastic control problem by absorbing the other agents' realized inputs into an effective disturbance. The resulting dependence on $N$ arises from
bounding the other agents' aggregate contribution to the effective disturbance.

\item \textbf{Improved population scaling under aggregate-input feedback.}
 In the aggregate-input feedback model, agents additionally observe the combined contribution of the other agents' controls. Under compatible individual and joint strong stability, we improve the per-agent policy regret bound to $\widetilde{\mathcal O}(N\sqrt{T})$~(Theorem~\ref{thm:local-gpc-regret-cp+}). 
 With an additional Lipschitz assumption on the costs, we remove the polynomial dependence on $N$, obtaining a $\widetilde \calO(\sqrt{T})$ regret bound with only polylogarithmic dependence on $N$~(Theorem~\ref{thm:local-gpc-regret-cp+-lip}). Technically, the improvement comes from exploiting the structure of the joint input contributions in the shared dynamics, rather than treating the other agents' actions as an unstructured disturbance. Our analysis controls the induced state evolution under simultaneous DAC policies and balances the memory-truncation error with the online-gradient regret term to obtain the sharper dependence on $N$. 
 By embedding the classical online learning lower bound into our admissible cost and controller
 classes, we show that the horizon dependence of Theorems~\ref{thm:local-gpc-regret-indep}--\ref{thm:local-gpc-regret-cp+-lip} is optimal up to logarithmic factors
(Theorem~\ref{thm:reg-lb}).

\item\textbf{Equilibrium tracking in common-interest games.} When all agents share identical cost functions, we show that the induced linear dynamical game forms a time-varying potential game. We derive equilibrium-tracking guarantees in which the average equilibrium gap is controlled by the temporal variation of the costs and disturbances (Theorem~\ref{thm:time-var-potential}). The proof uses a dedicated time-varying potential argument tailored to our stateful dynamical setting, independent of our individual regret bounds. Our equilibrium gap guarantee extends prior work in time-varying (stateless) normal-form games to continuous convex games whose time-varying costs depend on state trajectories generated by adversarially disturbed dynamics.
\end{itemize}

\noindent\textbf{Related Work.}
We give a brief discussion of related works and defer
more details to Appendix~\ref{app:extended-related-works}.

\noindent\textbf{Online non-stochastic control.} 
Our work builds on a recent and growing line of research focusing on the use of online learning techniques to address control problems with adversarially perturbed dynamical systems \citep{hardt2018gradient,abbasi2011regret,agarwal2019online,hazan2020nonstochastic,foster2020logarithmic,simchowitz2020improper,simchowitz2020making,gradu-et-al20onsc-bandit-feedback,ghai-et-al23ons-modelfree-rl,cai-et-al24performative-control,tsiamis-et-al24icml,golowich-et-al24oc-population}. 
On the one hand, when the dynamical system \eqref{eq:state-lds} involves only a single agent~(i.e., $N=1$), our setting collapses to (single-agent) \textit{online non-stochastic control}. This problem has been thoroughly studied over the past years, see e.g. \cite{hazan-singh22introduction} and the references therein. 
On the other hand, most of the works in this line of research are devoted to the control of linear dynamical systems influenced by a \textit{single} controller. We discuss a few exceptions in the next section. 

\noindent\textbf{Linear dynamical games and multi-agent control.} There is extensive research at the interface of control and game theory, see e.g. \cite{marden-shamma18,chen-ren19survey-multi-agent-control} for surveys. 
An important body of this literature has focused on linear-quadratic games \citep{basar-olsder98book,mazalov-et-al17lq-potential-games,hosseinirad-et-al23lq-potential-games,zhang-et-al19zslq,bu-et-al19zslq,zhang-et-al21derivative-free,wu-et-al23zslq-cdc,uz-zaman-et-al24,mazumdar-et-al20,hambly-et-al23}.
Some of these works typically consider the same~\eqref{eq:state-lds} and assume quadratic costs for systems which are either deterministic ($w_t = 0$) or perturbed by a noise sequence~$\{w_t\}$ which is i.i.d. Gaussian. 
Classical approaches to design robust controllers in optimal control rely either on using probabilistic models for disturbances 
or adopting a (worst-case) ‘minimax’ perspective \citep{basar-bernhard08f-infinity}. 

A few recent works adopt an online learning approach for \textit{distributed} control: 
\cite{chang-shahrampour23distributed-online-control,chang-shahrampour2023distributed-online-LQR} study a distributed online control problem over a multi-agent network of $m$ identical linear systems,
where each agent seeks to compete with the best centralized control policy in hindsight.
This is fundamentally different from our setting, 
where we consider \textit{selfish strategic} agents influencing a
\textit{single} linear dynamical system, and where each 
agent attempts to minimize their own individual cost. 
\cite{ghai-et-al22regret-multi-agent-control} propose a reduction from any standard regret minimizing control method to a distributed algorithm implemented by several controllers,
which is distinct from our setting of multiple strategically competing agents. 
Recently, \cite{golowich-et-al24oc-population} proposed an online control approach for population dynamics where states are distributions in the simplex. 
We rather focus on the case of a finite and discrete large number of agents and discuss the influence of the total number of agents on individual regret. 

\noindent\textbf{Online learning with memory and in time-varying games.} Our regret analysis uses tools from online learning with memory \citep{anava-hazan-mannor15online-learning-memory,kumar-dean-kleinberg23oco-unbounded-memory}.  
Some of our results relate to the active research area of online learning in time-varying games \citep{cardoso-et-al19,duvocelle-et-al23time-varying-games,mertikopoulos-staudigl21eq-tracking,fiez-et-al21periodic-zero-sum,zhang-et-al22time-var-zero-sum,anagnostides-et-al23no-regret-time-varying,feng-et-al23last-it-tv-zs,yan-et-al23time-var-strongly-monotone,meng-et-al24ppp-online,taha-et-al24,fujimoto-et-al24time-var-auctions,fujimoto-et-al25,crippa-et-al25}. 
However, these works do not address our online multi-agent control setting where time-varying costs depend on an underlying \eqref{eq:state-lds} with coupled state dynamics subject to adversarial disturbances.  

\section{Problem Formulation: Online Learning in Linear Dynamical Games}
\label{sec:pb-formulation}

In this section, we formally introduce the online learning problem in linear dynamical games over a finite time horizon~$T$.
The state process evolves as a linear dynamical system: 
\begin{equation}
x_{t+1} = A x_t + \sum\nolimits_{i=1}^N B_i u_t^i + w_t, \quad t \ge 0\,,  
\tag{LDS}
\label{eq:state-lds}
\end{equation}
 where $x_t \in \mathcal X := \mathbb{R}^d$ is the state of the system initialized at a given (possibly random) state $x_0$, $u_t^i \in \mathcal U_i := \mathbb{R}^{k_i}$ is the strategy of agent $i \in [N] := \{1, \cdots, N\}$, $w_t \in \R^d$ is an arbitrary system disturbance and~$A \in \mathbb{R}^{d \times d}, B_i \in \mathbb{R}^{d \times k_i}$ are the transition matrices defining the linear dynamical system. We assume $\|x_0\|_2\le X_0$ almost surely for a finite constant $X_0$. All counterfactual trajectories start from the
 same realized initial state $x_0$. The bounds established in this paper may depend on~$X_0$.

\subsection{Online learning setting and feedback models}

We consider the following online setting: at each time step $t$, all $N$ agents observe the state~$x_t$ of the system. 
Then, each agent $i \in [N]$ selects a control input
$u_t^i \in \R^{k_i}$ and incurs a loss $c_t^i(x_t, u_t^i)$, 
where $c^i_t : \R^d \times \R^{k_i} \to \R$ 
is an adversarially chosen cost function. 
Finally, the system transitions to the next state according to the dynamics~\eqref{eq:state-lds}. The goal of each agent~$i$ is to minimize their own cumulative cost over $T$ rounds. 
We assume that each agent~$i \in [N]$ knows the dynamics $(A, B_i)$. For each $i \in [N]$, the cost function~$c_t^i$ is only accessible to agent~$i$. The perturbation sequence~$\{w_t\}$ is a priori unknown to agents. 
We study two \textit{information settings} that differ in how much of the other agents’ contribution to the shared dynamics is observable. 

\begin{protocol}[\textbf{State-only}]
\label{cp0}
At each time step~$t$, agent $i \in [N]$ observes the state $x_t$ (fully observable setting) and then their own induced cost. 
Agent $i$ has no access to the control inputs of other agents~$j \neq i\,.$ 
\end{protocol}

We also consider a second setting where agents have access to more information about the other interacting agents in the~\eqref{eq:state-lds} system at each time step.

\begin{protocol}[\textbf{Aggregate-Input}] 
\label{cp+}
At each time step~$t$, agent $i$ observes the state $x_t$, and then their own induced cost and the aggregated feedback~$\sum_{j \neq i} B_j u_t^j$ that encodes 
information about other agents' control inputs.  
Each agent~$i$ knows the total number of agents $N$. 
\end{protocol}
This stronger information setting 
is analogous to the standard setting of full-information feedback (hindsight observability) in the literature of online learning in games. This setting allows a player to evaluate their loss for any counterfactual action. Similarly, in our setting, observing the state and aggregated control lets each agent reconstruct the disturbance and thus compute their counterfactual loss for any alternative control they could have individually chosen, given others’ actions. 

\subsection{Regret framework for online learning in linear dynamical games} 
\label{sec:regret-framework-multiagent-oc}

In this section, we give a formal 
definition of our performance metric for online learning in linear dynamical games, inspired from both single-agent online control and learning in games. 

\noindent\textbf{Individual policy regret.}
Since the system dynamics depend on unknown costs and possibly adversarial perturbations, determining an optimal controller a priori is not possible in general. Therefore, in contrast to classical and robust optimal control, we consider \textit{regret} as a performance measure, following the recent line of works on (single-agent) online non-stochastic control \citep{hazan2020nonstochastic}.  
For each agent $i$, let $\Pi_i$ be a benchmark class of
deterministic causal policies on its observation history.
Writing $\mathsf h_t^i$ for the information available before
the round-$t$ control and $\xi_i$ for internal randomness,
the learner chooses
$u_t^i=\mathcal A_{i}(\mathsf h_t^i,\xi_i)$ using algorithm~$\mathcal{A}_i$. 
For any $T \geq H \geq 1$, 
we define the regret of agent~$i$ w.r.t. policy class~$\Pi_i$ 
when agent~$i$ runs algorithm~$\mathcal{A}_i$ and other agents use controls~$\{u_t^{-i}\}$ as: 
\begin{equation}
\reg^{H:T}_i(\mathcal{A}_i, \{u_t^{-i}\}, \Pi_i) = 
\sum\nolimits_{t=H}^T c_t^i(x_t, u_t^i) - \underset{\pi^i \in \Pi_i}{\min} \sum\nolimits_{t=H}^T c_t^i(x_t^{\pi^i}, u_t^{\pi^i}), 
 \label{eq:reg-def}
\end{equation}
where $x_t^{\pi^i}, u_t^{\pi^i}$ are the \textit{counterfactual state and controls} under the policy~$\pi^i$ for agent~$i$: 
\begin{equation*}
x_{t+1}^{\pi^i}=Ax_t^{\pi^i}+B_i u_t^{\pi^i}
  +\sum\nolimits_{j \neq i}B_j u_t^j+w_t\,, \quad 
x_0^{\pi^i} =x_0,\qquad
u_t^{\pi^i}=\pi^i(\mathsf h_t^{i,\pi^i}). 
\end{equation*}
Here $\mathsf h_t^{i,\pi^i}$ is the policy's counterfactual
observation history. The counterfactual state sequence corresponds to the state sequence that would be observed if agent $i$ were to unilaterally deviate to using policy $\pi^i$, instead of their online learning algorithm $\calA_i$ (and where all other agents stick to their online control input sequence).  
Note that when $N=1$, expression~\eqref{eq:reg-def} recovers the regret definition 
for single-agent online control.

In this work, we consider two natural policy comparator classes, which we introduce as follows:

\noindent\textbf{Comparator policy class~1: Strongly stable linear controllers ($\Pi_i^{\text{lin}}$).} 
For agent $i$, a \textit{linear controller} is defined by a matrix $K_i \in \mathbb{R}^{k_i \times d}$ s.t. $u_t^i = - K_i x_t$. 
We say that a linear policy~$K_i$ is \textit{stable} 
if~$\rho(A - B_iK_i) < 1$ (where~$\rho(\cdot)$ denotes the spectral radius),
in which case the closed-loop state-feedback linear dynamical system is 
globally asymptotically stable. 
\textit{Strong stability} of a controller is a quantitative version of stability which allows 
for deriving non-asymptotic guarantees.

\begin{definition}[Strong stability, e.g. \cite{cohen-et-al18online-lq-control}]
\label{def:kappa-gamma-st-stab}
A linear policy $K$ is $(\kappa,\gamma)$-strongly stable (for $\kappa > 0$ and $0 < \gamma < 1$) for a linear dynamical system specified by~$(A,B)$ if $\|K\| \leq \kappa$ and if there exists matrices~$L, Q$ s.t. $A - BK = Q L Q^{-1}$  with $\|L\| \leq 1-\gamma,$ and $\|Q\| \cdot \|Q^{-1}\| \leq \kappa\,.$ 
\end{definition}

Note that strong stability implies stability, and any stable policy is strongly-stable for some~$(\kappa, \gamma)$.
A natural policy comparator class is that of \textit{strongly stable linear controllers}
$\Pi_i^{\text{lin}}$, parameterized by: 
\begin{equation*}
\mathcal{K}_i := \left\{  K_i \in \mathbb{R}^{k_i \times d} :  K_i \,\, \text{is}\, (\kappa_i,\gamma_i)\text{-strongly stable}  \right\}, \, \text{for some}\, \kappa_i > 0, \gamma_i \in (0,1).
\end{equation*}

\noindent\textbf{Comparator policy class~2: Disturbance Action Controller (DAC) policies ($\Pi_i^{\text{DAC}}$).} The state sequence induced by a linear controller is not a linear function of its parameters. As a consequence, the induced cost is non-convex in the control parameters in general, even if the cost function is convex in both the state and the control input (see, e.g., \cite{fazel-et-al18pg-lqr}). Following prior work in single-agent control, we consider 
\textit{Disturbance Action Controller} (DAC) policies.
The system state induced by such policies is \textit{affine} in the policy parameters and 
one can invoke tools from online convex optimization 
when the cost functions are convex in the state and control input. 
For a sequence of perturbations~$\{w_t\}$, a DAC policy~$\pi^i(M_i,K_i)$ for agent $i \in [N]$
is then specified by learnable matrix parameters $M_i = [M_i^{[0]}, M_i^{[1]}, \cdots, M_i^{[H-1]}]$ 
for a memory length~$H \geq 1$, with a fixed given stabilizing controller~${K_i}$. 
The policy~$\pi^i(M_i,K_i)$ selects action~$u_t^{i}$ at a state~$x_t$ as: 
\begin{equation}
u_t^i = - K_i x_t + \sum\nolimits_{p=1}^H M_i^{[p-1]} w_{t-p}\,.
\tag{DAC-$i$}
\label{eq:dac-i}
\end{equation}
Note that for $p < 0$ we let $w_p = 0$, 
and moreover, the perturbations~${w_t}$ are not observed by the learners
but rather computed online using the structure of \eqref{eq:state-lds} 
and the state observations (we discuss these points later). 
The policy can thus be implemented in an online fashion by agent~$i$, 
and we henceforth use the notation $M_{i,t} = [M_{i,t}^{[p]}]_{0 \leq p \leq H-1}$ 
to reference the parameters of player~$i$ at time $t$\,. 
For a fixed~$H$ and stabilizing controller $K_i$, 
let $\mathcal{M}_i
= \big\{ M_i = \{M_i^{[0]}, \cdots, M_i^{[H-1]}\}: \|M_i^{[p]}\|_{\mathrm{op}} \leq 2 \kappa^2 (1-\gamma)^p, p = 0, \cdots, H-1\big\}
$
denote the set of all DAC policy parameters for agent $i$,
where $(\kappa,\gamma)=(\kappa_i,\gamma_i)$ under
Setting~\ref{cp0} and
$(\kappa,\gamma)=(\bar\kappa,\bar\gamma)$ under
Setting~\ref{cp+}. 

\begin{remark}
The stacked representation
$x_{t+1}=Ax_t+Bu_t+w_t$, with $B=[B_1,\ldots,B_N]$ and
$u_t=(u_t^1,\ldots,u_t^N)$, does not by itself reduce our
individual regret problem to centralized online control~\cite{agarwal2019online}.
Player $i$ optimizes its own cost, chooses only $u_t^i$, and uses
local model knowledge $(A,B_i)$ and the prescribed feedback.
Its comparator changes only its own policy while fixing opponents'
realized controls $u_{1:T}^{-i}$.
A centralized controller instead optimizes a specified objective over
joint policies, generally inducing different counterfactual trajectories.
Therefore, its joint regret guarantee does not automatically imply
our individual unilateral guarantees.
\end{remark}

\section{Individual Regret Guarantees}
\label{sec:regret}

In this section, we present our results on individual regret guarantees.
We analyze an online gradient descent with memory algorithm, where each agent independently updates its DAC policy parameters via online gradient descent (Algorithm~\ref{algo:MAGPC}). In the single-agent setting ($N=1$), this algorithm was introduced and analyzed by~\cite{agarwal2019online}. 
In our multi-agent setting,
the coupling of state dynamics across all agents induces new obstacles to implementing and analyzing this gradient-based approach. 
We elaborate first on the computational challenge. 

\noindent\textbf{Memory.} The cost~$c_t^i(x_t, u_t^i)$ incurred by agent~$i \in [N]$ at time step~$t$ depends on the state~$x_t$ of the system, which itself depends on all past states and control inputs from $t= 0\,.$ 
However, to run the online gradient descent subroutine of 
Algorithm~\ref{algo:MAGPC}, agent $i$ must be able to evaluate 
its cost function~$c_t^i$ on counterfactual state-action pairs. Unlike the single-agent case, counterfactual evaluation here depends not only on the agent’s own past controls but also on the entire joint sequence of other agents’ controls. 
This dependence breaks the straightforward counterfactual construction of the single-agent setting and requires a new memory-based approximation tailored to the multi-agent coupling. 

Fix a player $i$ and the realized controls of the other agents.
Set $A_i=A-B_iK_i$ and
$d_s=w_s+\sum_{j\ne i}B_ju_s^j$. The DAC disturbance sequence is $r_s=d_s$
under Setting~\ref{cp0} and $r_s=w_s$ under Setting~\ref{cp+}.
Both sequences are held fixed in unilateral counterfactuals,
and $d_s=r_s=0$ for $s<0$. 
For a fixed candidate parameter $Q\in\mathcal M_i$, the true
counterfactual trajectory is generated by executing the corresponding
DAC policy from time zero, starting from $x_0$, and therefore depends
on the entire past input sequence. Recomputing this trajectory from
time zero for each candidate would require work growing with $t$.
To obtain an efficiently evaluable loss with bounded memory, we
truncate the counterfactual rollout and initialize it at zero at a
recent truncation boundary. Individual strong stability of $A_i$
makes the influence of the state at that boundary decay geometrically
with the rollout length. Together with the state bounds established
below, this allows us to control the cumulative approximation error
using a memory length $H$ logarithmic in $T$. 
Specifically, let $h=H+1$. Starting from zero at time $t-h$,
the truncated state and action under $Q$ are given as follows: 
\begin{equation}
y_t^i(Q)
=\sum_{a=1}^{h}A_i^{a-1}
\left(d_{t-a}+B_i\sum_{p=1}^H Q^{[p-1]}r_{t-a-p}\right),
v_t^i(Q) =-K_i y_t^i(Q)+\sum_{p=1}^H Q^{[p-1]}r_{t-p}\,,
\label{eq:individual-diagonal-rollout}
\end{equation}
and $f_t^i(Q)=c_t^i(y_t^i(Q),v_t^i(Q))$.
These finite sums define the surrogate for every $t\ge1$,
including early rounds through zero padding, and involve an input
history of length $\mathcal O(H)$. The map $f_t^i$ is convex and
can be evaluated using the player's local model, its cost feedback,
and the reconstructed inputs.
Appendix~\ref{subsec:notation-counterfa-state-action} defines
the true unilateral counterfactual and the corresponding
memory loss $\ell_t^i(M_{i,t-h:t})$, whose diagonal is $f_t^i$.
The single-argument notation $\ell_t^i(Q)$ also denotes $f_t^i(Q)$.

\noindent\textbf{Algorithm variants.} 
Depending on the information setting (Settings~\ref{cp0} and~\ref{cp+}), we define two variants of Algorithm~\ref{algo:MAGPC}, each described from the perspective of a fixed agent $i \in [N]$. These variants capture different levels of feedback and are essential for obtaining our regret guarantees.

\begin{algorithm}[htb!]
\begin{algorithmic}[1]
\STATE \textbf{Input:} memory $H$, step size $\eta>0$,
  $M_{i,1}\in\mathcal M_i$, given stabilizing baseline linear controller $K_i$. 
  \STATE Observe $x_0$, apply $u_0^i=-K_i x_0$, and set
  $M_{i,s}=M_{i,1}$ for $s\le0$ and $r_s=d_s=0$ for $s<0$.
\FOR{$t=1,\ldots,T$}
  \STATE Observe $x_t$ and compute
    $d_{t-1}=x_t-Ax_{t-1}-B_i u_{t-1}^i$.
  \IF{Information Setting~\ref{cp0}}
    \STATE Set $r_{t-1}=d_{t-1}$.
  \ELSE
    \STATE Observe $\sum_{j\ne i}B_j u_{t-1}^j$ and set
      $r_{t-1}=w_{t-1}
      =d_{t-1}-\sum_{j\ne i}B_j u_{t-1}^j$.
  \ENDIF
  \STATE Apply
    $u_t^i=-K_i x_t+\sum_{p=1}^H M_{i,t}^{[p-1]}r_{t-p}$.
  \STATE Incur $c_t^i(x_t,u_t^i)$ and receive the
    post-action cost feedback.
  \STATE Construct $f_t^i$ using
    \eqref{eq:individual-diagonal-rollout} and update
    $M_{i,t+1}=\Pi_{\mathcal M_i}
    \bigl(M_{i,t}-\eta\nabla f_t^i(M_{i,t})\bigr)$.
\ENDFOR
\end{algorithmic}
\caption{Online Gradient Descent with Memory for agent $i \in [N]$}
\label{algo:MAGPC}
\end{algorithm}

\begin{remark}
Gradients and projections use the block Frobenius inner product,
with $\|M_i\|_F^2=\sum_{p=0}^{H-1}\|M_i^{[p]}\|_F^2$. 
Algorithm~\ref{algo:MAGPC} assumes that each agent can evaluate the gradients $\nabla \ell_t^i$ of their own cost function evaluated at counterfactual state-action pairs. 
\end{remark}

\noindent\textbf{Standing Assumptions.} 
Finally, before presenting our regret guarantees, we introduce our standing assumptions,
all standard in the recent literature on online control (see e.g., \cite{agarwal2019online,hazan-singh22introduction,hazan2020nonstochastic,gradu-et-al20onsc-bandit-feedback,minasyan-et-al21unknowm-time-var}). 

\begin{assumption}[Cost functions]
\label{as:convexity}
The following assumptions hold for every $i \in [N]$:
\begin{enumerate}[
    label=(\roman*), 
    leftmargin=*, 
    topsep=0em, 
    partopsep=0pt, 
    itemsep=0em]
\item The cost function $c_t^i: \mathcal{X} \times \mathcal{U}_i \to \mathbb{R}$ is jointly convex and differentiable. 
\item \label{as:grad-bound} There exist $\beta, G>0$ s.t. for any~$D> 0$ and every $(x,u^i) \in  \mathcal{X} \times \mathcal{U}_i$ s.t. $\|x\| \leq D, \|u^i\| \leq D,$ we have $|c_t^i(x,u)| \leq \beta D^2$ and $\|\nabla_x c_t^i(x,u^i)\|, \| \nabla_u c_t^i(x,u^i)\| \leq G D\,.$  
\end{enumerate} 
\end{assumption} 

\begin{lemma}
\label{lem:convexity-loss-i}
Under Assumption~\ref{as:convexity}, the loss function~$\ell_t^i$ is convex w.r.t. $M_{i}$ for all $i \in [N].$
\end{lemma}

\begin{assumption}[Bounded disturbances]
\label{as:bounded-disturb}
There exists $W > 0$ s.t. for all $t \geq 0$, $\|w_t\| \leq W\,.$\footnote{This assumption is standard in the online control literature. See \cite{ito-tsuchiya24} for recent work relaxing it.} 
\end{assumption}

\subsection{Information setting~\ref{cp0}: State-only Feedback}
\label{sec:agent-wise-regret-analysis}

Under Information Setting~\ref{cp0}, agents do not have access to other agents' control inputs. 
However, from the viewpoint of a given agent~$i$, we observe that \eqref{eq:state-lds} can be 
rewritten: 
\begin{equation}
x_{t+1} = A x_t + B_i u_t^i + \widetilde w_t\,, \quad \widetilde w_t = \sum\nolimits_{j \neq i} B_j u_t^j + w_t\,. 
\label{eq:lds-cp0}
\end{equation}
In this view, in Algorithm~\ref{algo:MAGPC}, we naturally propose 
that agent $i$ executes a \eqref{eq:dac-i} policy with 
disturbance sequence~$\widetilde w_t$. 
Given expression~\eqref{eq:lds-cp0}, note that $\widetilde w_t$ (unlike $w_t$)
can be calculated by agent $i$ at each time step since 
$\widetilde w_t = x_{t+1} - A x_t - B_i u_t^i$,
and this computation only involves information observed
under the information setting
(state observations and the agent's control input).
Under this strategy, agent~$i$ thus faces a 
LDS \eqref{eq:lds-cp0} controlled by its own control inputs,
and for this we make a strong stability assumption adapted to the multi-agent setting. 
\begin{assumption}[Agent-wise strong stability]
\label{as:strong-stab}
Each learner~$i \in [N]$ knows a $(\kappa_i,\gamma_i)$-strongly stable linear controller~$K_i$ for the LDS specified by~$(A,B_i)$.
\end{assumption}

Under this assumption, we present our first regret bound. 
\begin{tcolorbox}[colframe=white!, top=2pt,left=2pt,right=2pt,bottom=2pt,before skip=1.5pt,  after skip=1.5pt] 
\begin{restatable}
[\textbf{Individual regret under state-only feedback}]{theorem}{thmlocal}
\label{thm:local-gpc-regret-indep}
Let Assumptions~\ref{as:convexity}, \ref{as:bounded-disturb} and~\ref{as:strong-stab} hold. 
Suppose there exists~$U \geq 1$ s.t. for all $t \geq 0, j \in [N], \|u_t^j\| \leq U$. 
If agent~$i \in [N]$ runs Algorithm~\ref{algo:MAGPC} under Setting~\ref{cp0} with \eqref{eq:dac-i} policy on perturbation sequence~$\{\widetilde w_t\}$ and step size $\eta = \Theta(1/(G \widetilde W^2 \sqrt{T}))$, 
where $\widetilde W = W + (N-1) U(\max_{j} \|B_j\|)$, 
and with $H=\max\{1,\lceil\log(\kappa_iT)/\gamma_i\rceil\}$,  
then for~$T \geq H+1$, 
$\reg^{H+1:T}_i(\mathcal{A}_i, \{u_t^{-i}\}, \Pi_i^{\text{lin}}) = \widetilde \calO(U^2N^2 \sqrt{T}).$\footnotemark[1]
\end{restatable}   
\end{tcolorbox}
\footnotetext[1]{For readability, here and throughout, we use~$\widetilde \calO$ to 
    hide polynomial factors
    in problem parameters 
    and (poly)logarithmic factors in $T$ and $N$.
    We state the exact dependencies in the proofs.}

The proof of Theorem~\ref{thm:local-gpc-regret-indep}
consists of applying the single-agent regret guarantee for 
gradient perturbation controllers~\citep[Theorem 5.1]{agarwal2019online} for each agent~$i$ on the new perturbation sequence~$\{\widetilde w_t\}$ in \eqref{eq:lds-cp0},
and we give the full details in Section~\ref{sec:proof-thm-local-gpc-regret-indep}.
While the theorem highlights the robustness of gradient perturbation controllers to adversarial disturbances in this setting,
the regret bound grows quadratically with both the number of agents~$N$ and the magnitude~$U$ of the control inputs. This dependence arises from bounding the other
agents' aggregate input as part of the effective
disturbance. The resulting upper bound deteriorates
with $N$ and $U$. Whether its dependence on $N$
is unavoidable remains open.

\begin{remark}
The boundedness of control inputs in Theorem~\ref{thm:local-gpc-regret-indep} is a restrictive condition on the realized opponent inputs. Individual stability alone does not imply it, see Section~\ref{apdx:discussion-assumptions} for further discussion. 
\end{remark}

\subsection{Information setting~\ref{cp+}: Aggregate-Input Feedback}  

The regret bound in Theorem~\ref{thm:local-gpc-regret-indep} scales quadratically
with the number of agents. We now consider Information Setting~\ref{cp+} and analyze the case in
which all agents use DAC policies. Under a global strong stability assumption, we obtain an
individual regret bound with improved dependence on $N$. 

\begin{assumption}[Global strong stability]
\label{as:strong-stab-global} 
Each learner~$i \in [N]$ knows a linear controller~$K_i$ such that $(K_1, \cdots, K_N)^T$ is $(\bar{\kappa},\bar{\gamma})$-strongly stable for the LDS~$(A, [B_1, \cdots, B_N])$.
\end{assumption}

Assumption~\ref{as:strong-stab-global} is a global assumption which is relevant when each agent~$i$ executes a \eqref{eq:dac-i} policy (with  matrix $K_i$). Indeed, the system state evolution of~\eqref{eq:state-lds} in the absence of disturbances, and when all players use linear controllers, can be written as $x_{t+1} = A x_t - [B_1, \cdots, B_N](K_1, \cdots, K_N)^T x_t$. Without global stability, if the spectral radius of $A - \sum_{i=1}^N B_i K_i$ is not smaller than 1, the states can become unbounded. 
Each agent~$i$ has access to the global parameters~$\bar{\kappa}, \bar{\gamma}$ which can be precomputed before each agent runs their own Algorithm~\ref{algo:MAGPC}. Recall that the matrices~$K_i$ are not learning parameters and are precomputed even in the independent learning setting.
Only the matrix parameters~$M_i$ of \eqref{eq:dac-i} policies are learned by the agents. 

Under Setting~\ref{cp+}, all agents can compute the original disturbance~$w_t$ at each time step (instead of~$\tilde{w}_t$ as in Theorem~\ref{thm:local-gpc-regret-indep}). 
However, at every timestep $t$, each agent updates their own policy parameters without access to other agent’s policy parameters at that round. After acting, each agent first incurs the loss according to their individual cost function, and \textit{then} observes the aggregated feedback. This feedback is used to inform their next policy update. 

Our next result shows that when agent~$i$ runs Algorithm~\ref{algo:MAGPC} with (a) a conservative stepsize scaled by $N$ and (b) a larger memory which depends logarithmically on~$N$ (compared to Theorem~\ref{thm:local-gpc-regret-indep}), they guarantee a regret 
w.r.t. the DAC policy class scaling only \textit{linearly} in $N$. 
This result is also robust to other agents' strategies (as they can execute arbitrary \eqref{eq:dac-i} policies). 
Theorem~\ref{thm:local-gpc-regret-cp+} compares with memory-$H$ DAC policies as other players also use DAC policies to be robust to the disturbance sequence $(w_t)$, whereas Theorem~\ref{thm:local-gpc-regret-indep} compares with strongly stable linear controllers. 
\begin{tcolorbox}[colframe=white!, top=2pt,left=2pt,right=2pt,bottom=2pt]
\begin{restatable}[\textbf{Individual regret under aggregate-input feedback}]{theorem}{thmglobal}
\label{thm:local-gpc-regret-cp+}
Let Assumptions~\ref{as:convexity} to \ref{as:strong-stab-global} hold. Then if agent~$i \in [N]$ runs Algorithm~\ref{algo:MAGPC} under Setting~\ref{cp+} with a \eqref{eq:dac-i} policy on perturbation sequence~$\{w_t\}$, step size $\eta = 1/(N\sqrt{T})$,
and with~$H=\lceil\log(2\kappa_*N^2\sqrt T)/\gamma_*\rceil$ where $\kappa_*,\gamma_*$ are the stability bounds in~\eqref{eq:stability-envelopes}, and when \textbf{all} other agents use \eqref{eq:dac-i} policies with perturbation sequence $(w_t)$ and feasible parameters
$M_{j,t}\in\mathcal M_j$, then for any~$T \geq H+1$, 
\begin{equation*}
\reg^{H+1:T}_i(\mathcal{A}_i, \{u_t^{-i}\}, \Pi_i^{\text{DAC}}) 
= \widetilde \calO(N\sqrt{T}).
\end{equation*}
\end{restatable}
\end{tcolorbox}

\noindent\textbf{Proof Overview.} 
Our analysis relies on a regret decomposition with two main terms: a counterfactual state-control error due to the use of a loss with limited memory~$H$, and a regret term induced by  the online gradient descent 
component of Algorithm~\ref{algo:MAGPC}. 
The main technical challenges we overcome are two-fold: first,
states may grow unbounded with an undesirable scaling in $N$, and thus
we control their magnitude by studying the state evolution 
when all agents use DAC policies (using Assumption~\ref{as:strong-stab-global}),
and while tracking the dependence on~$N$. 
Second, we control both terms of the regret decomposition by selecting $H$ and $\eta$ to balance the terms in our upper bound. 
See Appendix~\ref{sec:proof-thm-local-gpc-regret-cp+}.

We remark that the linear dependence on $N$ in the regret bound is enabled by global stability (Assumption~\ref{as:strong-stab-global}). By contrast, if only individual stability (Assumption~\ref{as:strong-stab}) holds, even when agents can access aggregated input information, the dependence on $N$ deteriorates (see Appendix~\ref{appx:as-strong-stab-global} for a discussion). 
Moreover, under a stronger assumption on the cost functions (compared to Assumption~\ref{as:convexity}-\ref{as:grad-bound}), we further prove a sublinear regret for agent~$i$ 
that scales only \textit{polylogarithmically} in $N$. 

\begin{assumption}[Lipschitz costs]
\label{as:lip-costs}
There exists~$\bar{L} > 0$ s.t. for any agent $i \in [N]$ and for all state-control pairs $(x,u^i), (\tilde{x},\tilde{u}^i) \in  \mathcal{X} \times \mathcal{U}_i\,,$
$|c_t^i(x,u^i) - c_t^i(\tilde{x},\tilde{u}^i)| \leq \bar{L} (\|x - \tilde{x}\| + \|u^i - \tilde{u}^i\|)$. 
\end{assumption}
Note that the Lipschitz constant does not scale with the state and control input magnitude. 
Under this assumption, we obtain an improved regret guarantee
(proven in Appendix~\ref{sec:proof-thm-local-gpc-regret-cp+-lip}):
\begin{tcolorbox}[colframe=white!, top=2pt,left=2pt,right=2pt,bottom=2pt]
\begin{restatable}{theorem}{thmgloballip}
\label{thm:local-gpc-regret-cp+-lip}
Under the setting of Theorem~\ref{thm:local-gpc-regret-cp+}, replace gradient boundedness in Assumption~\ref{as:convexity}~-\ref{as:grad-bound} by Assumption~\ref{as:lip-costs}. 
Set instead $\eta = 1/\sqrt{T}$ and choose~$H=\lceil\log(2\kappa_*N\sqrt T)/\gamma_*\rceil$ with $\kappa_*,\gamma_*$ as in~\eqref{eq:stability-envelopes}. Then we have for any~$T \geq H+1$: $\reg^{H+1:T}_i(\mathcal{A}_i, \{u_t^{-i}\}, \Pi_i^{\text{DAC}}) 
= \widetilde{\mathcal{O}}(\sqrt{T})$.
\end{restatable}
\end{tcolorbox}
Note that using Assumption~\ref{as:lip-costs} in Theorem~\ref{thm:local-gpc-regret-indep} does not result in the same improved dependence on~$N$ as the regret will still scale with the magnitude of the modified disturbance $\widetilde{w}_t$, which is of order $N$. 

\subsection{Regret lower bounds}

We next show that the $\sqrt{T}$ horizon dependence of our regret
guarantees is optimal up to logarithmic factors. The
following lower bound applies under both information
settings, against both strongly stable linear controllers
and fixed-baseline DAC policies, even when the costs
satisfy Assumption~\ref{as:lip-costs}. 
For a nonempty interval $I\subseteq\{1,\ldots,T\}$, let
$\operatorname{Reg}_i^I$ be the regret in~\eqref{eq:reg-def} with sums taken over $t\in I$.

\begin{tcolorbox}[colframe=white!, top=2pt,left=2pt,right=2pt,bottom=2pt]
\begin{restatable}[\textbf{Horizon lower bound for linear and DAC comparators}]{theorem}{thmreglb}
    \label{thm:reg-lb}
Let $N\ge1$ and $i\in[N]$.
There exist scalar dynamics, deterministic disturbances, and supplied
baseline controllers satisfying Assumptions~\ref{as:bounded-disturb}--\ref{as:strong-stab-global} with
$W=1, 
\kappa_j=\bar\kappa=1,
\gamma_j=\bar\gamma=\tfrac12,$ 
such that the following holds.
For every $T\ge1$, $H\ge1$, prescribed nonempty interval
$I\subseteq\{1,\ldots,T\}$, and causal, possibly randomized learner
$\mathcal A_i$ under either information setting, there exists a
deterministic oblivious cost sequence satisfying Assumption~\ref{as:convexity} with
$\beta=G=4$ and Assumption~\ref{as:lip-costs} with $\bar L=1$ for which
$
\mathbb E_{\mathcal A_i}[\reg^{I}_i(\mathcal{A}_i, \{u_t^{-i}\}, \Pi)]
\ge \frac{\sqrt{|I|}}{2\sqrt3}
$
simultaneously for the $(1,\tfrac12)$-strongly stable linear class
$\Pi = \Pi_i^{\mathrm{lin}}$ and the memory-$H$ DAC class
$\Pi= \Pi_i^{\mathrm{DAC}}$ with fixed supplied baseline $K_i=0$.
The conclusion holds for every sequence of opponents' controls.
\end{restatable}
\end{tcolorbox}

\begin{remark}
The expectation in Theorem~\ref{thm:reg-lb} is over the
learner's internal randomization. The dynamics,
disturbances, and supplied baseline controllers are
independent of $T,H,I,\mathcal A_i$. For each fixed
$T,H,I,\mathcal A_i$, the cost sequence is selected before
interaction. It may depend on the learner's decision
rule and randomization law, but not on its realized
random seed. The same cost sequence establishes the
bound for both comparator classes.
\end{remark}

\noindent\textbf{Proof idea and relation to online learning.}
The proof embeds the classical random sign lower bound
for online learning \cite{arora2012multiplicative,hazan-et-al16intro-oco}
into an admissible scalar instance of our model.
We choose $A=B_j=0$ for every player $j$, $w_t=1$,
and $x_0=0$, so the learner and every comparator have
the same state $x_t=1$ for all $t\ge1$.
At this common state, the constructed costs satisfy
$c_t^i(1,u)=3/4+\sigma_t u$ for $|u|\le1/2$,
where the signs $\sigma_t$ are independent and uniform.
The proof verifies that these losses extend to jointly
convex, continuously differentiable costs satisfying
Assumptions~\ref{as:convexity} and~\ref{as:lip-costs} globally, 
and that the comparison actions $u=\pm1/2$ are realized by both admissible
linear controllers and memory-$H$ DAC policies with
the fixed supplied baseline $K_i=0$.
Convexity also gives the required conditional
expected-cost bound for arbitrary real learner
controls. These verifications establish the embedding into
our cost and policy classes while the statistical
lower bound mechanism is the classical random sign
argument. We provide a complete and self-contained proof in Appendix~\ref{app:lower-bound}.

\noindent\textbf{Implications for the upper bounds.}
Taking $I=\{1,\ldots,T\}$ gives an
$\Omega(\sqrt{T})$ full-horizon lower bound.
For the evaluation interval in Theorems~\ref{thm:local-gpc-regret-indep}--\ref{thm:local-gpc-regret-cp+-lip}, 
taking $I=\{H+1,\ldots,T\}$ gives
$\sqrt{T-H}/(2\sqrt{3})$, which is
$\Omega(\sqrt{T})$ whenever $H\le T/2$.
Thus, with the remaining problem parameters fixed
and logarithmic memory, our upper bounds have
optimal horizon dependence up to logarithmic
factors. The constructed costs also satisfy
Assumption~\ref{as:lip-costs}, so this conclusion covers
Theorem~\ref{thm:local-gpc-regret-cp+-lip}. The lower bound already holds for a
single agent and does not establish the necessity
of the factors depending on $N$ in our upper bounds.

\section{Equilibrium Tracking in the Common Interest Setting}
\label{sec:time-var-pot-game}

In the previous section, we developed individual regret guarantees
when other agents execute linear or DAC policies with possibly misaligned
or adversarially-chosen cost functions. 
In this section, we focus on the \textit{common interest setting}, 
where all agents incur the same state-only cost
\footnote{Our analysis also extends to a common joint cost
$C_t(x,u^1,\ldots,u^N)$ under analogous joint convexity, growth,
and smoothness assumptions and access to the corresponding
coordinate gradients. Using a joint state--action surrogate
gives an analogue of Theorem~\ref{thm:time-var-potential},
with modified constants and cost-variation budget.
We restrict to state-only costs for simplicity.}:
$c_t^i(x,u_i)=c_t(x)$ for every $i \in [N], x \in \mathcal{X}, u_i \in \mathcal{U}_i$. 
We retain the bounded disturbance and
joint strong-stability assumptions, and the convexity and growth conditions
of Assumption~\ref{as:convexity}, specialized to $c_t(x)$. 

Fix a memory length $H\ge1$ and write
$\bar A=A-\sum_{i=1}^N B_iK_i$. Set $w_s=0$ for $s<0$.
Let $\mathcal M=\prod_{i=1}^N\mathcal M_i$ be the product of the DAC
parameter domains. For a fixed joint profile $M=(M_1,\ldots,M_N)$, define
the zero-initialized truncated rollout: 
\begin{equation}
\label{eqt:rollout}
y_t(M)=\sum_{r=1}^{H+1}\bar A^{r-1}
\left(w_{t-r}+\sum_{i=1}^N B_i\sum_{p=1}^{H}
M_i^{[p-1]}w_{t-r-p}\right),
F_t(M):=\bar\ell_t(M):=c_t(y_t(M)).
\end{equation}

Since the cost functions are time-varying (not only via the strategies of the different players), our linear dynamical game can be seen as a \textit{time-varying game}. 
There have been considerable efforts endeavoring to extend the scope of traditional game-theoretic results to the time-varying setting and this is an active research area (see related work in Section~\ref{sec:intro}). 
In particular, our results in this section are inspired from recent developments for time-varying, normal-form, finite potential games in \cite{anagnostides-et-al23no-regret-time-varying}. In such games, agents participate in a potential game at each time step. 
We observe that a common interest linear dynamical game can be seen as a stateful, time-varying potential \textit{continuous} convex game where costs are functions of states driven by an underlying~\eqref{eq:state-lds} influenced by multiple agents. At each time step, the utility of each player is given by their cost function, and their strategy is defined by their DAC policy parameters. 

Since our setting involves adversarial, time-varying costs depending on state dynamics influenced by adversarial (time-varying) disturbances, convergence to (static) Nash equilibria is irrelevant in general. Nevertheless, we establish equilibrium gap \textit{tracking} guarantees for our dynamic setting.
To state our result, we introduce notations for time-varying best responses and equilibrium gaps:
\begin{equation}
\label{eqt:gap}
g_{i,t}(M) :=F_t(M)-\min_{Q_i\in\mathcal M_i}F_t(Q_i,M_{-i}),
\qquad
g_t(M) := \max_{i\in[N]}g_{i,t}(M).
\end{equation}
After choosing its round-$t$ action, each player $i$ has access to $\nabla_iF_t(M_t)$ and performs the update:
\begin{equation}
\label{eqt:update}
M_{i,t+1}=\Pi_{\mathcal M_i}
\left(M_{i,t}-L^{-1}\nabla_iF_t(M_t)\right)\,,
\end{equation}
where $L>0$ is the uniform smoothness bound for $F_t$
specified in Appendix~\ref{sec:proof-thm-time-var}. 
Projections and parameter gradients use the block Frobenius inner product.
All players use the same current profile $M_t$ in this update. 
We additionally assume access to
coordinate gradients in this section.
This access is guaranteed for instance under access to the common cost function, the
relevant joint closed-loop response, the disturbance history, and the
current joint DAC profile. 

We now make a smoothness assumption on the common cost function $c_t$ that is standard in the analysis of gradient methods in both optimization and learning in games. 

\begin{assumption}[Smoothness of the common state cost]
\label{as:smoothness}
There exists $\zeta\ge0$ such that, for every $t$ and all $x,x'\in\mathbb R^d$,
$\|\nabla c_t(x)-\nabla c_t(x')\|_2
\le\zeta\|x-x'\|_2.$ 
\end{assumption}

Under these assumptions, when all agents use the updates~\eqref{eqt:update}, 
we bound the average equilibrium gap by the variation of both cost functions and disturbances 
defined as follows for any $T \geq 1$: 
\begin{equation}
\label{eqt:variation}
V_c(T) :=\sum_{t=1}^{T-1}
\sup_{\|x\|_2\le D_x}|c_{t+1}(x)-c_t(x)|\,,\quad
V_w^+(T) :=\|w_0\|_2+
\sum_{s=0}^{T-2}\|w_{s+1}-w_s\|_2\,,
\end{equation}
where $D_x$ uniformly bounds $\|y_t(M)\|_2$ and $G_x$
uniformly bounds $\|\nabla c_t(x)\|_2$ on $\|x\|_2\le D_x$. 
An empty sum is equal to zero.

\begin{tcolorbox}[colframe=white!, top=2pt,left=2pt,right=2pt,bottom=2pt]
\begin{restatable}[\textbf{Equilibrium tracking}]{theorem}{thmtimevarpotential}
\label{thm:time-var-potential}
Let Assumptions~\ref{as:convexity}, \ref{as:bounded-disturb}, and~\ref{as:strong-stab-global} hold for the common-state-cost model specified above, together with the stated feedback and coordinate-gradient
access and Assumption~\ref{as:smoothness}. 
Then, for every $M_1\in\mathcal M$, the simultaneous updates~\eqref{eqt:update}
satisfy, for every $T\ge1$,
\begin{equation}
\label{eqt:main-bound}
\frac1T\sum_{t=1}^{T}g_t(M_t)^2
\le \frac{2C_M}{LT}
\left[F_1(M_1) +V_c(T)+G_xQ\,V_w^+(T)\right]\,,
\end{equation}
where the constants $D_x,G_x,Q,L,C_M$, independent of $T$ for fixed $H$, are given in Appendix~\ref{sec:proof-thm-time-var}. 
\end{restatable}
\end{tcolorbox}

For fixed memory and fixed problem parameters, sublinear cost and
disturbance variation imply a vanishing average squared equilibrium gap.
If both variation budgets are uniformly bounded, the average squared gap
is $O(T^{-1})$ and some iterate has gap $O(T^{-1/2})$. 
Notice that these guarantees concern the diagonal truncated DAC stage games defined in \eqref{eqt:rollout}-\eqref{eqt:gap}. 

\noindent\textbf{Proof overview.}
To prove the theorem, we use a time-varying potential-based argument tailored to our dynamic stateful setting, independent from our individual regret guarantees. 
We extend the approach of \cite{anagnostides-et-al23no-regret-time-varying} for time-varying \textit{(finite) normal-form} potential games to (a) cover \textit{(continuous) convex games} and (b) account for state dynamics and adversarial disturbances in addition to the time-varying costs in our linear dynamic game setting. See Appendix~\ref{sec:proof-thm-time-var} for an overview and a detailed proof. 

\section{Conclusion and Future Work}
\label{sec:discussion}

This work initiates and makes progress on online learning in linear dynamic games in strategic environments subject to adversarial disturbances, taking a first step toward bridging online  control with learning in games. 
In particular, we proved individual regret guarantees and equilibrium tracking
guarantees for the induced truncated DAC stage games for online learning with adversarial disturbances and time-varying costs. 
Our results also open several directions for future research. 
On the technical side, it is interesting to investigate whether tighter regret bounds can be obtained with respect to the number of agents or under structural assumptions such as time-invariant costs.
On the modeling side, important challenges include extending our analysis to settings with a network of communicating agents with local state observations, unknown or time-varying dynamics \citep{hazan2020nonstochastic,minasyan-et-al21unknowm-time-var,gradu-et-al23time-varying}
and to feedback-limited regimes \citep{yan-et-al23jmlr-onsc-partial-feedback},
where learners can only access partially observed states and partially informed bandit costs. 
A broader challenge is to go beyond \textit{linear} state dynamics. 

\begin{ack}

We thank the anonymous reviewers for their valuable feedback, which helped improve this manuscript. The initial version of this work was completed without generative AI assistance. In preparing the revised version, we used ChatGPT to assist with improving the presentation of the results, revising the proofs, and conducting the simulations reported in the appendix. The authors take full responsibility for the final content.

This work is supported by the MOE Tier 2 Grant (MOE-T2EP20223-0018),  the CQT++ Core Research Funding Grant (SUTD) (RS-NRCQT-00002), the National Research Foundation Singapore and DSO National Laboratories under the AI Singapore Programme (Award Number: AISG2-RP-2020-016), and partially by Project MIS 5154714 of the National Recovery and Resilience Plan, Greece 2.0, funded by the European Union under the NextGenerationEU Program. 
\end{ack}

\bibliographystyle{plain}
\bibliography{references}


\newpage
\appendix

\tableofcontents

\newpage

\section{Extended Related Work Discussion}
\label{app:extended-related-works}

\noindent\textbf{Online non-stochastic control.} 
Our work builds on a recent and growing line of research focusing on the use of online learning techniques to address control problems with adversarially perturbed dynamical systems \citep{hardt2018gradient,abbasi2011regret,agarwal2019online,hazan2020nonstochastic,foster2020logarithmic,simchowitz2020improper,simchowitz2020making,gradu-et-al20onsc-bandit-feedback,ghai-et-al23ons-modelfree-rl,martin-et-al24regret-oc}.
We refer the reader to a nice introduction to the topic in the recent monograph of \cite{hazan-singh22introduction} and the references therein for a survey. 
Recent follow-up works include studies on dynamic regret for online tracking \citep{tsiamis-et-al24icml}, performative control \citep{cai-et-al24performative-control}, online control in population dynamics \citep{golowich-et-al24oc-population,lu-et-al25pop-control-partial-obs}, simultaneous system identification and MPC with regret guarantees \citep{zhou-tzoumas24system-id-mpc}, online RL \citep{muehlebach-et-al25online-rl,ghai-et-al23ons-modelfree-rl},  partial feedback settings \citep{yan-et-al23jmlr-onsc-partial-feedback} and bandit settings \cite{sun-et-al24bandit-control} to name a few. 
Most of the works in this line of research are devoted to the control of linear dynamical systems influenced by a \textit{single} controller. We discuss a few exceptions in the next section.

\noindent\textbf{Multi-agent control.} The interface between game theory and control has given rise to a large body of work over the last decades to study settings involving multiple interacting controllers, see e.g.\cite{marden-shamma15gt-distrib-control,marden-shamma18,chen-ren19survey-multi-agent-control} for relevant surveys. Within the game-theoretic control literature, linear-quadratic (LQ) games is one of the canonical benchmark problems which has been studied in a variety of settings including LQ differential games \citep[Chap.~6]{basar-olsder98book}, LQ potential games \citep{mazalov-et-al17lq-potential-games,hosseinirad-et-al23lq-potential-games}, zero-sum LQ games \citep{zhang-et-al19zslq,bu-et-al19zslq,zhang-et-al21derivative-free,wu-et-al23zslq-cdc}, static two-player quadratic games \citep{calderone-oishi24} and general-sum LQ games \citep{uz-zaman-et-al24,mazumdar-et-al20,hambly-et-al23,chiu-et-al24,guan-et-al24policy-iter-lq-games}. Some of these works typically consider the same~\eqref{eq:state-lds} and assume quadratic costs for systems which are either deterministic ($w_t = 0$) or perturbed by a noise sequence~$\{w_t\}$ which is i.i.d. Gaussian. In particular they do not adopt the online learning perspective and do not address the case of arbitrary disturbances. Classical approaches to design robust controllers in the optimal control literature rely either on using statistical and probabilistic models for disturbances such as for linear quadratic Gaussian design, or adopting a (worst-case) game theoretic perspective via designing ‘minimax’ controllers like in $H_{\infty}$ control \citep{basar-bernhard08f-infinity}. 
Only few recent works adopt an online learning perspective for \textit{distributed} control \citep{ghai-et-al22regret-multi-agent-control,chang-shahrampour23distributed-online-control,chang-shahrampour2023distributed-online-LQR,martinelli-et-al24}. 
\citep{chang-shahrampour23distributed-online-control,chang-shahrampour2023distributed-online-LQR} studied a distributed online control problem over a multi-agent network of $m$ identical linear time-invariant systems in the presence of adversarial perturbations. Each agent seeks to generate a control sequence that can compete with the best centralized control policy in hindsight.  
In contrast, we address a multi-agent setting involving \textit{strategic} agents influencing a \textit{single} linear dynamical system. Our state dynamics are not separable and are influenced by all the agents. The cost of each agent in our model is influenced by the (shared) observed state which is governed by all the agents' control inputs and the goal of each agent is to minimize their own individual cost.

\noindent\textbf{Markov Games.} Regret bounds have been previously established for discrete finite Markov games. Our multi-agent linear control setting can be seen as a continuous analog to Markov games. However, note that our linear dynamical system is fundamentally different from the usual Markov game (or stochastic game) setting involving an unknown state transition kernel outputting the next state probability distribution as a function of the current state and the (joint) actions of all players. 
When considering multi-agent potential games, there are three important distinctions with existing works on Markov potential games (e.g. \cite{leonardos-et-al22mpg,zhang-et-al24tac,ding-et-al22,zhang-et-al22softmax-mpg,sun-et-al23mpg}): 
\begin{itemize}[leftmargin=*]
\item  In our work, the state and action spaces are continuous and are not mixed extensions of finite sets of states and actions. Most of the bounds scale with the cardinality of the action spaces of the players and are therefore vacuous in our continuous action space setting. In addition, our results use a suitable control policy for the linear dynamical system setting. The finite-action softmax parameterizations in these works \cite{zhang-et-al22softmax-mpg,sun-et-al23mpg} do not directly apply to our continuous control spaces.
\item  Our results consider adversarial disturbances, and hence the state transitions of the underlying system may not even be Markovian or stochastic, the disturbances can be chosen adversarially depending on the far past.
\item Our work considers cost functions that are time-varying, which is in contrast with the standard fixed reward setting in the mentioned Markov potential games works. We also do not consider discounted rewards, and the potential assumption we use is with respect to the cost function itself, and not on the aggregate cost over a time horizon. 
\end{itemize}

\noindent\textbf{Decentralized Multi-Agent RL.} Our work is also complementary to decentralized Multi-Agent RL (e.g. \cite{zhang-et-al18dec-marl}), but our setting is different from several aspects.   
First, our setting is adversarial/nonstochastic: disturbances (and hence state transitions) need not be stochastic or Markovian and may depend adversarially on the history. Second, our model uses continuous state and action spaces, whereas much of the finite-time MARL theory is developed for finite-state/finite-action Markov games, with bounds that scale in the cardinality of the state/action spaces. Third, the objectives are different: we study a game-theoretic online control problem with agent-specific time-varying general convex costs, derive individual regret guarantees, and in the common-interest case prove equilibrium-tracking guarantees. By contrast, \cite{zhang-et-al18dec-marl} focuses on convergence for a single globally averaged return objective across agents.

\section{Examples}
\label{app:examples}

\subsection{Description}
\label{subsec:description-examples}

We provide a few concrete examples to illustrate our linear dynamical game setting. 

\noindent\textbf{(a) Smart grid markets.} In modern power grids, electricity is generated and distributed by a mix of independent energy producers such as traditional plants and renewable energy providers. These actors act selfishly and adapt to market conditions while they also jointly influence the grid. Let $x_t$ be the grid state defined by characteristics such as line loads and aggregate reserves, let $u_t^i$ be generator~$i$'s power output decision (i.e. their control input) and let the sequence $w_t$ capture the demand fluctuation, the system noise and/or renewable energy shocks. Then, the system dynamics may evolve according to~\eqref{eq:state-lds} (e.g. by linearization around an operating point). Each generator~$i$ has their local cost function which accounts for the cost of production including e.g. fuel and a penalty for deviating from a target grid state. 

\noindent\textbf{(b) Formation control.} Consider a multi-agent system consisting of $N$ vehicles or robots. The state (position, velocity) and control input of each agent~$i$ at each time step~$t$ are respectively given by~$x_t^i$ and~$u_t^i\,.$ Suppose the (joint) state of the multi-agent system evolves according to~\eqref{eq:state-lds}. The formation of the multi-agent system is defined by specifying a desired distance to be maintained over time between the states of agents that are adjacent. The goal of each agent is to minimize their own formation error and energy consumption. 
A similar formation control problem has been studied in the control literature \textit{in the absence of adversarial perturbations} ($w_t = 0$) using differential games (see e.g. \cite{aghajani-doustmohammadi15formation-control,han-et-al19}) and discrete linear quadratic games \citep{hosseinirad-et-al23lq-potential-games}. 

\noindent\textbf{(c) Bioresource management.} A set of firms (or countries) exploit a set of renewable resources (e.g. a fish population) whose evolution is driven by~\eqref{eq:state-lds} where $x_t \in \mathbb{R}^d$ denotes the vector of quantities of $d$ distinct resources, the matrix $A$ encodes their natural growth rate, the control $u_t^i$ models the exploitation rate of the firm~$i$ and $w_t$ refers to perturbations due to exogenous factors such as weather conditions. Each firm~$i$ has the goal to maximize their profit while minimizing their exploitation cost. See e.g. \cite{mazalov-et-al17lq-potential-games} in the noiseless setting ($w_t = 0$). 

\subsection{About Adversarial Disturbances}

In multi-agent systems, considering adversarial disturbances allows us to model a wide range of realistic, worst-case, or strategically motivated perturbations ranging from strategic behavior in energy markets to adversarial environments in robotics and ecological shocks in resource management, ensuring system robustness even under hostile or extreme scenarios. 

We provide below examples of adversarial disturbances in each of the examples described in section~\ref{subsec:description-examples} above and comment on their importance:

\begin{itemize}[leftmargin=*]

\item \textbf{Smart grid markets:} An adversarial disturbance could model sudden demand spikes, strategic demand manipulation by large consumers (i.e. major electricity buyers who have significant influence over the overall demand on the grid), malicious data injection attacks that falsify renewable generation forecasts or misreporting. For instance, an actor might manipulate demand predictions to influence market prices or grid loads in their favor.

\item \textbf{Formation control:} Adversarial disturbances capture environmental disturbances with structured worst-case behavior, such as wind gusts or magnetic interference that affect formations in potentially harmful ways. It can also capture adversarial agents or spoofed sensor data to destabilize the formation. In hostile or uncertain environments (e.g., surveillance drones in contested airspace), agents must maintain formation despite external influences that could intentionally disrupt coordination.

\item \textbf{Bioresource management:} Adversarial disturbances may reflect deliberate misinformation about resource levels, illegal over-harvesting by untracked actors, or policy shocks (e.g., sudden trade bans) that drastically affect the resource dynamics in a harmful way. Robust resource management must consider these disturbances to avoid collapse or irreversible damage.

\end{itemize}

\section{Further Discussion of Assumptions}
\label{apdx:discussion-assumptions}

\subsection{Assumption~\ref{as:bounded-disturb}}

Bounded disturbances are standard in nonstochastic online control. 
See~\cite{ito-tsuchiya24} for a recent work allowing unbounded noise. 
As for the boundedness of the control inputs, the projection of policy parameters in the algorithm bounds the learned DAC matrices but it does not by itself bound the complete control input $u_t^j$. In particular, agent-wise strong stability of each $A - B_i K_i$ does not imply stability of the joint closed-loop matrix $A - \sum_{i=1}^N B_i K_i$. Relaxing this restrictive boundedness condition would require for instance an additional joint closed-loop boundedness or small-gain condition.  

\subsection{Assumption~\ref{as:strong-stab}}

This stability assumption requires each agent to have access to an agent-wise strongly stable linear controller. We make three important remarks regarding this assumption: 

\begin{itemize}
\item This is a multi-agent analogue of the standard stabilizing-controller assumption used throughout the online/nonstochastic control literature. 

\item As is standard in prior work on single-agent online control, we assume that agents have initial access to a stabilizing controller. Note that such controllers can be obtained offline using an SDP relaxation (e.g., using the method of \cite{cohen-et-al18online-lq-control}). Our main focus is on the challenging task of learning DAC policy parameters under adversarial disturbances. 

\item Strong stability gives the uniform trajectory bounds
used in our analysis. Without suitable stability, state norms
and cumulative costs can grow exponentially with the horizon,
so the analysis used here is unavailable in general.

\end{itemize}

\subsection{Assumption~\ref{as:strong-stab-global}}
\label{appx:as-strong-stab-global}

Individual and joint strong stability are distinct requirements.
For example, with scalar matrices $A=0$, $B_1=B_2=1$ and
$K_1=K_2=3/4$, both individual closed-loop matrices equal
$-3/4$, whereas the joint closed-loop matrix equals $-3/2$
and is unstable. Conversely, $A=2$, $B_1=B_2=1$,
$K_1=0$ and $K_2=2$ give a stable joint closed-loop matrix
but an unstable individual matrix $A-B_1K_1=2$.

Theorems~\ref{thm:local-gpc-regret-cp+}
and~\ref{thm:local-gpc-regret-cp+-lip} use joint strong
stability to bound the realized trajectory and individual
strong stability of the same supplied baselines to control
unilateral counterfactuals with opponents' realized controls
held fixed.
Theorem~\ref{thm:time-var-potential} uses joint strong
stability for its joint truncated rollouts. Its proof does
not establish the same result from individual stability alone.
Theorem~\ref{thm:local-gpc-regret-indep} instead assumes a
bound on the realized opponent inputs.

\section{Preparatory Results for the Main Proofs}

\subsection{Notation: counterfactual and idealized states and actions}
\label{subsec:notation-counterfa-state-action}

Fix a player $i \in [N]$ and the realized control sequences of its opponents.
Define 
\[
A_i:=A-B_iK_i,\qquad
d_s:=w_s+\sum_{j\ne i}B_j u_s^j\,, \quad \forall s\ge0.
\]
The DAC policy uses the effective disturbance sequence $r_s=d_s$ under Information Setting~\ref{cp0}
and $r_s=w_s$ under Information Setting~\ref{cp+}. Set $d_s=r_s=0$
for $s<0$. Both sequences are held fixed when varying the unilateral
comparator. In particular, opponents' realized controls are not
recomputed on the counterfactual state.

For a feasible parameter sequence $Q_{0:T}$, the true unilateral
counterfactual is defined by
\begin{equation}
\label{eq:d-unilateral-recurrence}
\begin{aligned}
x_0^Q&=x_0,\\
u_s^Q&=-K_i x_s^Q+\sum_{p=1}^{H}Q_s^{[p-1]}r_{s-p},\\
x_{s+1}^Q&=A_i x_s^Q+
B_i\sum_{p=1}^{H}Q_s^{[p-1]}r_{s-p}+d_s.
\end{aligned}
\end{equation}
For a fixed $Q\in\mathcal M_i$, the same notation refers to the
stationary sequence $Q_s=Q$. Equivalently, these stationary
trajectories can be denoted by $x_t^{K_i}(Q),u_t^{i,K_i}(Q)$.
The time-zero DAC correction is zero, so every DAC comparator with
this fixed baseline applies $u_0^Q=-K_ix_0$. A linear comparator
with a different baseline $K_i^*$ instead applies $-K_i^*x_0$.

Write $h=H+1$. For every $t\ge1$, define the zero-initialized
finite-memory state, action, and loss by
\begin{align}
\label{eq:d-local-surrogate}
y_t^i(Q_{t-h:t-1})
&=\sum_{a=1}^{h}A_i^{a-1}
\left(d_{t-a}+B_i\sum_{p=1}^{H}
Q_{t-a}^{[p-1]}r_{t-a-p}\right),\\
\label{eq:d-local-action}
v_t^i(Q_{t-h:t})
&=-K_i y_t^i(Q_{t-h:t-1})+
\sum_{p=1}^{H}Q_t^{[p-1]}r_{t-p},\\
\label{eq:d-memory-loss}
\ell_t^i(Q_{t-h:t})
&=c_t^i\bigl(y_t^i(Q_{t-h:t-1}),v_t^i(Q_{t-h:t})\bigr),\\
\label{eq:d-diagonal-loss}
f_t^i(Q)&:=\ell_t^i(Q,\ldots,Q).
\end{align}
The single-argument notation $\ell_t^i(Q)$ is also used interchangeably with the 
diagonal function $f_t^i(Q)$. Similarly, the notation
$y_t^{i,K_i},v_t^{i,K_i}$ denotes $y_t^i,v_t^i$ when the baseline
needs to be displayed explicitly.

Negative-time parameter values can be set to the initial feasible
parameter: they multiply zero-padded disturbance values. These definitions
therefore also apply to early rounds $1\le t<h$. For $t\ge h$, the
truncation starts at time $t-h$ with state zero and includes the $h$
transitions through time $t-1$; the ideal action additionally uses
the current parameter $Q_t$. The memory loss consequently has
$h+1=H+2$ parameter arguments. The ideal state and action are affine
in these arguments. Thus $\ell_t^i$ is jointly convex and its
diagonal $f_t^i$ is convex under Assumption~\ref{as:convexity}.
Algorithm~\ref{algo:MAGPC} uses $\nabla f_t^i$.

\subsection{State evolution}
\label{subsec:proof-prop-state-evol}

We record both individual and joint state expansions. In this
subsection $m$ denotes a general unrolling index, so an expansion
with index $m$ contains $m+1$ transitions. For integers
$t\ge0$, $0\le m\le t$, and $0\le l\le H+m$, define
\begin{align}
\label{eq:notations-psi}
\widetilde A_{K_i}&:=A-B_iK_i,\qquad
\Psi_{t,l}^{i,m}(M_{i,t-m:t})
:=\widetilde A_{K_i}^{l}\mathbf 1_{\{l\le m\}}
+\sum_{k=0}^{m}\widetilde A_{K_i}^{k}B_i
M_{i,t-k}^{[l-k-1]}\mathbf 1_{\{1\le l-k\le H\}},\\
\label{eq:notations-psi-bar}
\bar A_K&:=A-\sum_{i=1}^N B_iK_i,\qquad
\bar\Psi_{t,l}^{m}(M_{t-m:t})
:=\bar A_K^{l}\mathbf 1_{\{l\le m\}}
+\sum_{k=0}^{m}\bar A_K^{k}\sum_{i=1}^NB_i
M_{i,t-k}^{[l-k-1]}\mathbf 1_{\{1\le l-k\le H\}}.
\end{align}
In these transfer-matrix formulas only, blocks outside
$\{0,\ldots,H-1\}$ are interpreted as zero. The matrices
$\Psi_{t,l}^{i,m}$ and $\bar\Psi_{t,l}^{m}$ are affine in their
DAC parameter arguments.

\begin{proposition}[State evolution]
\label{prop:state-evol}
For integers $t\ge0$ and $0\le m\le t$, the following hold.
\begin{enumerate}[label=(\roman*),leftmargin=*]
\item\label{eq:state-evol-cp0}
Suppose player $i$ uses a nonstationary DAC policy with 
$\widetilde w_s:=w_s+\sum_{j\ne i}B_j u_s^j$, while opponents'
realized controls are fixed. Then
\[
x_{t+1}=\widetilde A_{K_i}^{m+1}x_{t-m}
+\sum_{l=0}^{H+m}\Psi_{t,l}^{i,m}(M_{i,t-m:t})
\widetilde w_{t-l}.
\]
\item\label{eq:state-evol-cp+all-dac}
Suppose all agents use nonstationary DAC policies driven by $w_s$.
Then the actual joint state satisfies
\begin{equation}
\label{eq:state-evol-under-dac-w-all}
x_{t+1}=\bar A_K^{m+1}x_{t-m}
+\sum_{l=0}^{H+m}\bar\Psi_{t,l}^{m}(M_{t-m:t})w_{t-l}.
\end{equation}
\end{enumerate}
\end{proposition}

\begin{proof}
We prove the two claims separately. Recall that disturbances are
extended by zero to negative indices:
$w_s=\widetilde w_s=0$ for $s<0$.
All empty sums below are interpreted as zero.

We first record the elementary unrolling identity used in both
claims. For a recursion of the form
\[
z_{s+1}=Cz_s+b_s,
\]
we have, for every $0\le m\le t$,
\begin{equation}
\label{eq:state-evol-unrolling-id}
z_{t+1}
=C^{m+1}z_{t-m}+\sum_{k=0}^{m}C^k b_{t-k}.
\end{equation}
Indeed, the case $m=0$ is the defining recursion. If the identity
holds with $m=r<t$, substituting
$z_{t-r}=Cz_{t-r-1}+b_{t-r-1}$ gives
\begin{equation*}
z_{t+1}
=C^{r+1}\bigl(Cz_{t-r-1}+b_{t-r-1}\bigr)
  +\sum_{k=0}^{r}C^k b_{t-k} 
=C^{r+2}z_{t-r-1}
  +\sum_{k=0}^{r+1}C^k b_{t-k},
\end{equation*}
which proves~\eqref{eq:state-evol-unrolling-id} by induction.

\medskip
\noindent\textbf{Proof of Claim~\ref{eq:state-evol-cp0}.} For $s\ge0$, player $i$ uses the DAC policy: 
\[
u_s^i=-K_i x_s+
\sum_{p=1}^{H}M_{i,s}^{[p-1]}\widetilde w_{s-p}.
\]
Substituting this expression into the system dynamics and using 
$\widetilde w_s=w_s+\sum_{j\ne i}B_j u_s^j$ yields: 
\begin{align*}
x_{s+1}
&=Ax_s+B_i u_s^i+\sum_{j\ne i}B_j u_s^j+w_s\\
&=Ax_s+
B_i\left(-K_i x_s+
\sum_{p=1}^{H}M_{i,s}^{[p-1]}\widetilde w_{s-p}\right)
+\widetilde w_s\\
&=\widetilde A_{K_i}x_s
+\widetilde w_s
+B_i\sum_{p=1}^{H}M_{i,s}^{[p-1]}\widetilde w_{s-p},
\end{align*}
where $\widetilde A_{K_i}=A-B_iK_i$.

Applying~\eqref{eq:state-evol-unrolling-id} with
$C=\widetilde A_{K_i}$ and $b_s=\widetilde w_s+
B_i\sum_{p=1}^{H}M_{i,s}^{[p-1]}\widetilde w_{s-p}$, 
we obtain: 
\begin{align*}
x_{t+1}
=&\widetilde A_{K_i}^{m+1}x_{t-m} +\sum_{k=0}^{m}\widetilde A_{K_i}^{k}
\left(
\widetilde w_{t-k}
+B_i\sum_{p=1}^{H}
M_{i,t-k}^{[p-1]}\widetilde w_{t-k-p}
\right)\\
={}&\widetilde A_{K_i}^{m+1}x_{t-m}
+\sum_{k=0}^{m}\widetilde A_{K_i}^{k}\widetilde w_{t-k}
+\sum_{k=0}^{m}\sum_{p=1}^{H}
\widetilde A_{K_i}^{k}B_i
M_{i,t-k}^{[p-1]}\widetilde w_{t-k-p}.
\end{align*}

We now collect the terms involving the same disturbance.
For the direct disturbance contribution, simply rename $k$ as $l$:
\[
\sum_{k=0}^{m}\widetilde A_{K_i}^{k}\widetilde w_{t-k}
=
\sum_{l=0}^{H+m}
\widetilde A_{K_i}^{l}
\mathbf 1_{\{l\le m\}}\widetilde w_{t-l}.
\]
For the DAC contribution, set $l=k+p$. The conditions
$0\le k\le m$ and $1\le p\le H$ are equivalent to the following conditions: 
\[
1\le l\le H+m,
\qquad
\max\{0,l-H\}\le k\le\min\{m,l-1\}.
\]
Moreover, $p-1=l-k-1$. Therefore,
\begin{equation*}
\sum_{k=0}^{m}\sum_{p=1}^{H}
\widetilde A_{K_i}^{k}B_i
M_{i,t-k}^{[p-1]}\widetilde w_{t-k-p}
=
\sum_{l=1}^{H+m}
\left(
\sum_{k=\max\{0,l-H\}}^{\min\{m,l-1\}}
\widetilde A_{K_i}^{k}B_i
M_{i,t-k}^{[l-k-1]}
\right)\widetilde w_{t-l}.
\end{equation*}
There is no DAC contribution at $l=0$, since $k+p\ge1$.

Combining the direct and DAC contributions gives: 
\begin{align*}
x_{t+1}
&= \widetilde A_{K_i}^{m+1}x_{t-m} +\sum_{l=0}^{H+m}
\left(
\widetilde A_{K_i}^{l}\mathbf 1_{\{l\le m\}}
+
\sum_{\substack{0\le k\le m\\1\le l-k\le H}}
\widetilde A_{K_i}^{k}B_iM_{i,t-k}^{[l-k-1]}
\right)\widetilde w_{t-l}\\
&=\widetilde A_{K_i}^{m+1}x_{t-m}
+\sum_{l=0}^{H+m}
\Psi_{t,l}^{i,m}(M_{i,t-m:t})\widetilde w_{t-l},
\end{align*}
where the last equality is the definition of
$\Psi_{t,l}^{i,m}$.
This proves Claim~\ref{eq:state-evol-cp0}.

\medskip
\noindent\textbf{Proof of Claim~\ref{eq:state-evol-cp+all-dac}.} 
Each player $i\in[N]$ now uses the disturbance sequence
$\{w_s\}$ in its DAC policy:
\[
u_s^i=-K_i x_s+
\sum_{p=1}^{H}M_{i,s}^{[p-1]}w_{s-p}.
\]
Substituting all players' policies into the dynamics yields: 
\begin{align*}
x_{s+1}
&=Ax_s+\sum_{i=1}^{N}B_i u_s^i+w_s\\
&=Ax_s+
\sum_{i=1}^{N}B_i
\left(-K_i x_s+
\sum_{p=1}^{H}M_{i,s}^{[p-1]}w_{s-p}\right)
+w_s\\
&=\left(A-\sum_{i=1}^{N}B_iK_i\right)x_s
+w_s+
\sum_{i=1}^{N}B_i
\sum_{p=1}^{H}M_{i,s}^{[p-1]}w_{s-p}\\
&=\bar A_Kx_s+w_s+
\sum_{i=1}^{N}B_i
\sum_{p=1}^{H}M_{i,s}^{[p-1]}w_{s-p}.
\end{align*}
In particular, the direct disturbance $w_s$ appears once,
outside the sum over players.

Applying~\eqref{eq:state-evol-unrolling-id} with
$C=\bar A_K$ and $b_s=w_s+
\sum_{i=1}^{N}B_i
\sum_{p=1}^{H}M_{i,s}^{[p-1]}w_{s-p}$ 
gives: 
\begin{align*}
x_{t+1}
={}&\bar A_K^{m+1}x_{t-m} 
+\sum_{k=0}^{m}\bar A_K^k
\left(
w_{t-k}+
\sum_{i=1}^{N}B_i
\sum_{p=1}^{H}M_{i,t-k}^{[p-1]}w_{t-k-p}
\right)\\
={}&\bar A_K^{m+1}x_{t-m}
+\sum_{k=0}^{m}\bar A_K^k w_{t-k}
+\sum_{k=0}^{m}\sum_{i=1}^{N}\sum_{p=1}^{H}
\bar A_K^k B_iM_{i,t-k}^{[p-1]}w_{t-k-p}.
\end{align*}

As in the first claim, the direct disturbance contribution is
\[
\sum_{k=0}^{m}\bar A_K^k w_{t-k}
=
\sum_{l=0}^{H+m}
\bar A_K^l\mathbf 1_{\{l\le m\}}w_{t-l}.
\]
For the DAC contributions, use the same change of indices
$l=k+p$. The finite sum over players does not change the
admissible range of $k$, so
\begin{align*}
\sum_{k=0}^{m}\sum_{i=1}^{N}\sum_{p=1}^{H}
\bar A_K^k B_iM_{i,t-k}^{[p-1]}w_{t-k-p}
=
\sum_{l=1}^{H+m}
\left(
\sum_{k=\max\{0,l-H\}}^{\min\{m,l-1\}}
\bar A_K^k\sum_{i=1}^{N}
B_iM_{i,t-k}^{[l-k-1]}
\right)w_{t-l}.
\end{align*}

Combining these expressions, and extending the DAC sum to
$l=0$ using the empty-sum convention, we obtain: 
\begin{align*}
x_{t+1}
&=\bar A_K^{m+1}x_{t-m}
+\sum_{l=0}^{H+m}
\left(
\bar A_K^l\mathbf 1_{\{l\le m\}}
+
\sum_{\substack{0\le k\le m\\1\le l-k\le H}}
\bar A_K^k\sum_{i=1}^{N}
B_iM_{i,t-k}^{[l-k-1]}
\right)w_{t-l}\\
&= \bar A_K^{m+1}x_{t-m}
+\sum_{l=0}^{H+m}
\bar\Psi_{t,l}^{m}(M_{t-m:t})w_{t-l}.
\end{align*}
The last equality follows from the definition of
$\bar\Psi_{t,l}^{m}$ and proves
\eqref{eq:state-evol-under-dac-w-all}.
\end{proof}

\subsection{Transfer matrix bound}

\begin{lemma}
\label{lem:psi-bound}
Let Assumption~\ref{as:strong-stab-global} hold and put
$q=1-\bar\gamma$ and $B_\Sigma=\sum_{i=1}^N\|B_i\|_{\mathrm{op}}$.
Suppose that
$\|M_{i,s}^{[p]}\|_{\mathrm{op}}\le\tau q^p$
for some $\tau>0$, every $s\ge0$, and $0\le p<H$.
For integers $t\ge0$, $0\le m\le t$, and $0\le l\le H+m$,
\[
\|\bar\Psi_{t,l}^{m}(M_{t-m:t})\|_{\mathrm{op}}
\le\bar\kappa q^l\mathbf 1_{\{l\le m\}}
+\mathbf 1_{\{l\ge1\}}H\bar\kappa\tau B_\Sigma q^{l-1}.
\]
\end{lemma}

\begin{proof}
Recall from~\eqref{eq:notations-psi-bar} that the state transfer
matrix is
\begin{equation}
\label{eq:psi-bar-state-transfer-mat}
\bar\Psi_{t,l}^{m}(M_{t-m:t})
=
\bar A_K^l\mathbf 1_{\{l\le m\}}
+
\sum_{\substack{0\le k\le m\\1\le l-k\le H}}
\bar A_K^k
\sum_{i=1}^N B_iM_{i,t-k}^{[l-k-1]},
\end{equation}
where
$\bar A_K=A-\sum_{i=1}^N B_iK_i$.
We have written the second term as a restricted sum so that
every DAC block appearing in the expression has an index in
$\{0,\ldots,H-1\}$.
This is equivalent to the indicator notation
in~\eqref{eq:notations-psi-bar}.
Throughout the proof, an empty sum is equal to zero.

By Assumption~\ref{as:strong-stab-global}, there exist an
invertible matrix $Q$ and a matrix $L$ such that: 
\[
\bar A_K=QLQ^{-1},
\qquad
\|Q\|_{\mathrm{op}}\|Q^{-1}\|_{\mathrm{op}}
\le\bar\kappa,
\qquad
\|L\|_{\mathrm{op}}\le 1-\bar\gamma=q.
\]
Consequently, for every integer $a\ge0$, submultiplicativity
of the operator norm gives
\begin{equation}
\label{eq:bound-bar-A-k-strong-stab}
\|\bar A_K^a\|_{\mathrm{op}}
=\|QL^aQ^{-1}\|_{\mathrm{op}}
\le
\|Q\|_{\mathrm{op}}\,
\|L^a\|_{\mathrm{op}}\,
\|Q^{-1}\|_{\mathrm{op}} 
\le
\|Q\|_{\mathrm{op}}\,
\|Q^{-1}\|_{\mathrm{op}}\,
\|L\|_{\mathrm{op}}^a
\le \bar\kappa q^a.
\end{equation}
This bound also holds for $a=0$, since
$1=\|I_d\|_{\mathrm{op}}
\le
\|Q\|_{\mathrm{op}}\|Q^{-1}\|_{\mathrm{op}}
\le\bar\kappa.$ 

If $l=0$, there is no integer $k\ge0$ satisfying
$1\le l-k\le H$.
Thus the DAC contribution
in~\eqref{eq:psi-bar-state-transfer-mat} is zero, and
$\bar\Psi_{t,0}^{m}(M_{t-m:t})=\bar A_K^0=I_d.$ 
It follows that
$\|\bar\Psi_{t,0}^{m}(M_{t-m:t})\|_{\mathrm{op}}
=1\le\bar\kappa,$ which is exactly the claimed bound when $l=0$.

We are now ready to bound the transfer matrix for $l \ge 1$. 
Fix now an integer $1\le l\le H+m$.
For every index $k$ included in the restricted sum,
\[
t-k\ge t-m\ge0,
\qquad
0\le l-k-1\le H-1.
\]
The assumed DAC block bound therefore applies and gives
$\|M_{i,t-k}^{[l-k-1]}\|_{\mathrm{op}}
\le\tau q^{l-k-1}.$ 
Using the triangle inequality and submultiplicativity
in~\eqref{eq:psi-bar-state-transfer-mat}, we obtain: 
\begin{align*}
\|\bar\Psi_{t,l}^{m}(M_{t-m:t})\|_{\mathrm{op}}
&\le
\|\bar A_K^l\|_{\mathrm{op}}\mathbf 1_{\{l\le m\}}
+
\sum_{\substack{0\le k\le m\\1\le l-k\le H}}
\|\bar A_K^k\|_{\mathrm{op}}
\sum_{i=1}^N
\|B_i\|_{\mathrm{op}}\,
\|M_{i,t-k}^{[l-k-1]}\|_{\mathrm{op}}
\\
&\le
\bar\kappa q^l\mathbf 1_{\{l\le m\}}
+
\bar\kappa\tau
\sum_{i=1}^N\|B_i\|_{\mathrm{op}}
\sum_{\substack{0\le k\le m\\1\le l-k\le H}}
q^kq^{l-k-1}
\\
&=
\bar\kappa q^l\mathbf 1_{\{l\le m\}}
+
\bar\kappa\tau B_\Sigma q^{l-1}
\sum_{\substack{0\le k\le m\\1\le l-k\le H}}1.
\end{align*}
Here, the second inequality uses
\eqref{eq:bound-bar-A-k-strong-stab} and the DAC block
bound. 
In particular, after these bounds are applied, each admissible
index $k$ contributes the same upper bound.

The conditions $0\le k\le m$ and $1\le l-k\le H$ are
equivalent to
$\max\{0,l-H\}\le k\le\min\{m,l-1\}.$ 
Every such index therefore belongs to
$\{l-H,l-H+1,\ldots,l-1\}$, which contains exactly $H$
integers. Intersecting this set with $\{0,\ldots,m\}$
cannot increase its cardinality. Hence
$\sum_{\substack{0\le k\le m\\1\le l-k\le H}}1\le H.$ 
Substituting this estimate into the preceding bound gives,
for every $l\ge1$ in the stated range,
\[
\|\bar\Psi_{t,l}^{m}(M_{t-m:t})\|_{\mathrm{op}}
\le
\bar\kappa q^l\mathbf 1_{\{l\le m\}}
+
H\bar\kappa\tau B_\Sigma q^{l-1}.
\]
Combining this inequality with the case $l=0$ proves: 
\[
\|\bar\Psi_{t,l}^{m}(M_{t-m:t})\|_{\mathrm{op}}
\le
\bar\kappa q^l\mathbf 1_{\{l\le m\}}
+
\mathbf 1_{\{l\ge1\}}
H\bar\kappa\tau B_\Sigma q^{l-1},
\]
as claimed.
\end{proof}

\subsection{State, action and difference of state and action bounds}

The goal of the next proposition is to control the norms of states, actions and differences of states and actions. 
Note that we pay particular attention to the problem constants involved to elucidate the dependence of our bounds on the number of agents~$N$ and the magnitude of the control inputs of all the agents. The result is a more refined version of \cite[Lemma 5.5]{agarwal2019online} which is adapted to our multi-agent control setting when each agent executes a \eqref{eq:dac-i} policy. 

\begin{proposition}
\label{prop:bounds-states-actions}
Let Assumptions~\ref{as:bounded-disturb} and
\ref{as:strong-stab-global} hold, and suppose $\|x_0\|_2\le X_0$.
Let all agents use DAC policies with
$\|M_{i,s}^{[p]}\|_{\mathrm{op}}\le\tau q^p$, where
$q=1-\bar\gamma$, $\tau>0$, and $0\le p<H$. Set
\[
m_H:=\tau\sum_{p=0}^{H-1}q^p,\qquad
B_\Sigma:=\sum_{i=1}^N\|B_i\|_{\mathrm{op}},\qquad
Y:=\frac{\bar\kappa W}{\bar\gamma}(1+B_\Sigma m_H),
\qquad X:=\bar\kappa X_0+Y.
\]
With negative-time disturbances set to zero, define the joint ideal
state and action by
\begin{align}
\label{eq:state-evol-y}
y_t^K&:=\sum_{a=1}^{H+1}\bar A_K^{a-1}
\left(w_{t-a}+\sum_{j=1}^NB_j\sum_{p=1}^H
M_{j,t-a}^{[p-1]}w_{t-a-p}\right),\\
\label{eq:action-vtiKi}
v_t^{i,K}&:=-K_i y_t^K+
\sum_{p=1}^H M_{i,t}^{[p-1]}w_{t-p}.
\end{align}
Negative-time parameter values are arbitrary feasible values,
since their contributions are zero. Then, we have the following bounds: 
\begin{enumerate}[label=(\roman*),leftmargin=*]
\item\label{lem-item-i}
For every $t\ge0$, the actual joint state satisfies
\begin{equation}
\label{eq:state-bound-dac-i}
\|x_t\|_2\le X.
\end{equation}
\item\label{lem-item-ii}
For every $t\ge1$, $\|y_t^K\|_2\le Y$.
\item\label{lem-item-iii}
If $K^*$ is another jointly $(\bar\kappa,\bar\gamma)$-strongly
stable linear controller, then
\[
\|x_t^{K^*}(0)\|_2
\le\bar\kappa X_0+\bar\kappa W/\bar\gamma
\qquad(t\ge0).
\]
\item\label{lem-item-iv}
For every $t\ge0$, $\|u_t^i\|_2\le\bar\kappa X+Wm_H$.
\item\label{lem-item-v}
For every $t\ge1$, $\|v_t^{i,K}\|_2\le\bar\kappa Y+Wm_H$.
\item\label{lem-item-vi}
For every $t\ge H+1$,
\begin{equation}
\label{eq:diff-state-ideal-state-bound}
\|x_t-y_t^K\|_2\le\bar\kappa q^{H+1}X.
\end{equation}
\item\label{lem-item-vii}
For every $t\ge H+1$,
$\|u_t^i-v_t^{i,K}\|_2\le\bar\kappa^2q^{H+1}X$.
\item\label{lem-item-viii}
A common bound for all the displayed states and actions is
\begin{equation}
\label{eq:D-def}
D:=\bar\kappa^2 X+Wm_H.
\end{equation}
Moreover, for $t\ge H+1$,
\begin{equation}
\label{eq:uniform-diff-state-bound}
\max\{\|x_t-y_t^K\|_2,\|u_t^i-v_t^{i,K}\|_2\}
\le q^{H+1}D.
\end{equation}
\end{enumerate}
In our setting, the DAC domain uses $\tau=2\bar\kappa^2$.
If $X_0,W,\bar\kappa,\bar\gamma^{-1}$ and
$\max_{i\in[N]}\|B_i\|_{\mathrm{op}}$
are uniformly bounded in $N$, then
$X,Y,D=\mathcal O(N)$, uniformly in $H$, since
$m_H\le\tau/\bar\gamma$.
\end{proposition}

\begin{proof}
We prove each statement of the proposition separately.
Throughout the proof, we use the consequence of joint strong
stability established in~\eqref{eq:bound-bar-A-k-strong-stab}:
$\|\bar A_K^a\|_{\mathrm{op}}
\le \bar\kappa q^a$
for every integer $a\ge0$.
We also use $\|K_i\|_{\mathrm{op}}\le\bar\kappa$ for every $i\in[N]$.

Indeed, each $K_i$ is a block row of the stacked controller
$K=(K_1^\top,\ldots,K_N^\top)^\top$, and joint strong
stability includes the bound
$\|K\|_{\mathrm{op}}\le\bar\kappa$.

\medskip
\noindent\textbf{Proof of Claim~\ref{lem-item-i}.}
Substituting the DAC policies into the state dynamics gives,
for every $s\ge0$,
\begin{align*}
x_{s+1}
&=Ax_s+\sum_{j=1}^NB_j
\left(
-K_jx_s+\sum_{p=1}^H
M_{j,s}^{[p-1]}w_{s-p}
\right)+w_s\\
&=\bar A_Kx_s+
w_s+\sum_{j=1}^NB_j
\sum_{p=1}^H M_{j,s}^{[p-1]}w_{s-p}.
\end{align*}
For convenience, define: 
\[
b_s:=
w_s+\sum_{j=1}^NB_j
\sum_{p=1}^H M_{j,s}^{[p-1]}w_{s-p}.
\]
Thus the actual joint state satisfies
$x_{s+1}=\bar A_Kx_s+b_s$.

We first bound $b_s$.
For every player $j$, the block constraints and the disturbance
bound give: 
\begin{equation*}
\left\|
\sum_{p=1}^H M_{j,s}^{[p-1]}w_{s-p}
\right\|_2
\le
\sum_{p=1}^H
\|M_{j,s}^{[p-1]}\|_{\mathrm{op}}\,
\|w_{s-p}\|_2
\le
W\tau\sum_{p=1}^Hq^{p-1}
=Wm_H.
\end{equation*}
The estimate remains valid when some indices $s-p$ are
negative, since the corresponding disturbances are zero.
Consequently,
\begin{equation}
\label{eq:joint-dac-forcing-bound}
\|b_s\|_2
\le
\|w_s\|_2+
\sum_{j=1}^N\|B_j\|_{\mathrm{op}}
\left\|
\sum_{p=1}^H M_{j,s}^{[p-1]}w_{s-p}
\right\|_2
\le W+B_\Sigma Wm_H
=W(1+B_\Sigma m_H).
\end{equation}
For $s<0$, the zero-padding convention gives $b_s=0$,
so the same bound also holds for negative indices.

Unrolling the actual recurrence from time zero yields: 
\[
x_t
=
\bar A_K^t x_0+
\sum_{s=0}^{t-1}\bar A_K^{t-1-s}b_s.
\]
For $t=0$, the sum is empty and this identity reduces to
$x_0=x_0$.
Using~\eqref{eq:bound-bar-A-k-strong-stab},
\eqref{eq:joint-dac-forcing-bound}, and
$\|x_0\|_2\le X_0$, we obtain: 
\begin{align*}
\|x_t\|_2
&\le
\|\bar A_K^t\|_{\mathrm{op}}\|x_0\|_2+
\sum_{s=0}^{t-1}
\|\bar A_K^{t-1-s}\|_{\mathrm{op}}\|b_s\|_2\\
&\le
\bar\kappa q^tX_0+
\bar\kappa W(1+B_\Sigma m_H)
\sum_{s=0}^{t-1}q^{t-1-s}\\
&=
\bar\kappa q^tX_0+
\bar\kappa W(1+B_\Sigma m_H)
\frac{1-q^t}{1-q}\\
&\le
\bar\kappa X_0+
\frac{\bar\kappa W}{\bar\gamma}
(1+B_\Sigma m_H)\\
&=\bar\kappa X_0+Y=X.
\end{align*}
Here we used $1-q=\bar\gamma$ and $0<q<1$.
This proves~\eqref{eq:state-bound-dac-i} for every $t\ge0$.

\medskip
\noindent\textbf{Proof of Claim~\ref{lem-item-ii}.}
By the definition of the joint ideal state
in~\eqref{eq:state-evol-y},
\[
y_t^K
=
\sum_{a=1}^{H+1}\bar A_K^{a-1}b_{t-a}.
\]
There is no initial-state contribution in this expression
because the ideal rollout starts from zero at its truncation
boundary.
Applying the power bound and
\eqref{eq:joint-dac-forcing-bound}, we obtain: 
\begin{align*}
\|y_t^K\|_2
&\le
\sum_{a=1}^{H+1}
\|\bar A_K^{a-1}\|_{\mathrm{op}}\,
\|b_{t-a}\|_2\\
&\le
\bar\kappa W(1+B_\Sigma m_H)
\sum_{a=1}^{H+1}q^{a-1}\\
&=
\bar\kappa W(1+B_\Sigma m_H)
\frac{1-q^{H+1}}{1-q}\\
&\le
\frac{\bar\kappa W}{\bar\gamma}
(1+B_\Sigma m_H)
=Y.
\end{align*}
This argument applies for every $t\ge1$.
In particular, during early rounds, terms with $t-a<0$
vanish because $b_{t-a}=0$.

\medskip
\noindent\textbf{Proof of Claim~\ref{lem-item-iii}.}
Let
\[
\bar A_{K^*}:=
A-\sum_{j=1}^NB_jK_j^*.
\]
When all players use the linear controllers $K_j^*$,
the state satisfies
\[
x_{s+1}^{K^*}(0)
=
\bar A_{K^*}x_s^{K^*}(0)+w_s,
\qquad
x_0^{K^*}(0)=x_0.
\]
Here the argument $0$ denotes zero DAC correction;
the initial state remains $x_0$.

Unrolling this recurrence gives: 
\[
x_t^{K^*}(0)
=
\bar A_{K^*}^t x_0+
\sum_{s=0}^{t-1}\bar A_{K^*}^{t-1-s}w_s.
\]
Since $K^*$ is jointly
$(\bar\kappa,\bar\gamma)$-strongly stable, its closed-loop
matrix satisfies
$\|\bar A_{K^*}^a\|_{\mathrm{op}}
\le\bar\kappa q^a$ for all $a \ge 0$. 
Therefore,
\begin{align*}
\|x_t^{K^*}(0)\|_2
&\le
\|\bar A_{K^*}^t\|_{\mathrm{op}} \cdot \|x_0\|_2+
\sum_{s=0}^{t-1}
\|\bar A_{K^*}^{t-1-s}\|_{\mathrm{op}} \cdot\|w_s\|_2\\
&\le
\bar\kappa q^tX_0+
\bar\kappa W\sum_{s=0}^{t-1}q^{t-1-s}\\
&\le
\bar\kappa X_0+
\frac{\bar\kappa W}{\bar\gamma}.
\end{align*}
The same calculation includes $t=0$ through the
empty sum convention.

\medskip
\noindent\textbf{Proof of Claim~\ref{lem-item-iv}.}
The actual action of player $i$ is: 
\begin{equation}
\label{eq:action-utiKi}
u_t^i
=
-K_ix_t+
\sum_{p=1}^H M_{i,t}^{[p-1]}w_{t-p}.
\end{equation}
By the triangle inequality, the controller norm bound,
and Claim~\ref{lem-item-i},
\begin{equation*}
\|u_t^i\|_2
\le
\|K_i\|_{\mathrm{op}}\|x_t\|_2+
\sum_{p=1}^H
\|M_{i,t}^{[p-1]}\|_{\mathrm{op}}\|w_{t-p}\|_2
\le
\bar\kappa X+
W\tau\sum_{p=1}^Hq^{p-1}
=\bar\kappa X+Wm_H.
\end{equation*}
This holds for every $t\ge0$.
At $t=0$, all disturbances in the DAC correction are zero,
so the action is simply $u_0^i=-K_ix_0$.

\medskip
\noindent\textbf{Proof of Claim~\ref{lem-item-v}.}
By~\eqref{eq:action-vtiKi}, the ideal action is given by: 
\[
v_t^{i,K}
=
-K_iy_t^K+
\sum_{p=1}^H M_{i,t}^{[p-1]}w_{t-p}.
\]
Using Claim~\ref{lem-item-ii} in place of the actual-state
bound, we obtain for every $t \ge 1$, 
\begin{align*}
\|v_t^{i,K}\|_2
&\le
\|K_i\|_{\mathrm{op}} \cdot \|y_t^K\|_2+
\sum_{p=1}^H
\|M_{i,t}^{[p-1]}\|_{\mathrm{op}} \cdot \|w_{t-p}\|_2\\
&\le
\bar\kappa Y+
W\tau\sum_{p=1}^Hq^{p-1}\\
&=\bar\kappa Y+Wm_H \,.
\end{align*}

\medskip
\noindent\textbf{Proof of Claim~\ref{lem-item-vi}.}
Fix $t\ge H+1$.
Unrolling the actual recurrence
$x_{s+1}=\bar A_Kx_s+b_s$ over its last $H+1$
transitions gives: 
\begin{equation}
\label{eq:state-evol-t-1-H}
x_t
=
\bar A_K^{H+1}x_{t-H-1}
+
\sum_{a=1}^{H+1}\bar A_K^{a-1}b_{t-a}.
\end{equation}
The sum in this expression is exactly the joint ideal
state $y_t^K$ defined in~\eqref{eq:state-evol-y}.
In particular, the actual and ideal rollouts use the same
joint DAC parameter sequence and the same disturbance
values over the truncation window.
Subtracting the two expressions therefore yields the
exact identity: 
\[
x_t-y_t^K
=
\bar A_K^{H+1}x_{t-H-1}.
\]
It follows that: 
\begin{equation*}
\|x_t-y_t^K\|_2
\le
\|\bar A_K^{H+1}\|_{\mathrm{op}} \cdot 
\|x_{t-H-1}\|_2
\le
\bar\kappa q^{H+1}\|x_{t-H-1}\|_2
\le
\bar\kappa q^{H+1}X.
\end{equation*}
The last inequality follows from Claim~\ref{lem-item-i},
which applies because $t-H-1\ge0$.
This includes the boundary case $t=H+1$, where the
omitted state contribution is
$\bar A_K^{H+1}x_0$.
Thus~\eqref{eq:diff-state-ideal-state-bound} holds.

\medskip
\noindent\textbf{Proof of Claim~\ref{lem-item-vii}.}
The actual and ideal actions use the same time-$t$
DAC parameters and disturbance history.
Consequently, subtracting~\eqref{eq:action-vtiKi}
from~\eqref{eq:action-utiKi} cancels their DAC corrections:
\begin{equation*}
u_t^i-v_t^{i,K}
=
-K_ix_t+
\sum_{p=1}^H M_{i,t}^{[p-1]}w_{t-p}
+
K_iy_t^K-
\sum_{p=1}^H M_{i,t}^{[p-1]}w_{t-p}
=-K_i(x_t-y_t^K).
\end{equation*}
Using $\|K_i\|_{\mathrm{op}}\le\bar\kappa$ and
Claim~\ref{lem-item-vi}, we obtain, for $t\ge H+1$,
\begin{equation*}
\|u_t^i-v_t^{i,K}\|_2
\le
\|K_i\|_{\mathrm{op}}\|x_t-y_t^K\|_2
\le
\bar\kappa\|x_t-y_t^K\|_2
\le
\bar\kappa^2q^{H+1}X.
\end{equation*}

\medskip
\noindent\textbf{Proof of Claim~\ref{lem-item-viii}.}
First, joint strong stability implies $\bar\kappa\ge1$.
Indeed, for its similarity matrix $Q$,
$1=\|I_d\|_{\mathrm{op}}
=\|QQ^{-1}\|_{\mathrm{op}}
\le
\|Q\|_{\mathrm{op}}\|Q^{-1}\|_{\mathrm{op}}
\le\bar\kappa.$ 
Recall that
$D=\bar\kappa^2X+Wm_H.$

We first check the state bounds.
Since $X=\bar\kappa X_0+Y$, we have $Y\le X$.
Moreover,
\[
\bar\kappa X_0+\frac{\bar\kappa W}{\bar\gamma}
\le
\bar\kappa X_0+
\frac{\bar\kappa W}{\bar\gamma}(1+B_\Sigma m_H)
=X.
\]
Thus every state bound in
Claims~\ref{lem-item-i}--\ref{lem-item-iii}
is at most $X$, and
$X\le\bar\kappa^2X\le D.$

Next, the action bounds in
Claims~\ref{lem-item-iv} and~\ref{lem-item-v}
satisfy
\[
\bar\kappa Y+Wm_H
\le
\bar\kappa X+Wm_H
\le
\bar\kappa^2X+Wm_H
=D.
\]
Hence $D$ bounds every displayed state and action,
within the time ranges specified in the proposition.

Finally, Claims~\ref{lem-item-vi} and~\ref{lem-item-vii},
together with $\bar\kappa\ge1$, give: 
\begin{equation*}
\max\{
\|x_t-y_t^K\|_2,\,
\|u_t^i-v_t^{i,K}\|_2
\}
\le
q^{H+1}\max\{\bar\kappa X,\bar\kappa^2X\}
=
q^{H+1}\bar\kappa^2X
\le q^{H+1}D,
\end{equation*}
for every $t\ge H+1$.
This proves~\eqref{eq:uniform-diff-state-bound}.

To verify the stated population dependence, observe that
\[
m_H
=
\tau\frac{1-q^H}{1-q}
\le\frac{\tau}{\bar\gamma},
\qquad
B_\Sigma
\le N\max_{j\in[N]}\|B_j\|_{\mathrm{op}}.
\]
Therefore, we obtain: 
\[
Y
\le
\frac{\bar\kappa W}{\bar\gamma}
\left(
1+
N\max_{j\in[N]}\|B_j\|_{\mathrm{op}}
\frac{\tau}{\bar\gamma}
\right).
\]
If $X_0,W,\bar\kappa,\tau,\bar\gamma^{-1}$ and
$\max_{j\in[N]}\|B_j\|_{\mathrm{op}}$
are uniformly bounded in $N$, this gives
$Y=\mathcal O(N)$.
The definitions of $X$ and $D$ then imply
$X,D=\mathcal O(N)$ as well and these estimates are uniform in $H$.
\end{proof}

\section{Proof of Theorem~\ref{thm:local-gpc-regret-indep}}
\label{sec:proof-thm-local-gpc-regret-indep}

Here, we give the proof of Theorem~\ref{thm:local-gpc-regret-indep},
which we restate here:

\thmlocal*

\begin{remark}
The notation~$\widetilde \calO$ in Theorem~\ref{thm:local-gpc-regret-indep} hides polynomial factors in $\gamma^{-1}_i, \kappa_i, \|B_i\|, G, d$ and logarithmic factors in~$T$\,. 
\end{remark}

\begin{proof}
Under Assumptions~\ref{as:convexity}, \ref{as:bounded-disturb} and~\ref{as:strong-stab}, we apply \cite[Theorem 5.1]{agarwal2019online} for each agent~$i \in [N]$. 
It remains to ensure that the considered perturbation
sequence~$\{\widetilde w_t\}$ in \eqref{eq:lds-cp0} also satisfies the boundedness condition of Assumption~\ref{as:bounded-disturb} using the boundedness of control inputs by $U$ as follows: 
\begin{equation}
\|\widetilde w_t\| = \left\|\sum\nolimits_{j \neq i} B_j u_t^j + w_t \right\| \leq \|w_t\| + \sum\nolimits_{j \neq i} \|B_j\| \cdot \|u_t^j\| \leq W + (N-1) U(\max_{j} \|B_j\|)\,,
\end{equation}
where the last inequality follows from using boundedness of the control inputs of all the agents together with the bounded disturbances assumption (Assumption~\ref{as:bounded-disturb}). 

Selecting a step size $\eta = \Theta(1/(G \widetilde W^2 \sqrt{T}))$, where $\widetilde W = W + (N-1) U(\max_{j} \|B_j\|)$, and a (per-agent) memory length $H \ge \log(\kappa_i T)/\gamma_i$, we obtain the desired regret for any~$T \geq H+1,$ 
\begin{equation}
\reg^{H+1:T}_i(\mathcal{A}_i, \{u_t^{-i}\}, \Pi_i^{\text{lin}}) = \widetilde \calO(U^2N^2 \sqrt{T})\,. 
\end{equation}
This concludes the proof. 
\end{proof}

\section{Proof of Theorem~\ref{thm:local-gpc-regret-cp+}}
\label{sec:proof-thm-local-gpc-regret-cp+}

We prove Theorem~\ref{thm:local-gpc-regret-cp+}, restated here.

\thmglobal*

Fix a player $i$ and a realized trajectory. Opponents' realized
controls, costs, and disturbances are held fixed in every unilateral
counterfactual below. In particular, opponents' feedback policies
are not rerun on the counterfactual state.
Recall the following useful notation: 
\begin{equation}
\label{eq:aggregate-proof-notation}
\begin{gathered}
A_i=A-B_iK_i,\quad q_i=1-\gamma_i,\quad
b_i=\|B_i\|_{\mathrm{op}},\quad
\nu_i=\|K_i\|_{\mathrm{op}},\quad h=H+1,\\
\bar q=1-\bar\gamma,\quad a_p=2\bar\kappa^2\bar q^p,
\quad m_H=\sum_{p=0}^{H-1}a_p,
\quad B_\Sigma=\sum_{j=1}^N\|B_j\|_{\mathrm{op}}.
\end{gathered}
\end{equation}
Define the following constants, whose bounds are proved below:
\begin{equation}
\label{eq:aggregate-proof-state-constants}
\begin{gathered}
X=\bar\kappa X_0+
\frac{\bar\kappa W}{\bar\gamma}(1+B_\Sigma m_H),
\qquad
R_i=\frac{2\kappa_i b_iWm_H}{\gamma_i},
\qquad X_i=X+R_i,\\
D_i^{\mathrm{sa}}=(1+\nu_i)(1+\kappa_i)X_i+Wm_H,
\qquad E_i=(1+\nu_i)\kappa_iX_i.
\end{gathered}
\end{equation}
The proof uses a constant $\Lambda_i$ such that
\begin{equation}
\label{eq:aggregate-proof-local-lip}
|c_t^i(x,u)-c_t^i(x',u')|
\le\Lambda_i\bigl(\|x-x'\|_2+\|u-u'\|_2\bigr)
\end{equation}
whenever all four state/action vectors have norm at most
$D_i^{\mathrm{sa}}$. For Theorem~\ref{thm:local-gpc-regret-cp+},
Assumption~\ref{as:convexity} gives
$\Lambda_i=GD_i^{\mathrm{sa}}$.
We keep $\Lambda_i$ explicit so the same argument also applies
under Assumption~\ref{as:lip-costs} with $\Lambda_i=\bar L$.

\begin{remark}
The population rates below keep the dimensions,
$X_0,W,G,\bar\kappa,\bar\gamma^{-1}$,
$\max_j\|B_j\|_{\mathrm{op}}$, the individual strong-stability certificates, and $\kappa_*,\gamma_*^{-1}$ uniformly bounded as $N$ varies.
Factors depending polynomially on these quantities and
polylogarithmically on $N,T$ are suppressed in
$\widetilde{\mathcal O}$.
\end{remark}

\begin{remark}[The initial evaluation interval]
For every fixed DAC comparator $Q$, the trajectory comparison
proved below gives: 
\[
\|x_t^Q-x_t\|_2+\|u_t^Q-u_t^i\|_2
\le P_i:=(1+\nu_i)R_i+2Wm_H.
\]
Consequently,
$\reg_i^{1:H}(\mathcal A_i,\{u_t^{-i}\},\Pi_i^{\mathrm{DAC}})
\le H\Lambda_iP_i$.
Here $P_i=O(1)$, so the prefix is $O(NH)$ for
Theorem~\ref{thm:local-gpc-regret-cp+} and $O(H)$ under
Assumption~\ref{as:lip-costs}.
Together with Lemma~\ref{lem:time-regret-decomp}, these estimates
also preserve the corresponding full-horizon rates.
\end{remark}

\subsection{Regret decomposition and proof overview}
\label{subsec:regret-decomp-proof-thm}

Let $x_t^Q,u_t^Q$ denote the unilateral counterfactual obtained by
using the fixed parameter $Q\in\mathcal M_i$ from time zero,
with the same initial state and the opponents' realized controls
held fixed. Let $C_t(Q) :=c_t^i(x_t^Q,u_t^Q)$.
For the interval $h=H+1,\ldots,T$, define the shorthand: 
\begin{equation}
\label{eq:reg-H-T}
\reg_i^{h:T}:=
\reg_i^{h:T}(\mathcal A_i,\{u_t^{-i}\},\Pi_i^{\mathrm{DAC}})
=\sum_{t=h}^T c_t^i(x_t,u_t^i)
-\min_{Q\in\mathcal M_i}\sum_{t=h}^T C_t(Q).
\end{equation}
Use the unilateral memory loss from
Section~\ref{subsec:notation-counterfa-state-action}, with external
input $d_s=w_s+\sum_{j\ne i}B_ju_s^j$, DAC perturbation sequence $r_s=w_s$,
and diagonal $f_t^i(Q)=\ell_t^i(Q,\ldots,Q)$.
Adding and subtracting the surrogate costs gives: 
\begin{align}
\label{eq:regret-decomp}
\reg_i^{h:T}
={}&\underbrace{\sum_{t=h}^T
\left[c_t^i(x_t,u_t^i)-
\ell_t^i(M_{i,t-h:t})\right]}_{\text{actual trajectory approximation error}}
\nonumber\\
&+\underbrace{\sum_{t=h}^T\ell_t^i(M_{i,t-h:t})
-\min_{Q\in\mathcal M_i}\sum_{t=h}^T f_t^i(Q)}_
{\text{online gradient descent with memory regret}}
\nonumber\\
&+\underbrace{\min_{Q\in\mathcal M_i}\sum_{t=h}^T f_t^i(Q)
-\min_{Q\in\mathcal M_i}\sum_{t=h}^T C_t(Q)}_
{\text{comparator approximation error}}.
\end{align}
The minima exist because the parameter domain is compact and
the costs are continuous functions of the parameter.
Section~\ref{subsec:counterfa-state-action-error} bounds the first
and third terms using individual stability of $A_i$.
Section~\ref{subsec:ogd-regret-bound} bounds the middle term and
combines the estimates. Joint stability is used to bound the
actual trajectory. 

\subsection{Counterfactual state and action deviation error}
\label{subsec:counterfa-state-action-error}

For any feasible parameter sequence $Q_s$, the unilateral dynamics
are: 
\begin{equation}
\label{eq:aggregate-proof-unilateral-recurrence}
x_{s+1}^Q=A_ix_s^Q+
B_i\sum_{p=1}^H Q_s^{[p-1]}w_{s-p}+d_s,
\qquad x_0^Q=x_0.
\end{equation}
Under the simultaneous DAC policies of the actual agents,
Proposition~\ref{prop:bounds-states-actions} gives
$\|x_t\|_2\le X$ with $X$ as in
\eqref{eq:aggregate-proof-state-constants}.
Subtracting the actual recurrence from
\eqref{eq:aggregate-proof-unilateral-recurrence} cancels $d_s$:
\[
x_{s+1}^Q-x_{s+1}
=A_i(x_s^Q-x_s)
+B_i\sum_{p=1}^H
(Q_s^{[p-1]}-M_{i,s}^{[p-1]})w_{s-p}.
\]
The correction difference has norm at most $2Wm_H$.
Since the initial difference is zero, individual strong
stability therefore yields: 
\begin{equation}
\label{eq:aggregate-proof-counterfactual-distance}
\|x_t^Q-x_t\|_2
\le2\kappa_i b_iWm_H\sum_{r\ge0}q_i^r
=R_i,
\qquad
\|x_t^Q\|_2\le X_i.
\end{equation}
This holds for every feasible sequence $Q_s$, including any
extension of a truncated parameter window to a full sequence.

The correct ideal state and action are as follows: 
\begin{align}
\label{eq:aggregate-proof-unilateral-surrogate}
y_t^i(Q_{t-h:t-1})
&=\sum_{r=1}^h A_i^{r-1}
\left[d_{t-r}+B_i\sum_{p=1}^H
Q_{t-r}^{[p-1]}w_{t-r-p}\right],\nonumber\\
v_t^i(Q_{t-h:t})
&=-K_i y_t^i(Q_{t-h:t-1})
+\sum_{p=1}^H Q_t^{[p-1]}w_{t-p}.
\end{align}
Set $d_s=w_s=0$ for negative indices.
For $t\ge h$, unrolling
\eqref{eq:aggregate-proof-unilateral-recurrence} gives
\[
x_t^Q-y_t^i(Q_{t-h:t-1})=A_i^h x_{t-h}^Q.
\]
For $1\le t<h$, zero padding instead gives
$y_t^i(Q_{t-h:t-1})=x_t^Q-A_i^t x_0$.
Thus all true and ideal states have norm at most
$(1+\kappa_i)X_i$, and their actions have norm at most
$\nu_i(1+\kappa_i)X_i+Wm_H$.
All these quantities are bounded by $D_i^{\mathrm{sa}}$,
including early rounds and mixed feasible parameter histories.
For the scored rounds $t\ge h$,
\begin{equation}
\label{eq:aggregate-proof-truncation-distance}
\|x_t^Q-y_t^i(Q_{t-h:t-1})\|_2
+\|u_t^Q-v_t^i(Q_{t-h:t})\|_2
\le E_iq_i^h.
\end{equation}
Combining \eqref{eq:aggregate-proof-local-lip} and
\eqref{eq:aggregate-proof-truncation-distance} proves
\begin{equation}
\label{eq:lip-cost-proof-thm}
\left|c_t^i(x_t,u_t^i)-
\ell_t^i(M_{i,t-h:t})\right|
\le\Lambda_iE_iq_i^h.
\end{equation}
Writing $n=T-H$, summing the previous inequality gives: 
\begin{equation}
\label{eq:counterfactual-error-bound-reg}
\sum_{t=h}^T
\left[c_t^i(x_t,u_t^i)-\ell_t^i(M_{i,t-h:t})\right]
\le\Lambda_iE_i n q_i^h.
\end{equation}
The same estimate holds uniformly for every fixed comparator:
\[
\left|\sum_{t=h}^T f_t^i(Q)-\sum_{t=h}^T C_t(Q)\right|
\le\Lambda_iE_i n q_i^h.
\]
Taking a minimizer of the true comparator cost, or directly
comparing the two minima using this uniform inequality, gives
\begin{equation}
\label{eq:bound-counterfa-error-final}
\sum_{t=h}^T
\left[c_t^i(x_t,u_t^i)-\ell_t^i(M_{i,t-h:t})\right]
+\min_{Q\in\mathcal M_i}\sum_{t=h}^T f_t^i(Q)
-\min_{Q\in\mathcal M_i}\sum_{t=h}^T C_t(Q)
\le2\Lambda_iE_i n q_i^h.
\end{equation}
Finally, subtracting the DAC action formulas and using
\eqref{eq:aggregate-proof-counterfactual-distance} yields
the bound $P_i$ in the prefix remark.

\subsection{Online gradient descent with memory regret bound}
\label{subsec:ogd-regret-bound}

All projections and gradients use block Frobenius geometry.
The DAC domain has diameter at most: 
\begin{equation}
\label{eq:aggregate-proof-diameter}
D_{0,i}=2\left(\min(k_i,d)
\sum_{p=0}^{H-1}a_p^2\right)^{1/2}.
\end{equation}
Indeed, a block difference has operator norm at most $2a_p$,
and its Frobenius norm is at most
$2\sqrt{\min(k_i,d)}a_p$. Squaring these bounds and summing
over the blocks proves the claim.

The following two subsections establish the constants
\begin{equation}
\label{eq:aggregate-proof-oco-constants}
\begin{aligned}
L_i^{\mathrm{mem}}
&=\Lambda_iW\sqrt H
\max\{1,(1+\nu_i)\kappa_i b_i\},\\
G_{0,i}
&=\Lambda_iW\sqrt H
\left(1+\frac{(1+\nu_i)\kappa_i b_i}{\gamma_i}\right)
\end{aligned}
\end{equation}
for coordinate-wise Lipschitzness of $\ell_t^i$ and the gradient
bound for $f_t^i$. These bounds hold for every $t\ge1$.
Convexity follows from the affine unilateral state and action maps.
Theorem~\ref{thm:ogdm-regret}, applied on the interval
$h,\ldots,T$, therefore gives: 
\begin{equation}
\label{eq:regret-ogd-bound-proof}
\sum_{t=h}^T\ell_t^i(M_{i,t-h:t})
-\min_{Q\in\mathcal M_i}\sum_{t=h}^T f_t^i(Q)
\le\frac{D_{0,i}^2}{2\eta}
+\frac{\eta n}{2}
\left[G_{0,i}^2+
L_i^{\mathrm{mem}}G_{0,i}h(h+1)\right].
\end{equation}
For completeness, projection nonexpansiveness implies
$\|M_{i,s+1}-M_{i,s}\|_F\le\eta G_{0,i}$ for $s\ge1$.
The convention $M_{i,0}=M_{i,1}$ gives the same bound at $s=0$.
Thus the discrepancy between a memory loss and its diagonal is
at most: 
\[
L_i^{\mathrm{mem}}
\sum_{k=1}^h\|M_{i,t-k}-M_{i,t}\|_F
\le\eta L_i^{\mathrm{mem}}G_{0,i}\frac{h(h+1)}2.
\]
Adding this to the ordinary OGD bound
$D_{0,i}^2/(2\eta)+\eta nG_{0,i}^2/2$
also proves \eqref{eq:regret-ogd-bound-proof} directly.
In particular, the updates before the evaluation interval are
accounted for and the parameter history is not reset at $h$.

Combining \eqref{eq:regret-decomp},
\eqref{eq:bound-counterfa-error-final}, and
\eqref{eq:regret-ogd-bound-proof} gives the finite-time estimate: 
\begin{equation}
\label{eq:aggregate-proof-master}
\reg_i^{H+1:T}\le 
2\Lambda_iE_i(T-H)q_i^h+
\frac{D_{0,i}^2}{2\eta}
+\frac{\eta(T-H)}2
\left[G_{0,i}^2+
L_i^{\mathrm{mem}}G_{0,i}h(h+1)\right].
\end{equation}

\noindent\textbf{End of Proof of Theorem~\ref{thm:local-gpc-regret-cp+}.}
Set $\Lambda_i=GD_i^{\mathrm{sa}}$ and use: 
\[
\eta=\frac1{N\sqrt T},\qquad
H=\max\left\{1,
\left\lceil\frac{\log(2\kappa_*N^2\sqrt T)}{\gamma_*}
\right\rceil\right\},
\]
where the given stability bounds satisfy: 
\begin{equation}
\label{eq:stability-envelopes}
\kappa_*\ge
\max\{\bar\kappa,\max_{j\in[N]}\kappa_j\},
\qquad
0<\gamma_*\le
\min\{\bar\gamma,\min_{j\in[N]}\gamma_j\}.
\end{equation} 
Under the uniform parameter convention stated above,
$m_H=O(1)$, $X_i,D_i^{\mathrm{sa}},E_i=O(N)$,
$D_{0,i}=O(1)$, and
$L_i^{\mathrm{mem}},G_{0,i}=O(N\sqrt H)$.
Moreover,
\[
q_i^h\le e^{-\gamma_*H}
\le\frac1{2\kappa_*N^2\sqrt T}.
\]
The truncation term in \eqref{eq:aggregate-proof-master} is
therefore $O(\sqrt T)$, and the memory-OGD term is bounded by
\begin{equation}
\label{eq:regret-ogd-bound-final}
O\bigl(N\sqrt T(1+H^3)\bigr)
=\widetilde{\mathcal O}(N\sqrt T).
\end{equation}
This proves the claimed suffix regret bound for $T\ge H+1$.
The prefix remark gives the corresponding full-horizon rate.

\subsubsection{Coordinate-wise loss Lipschitzness}
\label{subsec:coordinate-wise-lip}

\begin{lemma}[Coordinate-wise loss Lipschitzness]
\label{lem:coordinate-wise-lip}
Let $t\ge1$ and $k\in\{0,\ldots,h\}$.
Consider two feasible parameter histories that differ only in
the entry at time $t-k$, replacing $Q$ by $\widetilde Q$.
With the opponents' realized controls held fixed,
\[
\left|\ell_t^i(\ldots,Q,\ldots)
-\ell_t^i(\ldots,\widetilde Q,\ldots)\right|
\le L_i^{\mathrm{mem}}\|Q-\widetilde Q\|_F,
\]
where $L_i^{\mathrm{mem}}$ is given in
\eqref{eq:aggregate-proof-oco-constants}.
\end{lemma}

\begin{proof}
Let $\Delta Q=Q-\widetilde Q$, and let $\Delta y_t^i$ and
$\Delta v_t^i$ be the corresponding ideal state and action
differences. The unilateral rollout
\eqref{eq:aggregate-proof-unilateral-surrogate} gives: 
\begin{equation}
\label{eq:diff-counterfactual-states-expression}
\Delta y_t^i=
\begin{cases}
0,&k=0,\\[1mm]
A_i^{k-1}B_i\displaystyle\sum_{p=1}^H
\Delta Q^{[p-1]}w_{t-k-p},&1\le k\le h.
\end{cases}
\end{equation}
Moreover,
$
\Delta v_t^i=-K_i\Delta y_t^i
+\mathbf 1_{\{k=0\}}
\sum_{p=1}^H\Delta Q^{[p-1]}w_{t-p}.
$ 
The block norm inequality
\[
\sum_{p=1}^H\|\Delta Q^{[p-1]}\|_{\mathrm{op}}
\le\sum_{p=1}^H\|\Delta Q^{[p-1]}\|_F
\le\sqrt H\|\Delta Q\|_F
\]
implies
\[
\|\Delta y_t^i\|_2+\|\Delta v_t^i\|_2
\le W\sqrt H\max\{1,(1+\nu_i)\kappa_i b_i\}
\|\Delta Q\|_F.
\]
All the states and actions involved lie in the region bounded
by $D_i^{\mathrm{sa}}$. Applying
\eqref{eq:aggregate-proof-local-lip} completes the proof.
This argument includes the oldest coordinate $k=h$ and remains
valid when $B_i=0$.
\end{proof}

\subsubsection{Gradient loss boundedness}
\label{subsec:grad-bound}

\begin{lemma}[Diagonal parameter-gradient bound]
\label{lem:grad-bound}
For every $t\ge1$ and every $Q\in\mathcal M_i$,
\[
\|\nabla f_t^i(Q)\|_F\le G_{0,i},
\]
where $G_{0,i}$ is defined in
\eqref{eq:aggregate-proof-oco-constants}.
\end{lemma}

\begin{proof}
The derivative of the diagonal state map in direction $\Delta Q$
is given by: 
\[
\mathcal J_{i,t}[\Delta Q]
=\sum_{r=1}^h A_i^{r-1}B_i
\sum_{p=1}^H\Delta Q^{[p-1]}w_{t-r-p}.
\]
Its Euclidean norm is at most
$W\sqrt H\kappa_i b_i\|\Delta Q\|_F/\gamma_i$.
The diagonal action derivative is
$-K_i\mathcal J_{i,t}[\Delta Q]
+\sum_{p=1}^H\Delta Q^{[p-1]}w_{t-p}$.
The differentiable cost satisfying
\eqref{eq:aggregate-proof-local-lip} has partial gradients
bounded by $\Lambda_i$ on the relevant region.
For $\Lambda_i=GD_i^{\mathrm{sa}}$ this follows directly from
Assumption~\ref{as:convexity}, and for $\Lambda_i=\bar L$ it
follows from Assumption~\ref{as:lip-costs}.
The chain rule therefore bounds every directional derivative
of $f_t^i$ by: 
\[
\Lambda_iW\sqrt H
\left(1+\frac{(1+\nu_i)\kappa_i b_i}{\gamma_i}\right)
\|\Delta Q\|_F.
\]
Taking the supremum over directions of unit block Frobenius
norm proves the claim.
\end{proof}

\section{Proof of Theorem~\ref{thm:local-gpc-regret-cp+-lip}}
\label{sec:proof-thm-local-gpc-regret-cp+-lip}

Here, we develop the proof of
Theorem~\ref{thm:local-gpc-regret-cp+-lip}, restated here:

\thmgloballip*

\begin{proof}
The proof follows the same strategy as the proof of
Theorem~\ref{thm:local-gpc-regret-cp+}.
We explain the modifications needed when the cost functions
satisfy the uniform Lipschitz assumption~\ref{as:lip-costs}.
The state and action bounds established in
Section~\ref{sec:proof-thm-local-gpc-regret-cp+}
depend only on the dynamics, the feasible DAC parameters,
the initial-state bound, and the disturbance bound.
Consequently, these trajectory bounds remain unchanged.
The modification concerns how trajectory differences and
parameter changes are converted into cost differences.

Fix a player $i \in [N]$.
We use the unilateral counterfactuals and surrogates from
Section~\ref{sec:proof-thm-local-gpc-regret-cp+}, with the
opponents' realized controls held fixed.
Recall the notation: 
\[
h=H+1,\qquad
q_i=1-\gamma_i,\qquad
b_i=\|B_i\|_{\mathrm{op}},\qquad
\nu_i=\|K_i\|_{\mathrm{op}},
\]
and the constants
$X_i,D_i^{\mathrm{sa}},E_i,D_{0,i}$ defined there.
Recall also the notation: 
\[
C_t(Q):=c_t^i(x_t^Q,u_t^Q),
\qquad
f_t^i(Q):=\ell_t^i(Q,\ldots,Q),
\qquad Q\in\mathcal M_i,
\]
for the true unilateral comparator cost and the diagonal
surrogate cost, respectively.

Under Assumption~\ref{as:lip-costs}, we have
\[
|c_t^i(x,u)-c_t^i(x',u')|
\le
\bar L\bigl(\|x-x'\|_2+\|u-u'\|_2\bigr).
\]
Thus the local cost bound
\eqref{eq:aggregate-proof-local-lip}
holds with $\Lambda_i=\bar L$, replacing the
radius-dependent factor
$\Lambda_i=GD_i^{\mathrm{sa}}$
used in the preceding theorem.

Recall the regret decomposition
\eqref{eq:regret-decomp} in
Section~\ref{subsec:regret-decomp-proof-thm}.
We adapt its approximation terms and its online gradient descent regret term as follows.

\begin{itemize}[
    leftmargin=1em,
    rightmargin=1em,
]

\item
\textbf{Counterfactual state and action deviation errors.}
We first bound the two approximation terms in the regret
decomposition: the discrepancy between the actual cost and
its memory surrogate, and the discrepancy between the
surrogate comparator cost and the true comparator cost.

The trajectory estimate
\eqref{eq:aggregate-proof-truncation-distance}
bounds the sum of the state and action truncation errors by
$E_iq_i^h$ for every $t\ge h$.
Applying the uniform Lipschitz assumption therefore gives
\[
\left|
c_t^i(x_t,u_t^i)
-\ell_t^i(M_{i,t-h:t})
\right|
\le
\bar L E_iq_i^h.
\]
The same argument applies to every fixed feasible
comparator $Q\in\mathcal M_i$:
\[
|f_t^i(Q)-C_t(Q)|
\le
\bar L E_iq_i^h.
\]
In both inequalities, the trajectory estimate is unchanged;
only the cost factor $\Lambda_i$ is replaced by~$\bar L$.

There are $T-H$ rounds in the evaluation interval
$\{H+1,\ldots,T\}$.
Summing the first inequality therefore bounds the actual
trajectory approximation term by
$\bar L E_i(T-H)q_i^h$.
Moreover, the second inequality is uniform over $Q$, so
\[
\min_{Q\in\mathcal M_i}\sum_{t=h}^T f_t^i(Q)
-
\min_{Q\in\mathcal M_i}\sum_{t=h}^T C_t(Q)
\le
\bar L E_i(T-H)q_i^h.
\]
Indeed, this follows by evaluating the surrogate sum at a
minimizer of the true comparator sum and applying the
uniform approximation bound.

Consequently, the sum of the first and third terms in
\eqref{eq:regret-decomp} satisfies
\begin{multline}
\label{eq:new-bound-counterfa-error}
\sum_{t=h}^T
\left[
c_t^i(x_t,u_t^i)
-\ell_t^i(M_{i,t-h:t})
\right] 
+
\min_{Q\in\mathcal M_i}\sum_{t=h}^T f_t^i(Q)
-
\min_{Q\in\mathcal M_i}\sum_{t=h}^T C_t(Q)\\
\le
2\bar L E_i(T-H)(1-\gamma_i)^{H+1}.
\end{multline}
Under the uniform problem parameter bounds stated in
Section~\ref{sec:proof-thm-local-gpc-regret-cp+},
we have $E_i=\mathcal O(N)$.
Thus the approximation bound has a coefficient growing
linearly in $N$.
The additional state-radius factor present when
$\Lambda_i=GD_i^{\mathrm{sa}}$ has been removed.

\item
\textbf{Online gradient descent with memory regret bound.}
The parameter domain, the block Frobenius geometry, and the
projected gradient update are unchanged.
In particular, the same diameter bound $D_{0,i}$ applies.

Recall from~\eqref{eq:regret-ogd-bound-proof} that the
online-gradient contribution is bounded by: 
\begin{equation}
\label{eq:new-regret-bound-error-lip}
\sum_{t=h}^T\ell_t^i(M_{i,t-h:t})
-
\min_{Q\in\mathcal M_i}\sum_{t=h}^T f_t^i(Q) 
\le
\frac{D_{0,i}^2}{2\eta}
+
\frac{\eta(T-H)}{2}
\left[
G_{0,i}^2+
L_i^{\mathrm{mem}}G_{0,i}(H+1)(H+2)
\right].
\end{equation}
The factor $(H+1)(H+2)$ is $h(h+1)$, because the
memory loss depends on the parameter history
$M_{i,t-h:t}$ with $h=H+1$.

It remains to specify the coordinate-wise Lipschitz
constant $L_i^{\mathrm{mem}}$ and the diagonal gradient
bound $G_{0,i}$ under Assumption~\ref{as:lip-costs}.
The following two modifications suffice.

\begin{enumerate}[label=(\roman*),leftmargin=*]

\item
\textit{Coordinate-wise Lipschitzness.}
In the proof of Lemma~\ref{lem:coordinate-wise-lip},
the effect of changing a single parameter in the history
is first bounded at the level of the ideal state and
action. These bounds depend on the individual
closed-loop matrix $A_i=A-B_iK_i$, the disturbance
bound, and the DAC parameter geometry.

The cost difference is then bounded using
\eqref{eq:aggregate-proof-local-lip}.
Replacing $\Lambda_i$ by $\bar L$ gives: 
\[
L_i^{\mathrm{mem}}
=
\bar LW\sqrt H
\max\{1,(1+\nu_i)\kappa_i b_i\}.
\]
Thus, when two feasible parameter histories differ only
in one coordinate, their memory losses differ by at most
$L_i^{\mathrm{mem}}$ times the block Frobenius distance
between those two parameter values. 
In particular, this Lipschitz constant no longer
contains the state-radius factor $D_i^{\mathrm{sa}}$.

\item
\textit{Diagonal gradient boundedness.}
Similarly, the proof of Lemma~\ref{lem:grad-bound}
bounds the sensitivity of the diagonal ideal state and
action to the fixed parameter $Q$.
Applying the same argument with $\Lambda_i=\bar L$
gives: 
\[
\|\nabla f_t^i(Q)\|_F
\le
G_{0,i}
:=
\bar LW\sqrt H
\left(
1+\frac{(1+\nu_i)\kappa_i b_i}{\gamma_i}
\right),
\qquad Q\in\mathcal M_i.
\]
This gradient bound also contains no state-radius factor.

\end{enumerate}

Under the same uniform problem-parameter bounds,
these constants satisfy
\[
D_{0,i}=\mathcal O(1),
\qquad
L_i^{\mathrm{mem}}=\mathcal O(\sqrt H),
\qquad
G_{0,i}=\mathcal O(\sqrt H).
\]
Therefore, the online-gradient bound has no explicit
polynomial dependence on $N$.
Its dependence on the population size through the chosen
memory length $H$ will be only polylogarithmic.

\end{itemize}

\medskip
\noindent\textbf{Choice of memory and step size.}
Combining the two parts above with
\eqref{eq:regret-decomp} gives: 
\[
\reg_i^{H+1:T}
\le 
2\bar L E_i(T-H)q_i^{H+1}
+
\frac{D_{0,i}^2}{2\eta}
+
\frac{\eta(T-H)}2
\left[
G_{0,i}^2+
L_i^{\mathrm{mem}}G_{0,i}(H+1)(H+2)
\right].
\]
This is the estimate
\eqref{eq:aggregate-proof-master}
with $\Lambda_i=\bar L$.
Then, choose
\[
\eta=\frac1{\sqrt T},
\qquad
H=
\max\left\{
1,
\left\lceil
\frac{\log(2\kappa_*N\sqrt T)}{\gamma_*}
\right\rceil
\right\},
\]
where $\kappa_*,\gamma_*$ are the stability bounds
in~\eqref{eq:stability-envelopes}. 
In particular, $\gamma_*\le\gamma_i$.
Since $1-z\le e^{-z}$ for $z\ge0$, this choice gives: 
\[
q_i^{H+1}
=
(1-\gamma_i)^{H+1}
\le
e^{-\gamma_i(H+1)}
\le
e^{-\gamma_*H}
\le
\frac1{2\kappa_*N\sqrt T}.
\]

For the approximation term, using $E_i=\mathcal O(N)$
and $T-H\le T$, we obtain: 
\[
2\bar L E_i(T-H)q_i^{H+1}
\le
\frac{\bar L E_i(T-H)}
{\kappa_*N\sqrt T}
=
\mathcal O(\sqrt T).
\]

For the online gradient descent regret term, the bounds on its constants
give: 
\[
G_{0,i}^2=\mathcal O(H),
\qquad
L_i^{\mathrm{mem}}G_{0,i}(H+1)(H+2)
=\mathcal O(H^3).
\]
Consequently, with $\eta=1/\sqrt T$,
\begin{align*}
&\frac{D_{0,i}^2}{2\eta}
+
\frac{\eta(T-H)}2
\left[
G_{0,i}^2+
L_i^{\mathrm{mem}}G_{0,i}(H+1)(H+2)
\right]\\
&\qquad\le
\frac{D_{0,i}^2\sqrt T}{2}
+
\frac{\sqrt T}{2}
\left[
G_{0,i}^2+
L_i^{\mathrm{mem}}G_{0,i}(H+1)(H+2)
\right]\\
&\qquad=
\mathcal O\bigl(\sqrt T(1+H^3)\bigr).
\end{align*}

For fixed stability envelopes, the selected $H$ is
logarithmic in $N$ and $T$.
Combining the preceding estimates proves, for
$T\ge H+1$,
\[
\reg_i^{H+1:T}
\bigl(
\mathcal A_i,\{u_t^{-i}\},\Pi_i^{\mathrm{DAC}}
\bigr)
\le
\widetilde{\mathcal O}(\sqrt T).
\]

\medskip
\noindent\textbf{The initial evaluation interval.}
The same full-horizon rate follows by including the
initial rounds.
Recall from
Section~\ref{sec:proof-thm-local-gpc-regret-cp+}
that, for every fixed feasible comparator $Q$,
\[
\|x_t^Q-x_t\|_2\le R_i,
\qquad
\|u_t^Q-u_t^i\|_2
\le \nu_iR_i+2Wm_H,
\]
where
$R_i=\frac{2\kappa_i b_iWm_H}{\gamma_i}.$ 
Thus, it follows that: 
\[
\|x_t^Q-x_t\|_2+\|u_t^Q-u_t^i\|_2
\le
P_i:=(1+\nu_i)R_i+2Wm_H.
\]
Then, Assumption~\ref{as:lip-costs} implies: 
\[
|c_t^i(x_t,u_t^i)-C_t(Q)|
\le\bar L P_i
\]
uniformly over $Q$.
Summing over the first $H$ rounds and taking the best
fixed comparator gives
\[
\reg_i^{1:H}
\bigl(
\mathcal A_i,\{u_t^{-i}\},\Pi_i^{\mathrm{DAC}}
\bigr)
\le H\bar L P_i.
\]
Since $m_H$ and $R_i$ are uniformly bounded under the
stated parameter convention, $P_i=\mathcal O(1)$.
The prefix contribution is therefore $\mathcal O(H)$.
By Lemma~\ref{lem:time-regret-decomp},
\[
\reg_i^{1:T}
\le
\reg_i^{1:H}+\reg_i^{H+1:T}
\le
\widetilde{\mathcal O}(\sqrt T).
\]

The hidden factors depend polynomially on
$X_0,W,\bar L,\bar\kappa,\bar\gamma^{-1},
\kappa_i,\gamma_i^{-1},\kappa_*,\gamma_*^{-1}$,
the input-matrix bounds, and the dimensions, and only
polylogarithmically on $N$ and $T$.
\end{proof}

\section{Proof of Theorem~\ref{thm:reg-lb}}
\label{app:lower-bound}

In this section, we develop the proof of 
Theorem~\ref{thm:reg-lb}, which we restate here:

\thmreglb*

The horizon lower bound is inherited from the classical random sign argument for online learning; see~\cite[Theorem~4.1]{arora2012multiplicative}
and~\cite[Section~3.2, Theorem~3.2]{hazan-et-al16intro-oco}.
We give a self-contained embedding that verifies its compatibility with our cost assumptions, controller benchmarks, and unrestricted learner
controls. The statistical lower bound mechanism is standard.

\noindent\textbf{Why the embedding needs a proof.} 
While the lower bound relies on the same core argument as in standard online learning, we cannot immediately apply the classical result since the following properties need to be verified:
\begin{enumerate}
\item \textbf{Controller realization and counterfactual states.}
Fixed comparison actions must be represented by admissible controllers.
Moreover, the learner and comparator generally induce different states.
We must construct an instance where the desired action correspondence and
counterfactual cost comparison are valid.
\item \textbf{Admissible cost functions.}
Nonzero affine cost functions fail the origin-centered growth and gradient conditions in Assumption~\ref{as:convexity}. We must preserve the classical affine losses
on the comparison slice while extending them to jointly convex, continuously differentiable costs satisfying the global assumptions.

\item \textbf{All allowed learner controls.}
The comparison interval does not restrict the learner's control space.
After making the costs nonlinear outside that interval, we must verify that every real control still has the required conditional expected cost lower bound. While the standard example relies on linearity of the loss functions and sign cancellation, we extend the argument here to our nonlinear cost setting using convexity. 

\item \textbf{DAC feasibility.}
For the DAC conclusion, the same comparison trajectories must be realized by feasible memory-$H$ parameters with the fixed supplied baseline. 
\end{enumerate}

\noindent\textbf{Notation and timing.}
Dynamics begin at time zero, and costs are evaluated on a prescribed
nonempty interval $I\subseteq\{1,\ldots,T\}$.
The current cost function is revealed only after the current control input is chosen. The lower bound remains valid even if the learner knows the dynamics, the entire deterministic disturbance sequence, and all past cost functions.
For a fixed disturbance sequence, recall that the regret incurred by player $i$ over the time interval $I$ is given by: 
\begin{equation}
R_i^I(\Pi)
:=
\sum_{t\in I}c_t^i(x_t,u_t^i)
-\min_{\pi\in\Pi}\sum_{t\in I}c_t^i(x_t^\pi,u_t^\pi).
\label{eq:lb-interval-regret}
\end{equation}
All counterfactual trajectories begin at the same initial state at time zero,
and opponents' realized control sequences are held fixed.

\begin{proof}
Fix $T,H,I,\mathcal A_i$, let $n=|I|$, and let $\xi$ denote the learner's
random seed.

\noindent\textbf{Step 0. Classical proof template and required embedding}

For independent uniform signs $\sigma_t\in\{-1,1\}$, the classical argument
uses $f_t(u)=\sigma_tu$ and the comparison interval $[-1/2,1/2]$.
A causal learner restricted to this interval has expected cumulative
loss zero since $\mathbb{E}[\sigma_t] = 0$, whereas, writing $S_I=\sum_{t\in I}\sigma_t$,
\[
\min_{u\in[-1/2,1/2]}\sum_{t\in I}\sigma_tu
=-\tfrac12|S_I|.
\]
Thus the expected regret is $\tfrac12\mathbb E|S_I|$.

We reproduce this calculation within the dynamical model, up to a
common per-round constant $b=3/4$. Specifically, we will prove
\begin{equation}
\mathbb E C_I^{\mathcal A}\ge \tfrac34n,
\qquad
\min_{\pi\in\Pi}C_I^\pi
\le \tfrac34n-\tfrac12|S_I|,
\label{eq:lb-template}
\end{equation}
where $C_I^{\mathcal A}$ and $C_I^\pi$ are the learner's and comparator's cumulative costs. Steps~1--4 establish these bounds with admissible dynamics, costs, and
controllers. 
Combining the above inequalities implies:  
\begin{equation}
\mathbb{E}_{\sigma, \xi}[R_i^I(\Pi)] \ge \frac 12 \mathbb{E}_{\sigma}[|S_{I}|]\,, \quad \Pi\in\{\Pi_i^{\mathrm{lin}},\Pi_i^{\mathrm{DAC}}\}\,.
\end{equation}

Steps~5--7 complete the standard fluctuation argument showing that $\mathbb{E}|S_{I}| = \Omega(\sqrt{|I|})$ and deducing the existence of the cost function from the lower bound in expectation. 

\subsubsection*{Step 1. Construction of the scalar dynamical system instance}

Take scalar states and scalar controls for every player, and choose
\begin{equation}
A=0,\qquad B_j=0\quad(j\in[N]),\qquad
w_t=1\quad(t\ge0),\qquad x_0=0.
\label{eq:lb-dynamics}
\end{equation}
As in the DAC definition, recall that $w_t=0$ for $t<0$.
Then $x_{t+1}=1$ for every $t\ge0$, so
\begin{equation}
x_t=x_t^\pi=1\,, \qquad \forall t\ge1\,, 
\label{eq:lb-common-state}
\end{equation}
for the learner and every comparator, independently of all controls.

The choice $B_j=0$ makes the counterfactual states coincide.
The choice $w_t=1$ ensures that homogeneous linear controllers can
represent nonzero constant actions: $u_t=-Kx_t=-K$ for $t\ge1$.
In contrast, zero disturbances and $x_0=0$ would make all such controllers
produce zero inputs. 

Let $K_j=0$ be the baseline controller for every player. Then  $Q=1$ and $L=0$ guarantee strong stability with $\kappa_j = \bar{\kappa} =1$ and $\gamma_j = \bar{\gamma} = \frac 12$. Moreover, the disturbance bound holds with $W =1$ since $|w_t| = 1\le 1$ for all $t \geq 0$. 
Opponents' controls affect neither the state nor aggregate feedback,
since $\sum_{j\ne i}B_ju_t^j=0$.
Their costs may be identically set to zero. The construction works under both information settings.

\subsubsection*{Step 2. Construction of the hard cost functions satisfying the assumptions}

We now construct hard cost functions that satisfy our assumptions while retaining the main properties of the classical affine losses used in the standard online learning proof at the common state $x=1$. 

Consider first the Huber function with threshold $\frac 12$: 
\begin{equation}
h(z)=
\begin{cases}
z^2,& |z|\le\tfrac12,\\
|z|-\tfrac14,& |z|>\tfrac12.
\end{cases}
\label{eq:lb-h}
\end{equation}

Sample independent uniform signs $\sigma_t\in\{-1,1\}$ for $t\in I$
before play, independently of $\xi$, and set
\begin{equation}
c_t^i(x,u)=
\begin{cases}
h(x+\sigma_tu),&t\in I,\\
0,&t\notin I.
\end{cases}
\label{eq:lb-costs}
\end{equation}
Set $c_0^i=0$ and the other players' costs to zero.
This defines an oblivious distribution over cost sequences.

The derivative $h'(z)$ equals $-1$ for $z<-1/2$, $2z$ for
$|z|\le1/2$, and $1$ for $z>1/2$.
It is continuous and nondecreasing, so $h$ is continuously
differentiable and convex. Moreover,
\begin{equation}
0\le h(z)\le z^2,\qquad
|h'(z)|\le2|z|,\qquad |h'(z)|\le1.
\label{eq:lb-h-bounds}
\end{equation}

Composition with the linear map $(x,u)\mapsto x+\sigma_tu$ makes
$c_t^i$ jointly convex and continuously differentiable.
If $|x| \le D$ and $|u|\le D$, then $|x+\sigma_tu|\le2D$, and
\[
0\le c_t^i(x,u)\le4D^2,\qquad
|\partial_xc_t^i(x,u)|,\ |\partial_uc_t^i(x,u)|\le4D.
\]
This verifies Assumption~\ref{as:convexity} with $\beta=G=4$.
Also,
\[
|c_t^i(x,u)-c_t^i(x',u')|
\le |x-x'|+|u-u'|,
\]
which verifies Assumption~\ref{as:lip-costs} with $\bar L=1$.

Finally, we have the exact identity: 
\begin{equation}
c_t^i(1,u)=\tfrac34+\sigma_tu\,, 
\qquad \forall t\in I\,, |u|\le\tfrac12.
\label{eq:lb-affine-slice}
\end{equation}
This is the standard loss plus a common constant.
The learner is not restricted to this slice, so its expected cost
still requires a bound valid for every real control.

\begin{remark}[On the choice of the function $h$]
For $|u|\le1/2$, the arguments $1+\sigma_tu$ lie in $[1/2,3/2]$.
We therefore seek a convex function $h$ that is affine with slope one
on that interval, while satisfying $h(0)=h'(0)=0$ to accommodate Assumption~\ref{as:convexity}. 
Joining the quadratic $z^2$ at $z=1/2$ to an affine branch of slope one
gives $z-1/4$: both its value and derivative agree with the quadratic at the joining point. Thus the quadratic core enforces the origin-centered conditions while
the positive affine branch preserves the classical comparison losses.
\end{remark}

\subsubsection*{Step 3. Lower bound on the expected learner cost for arbitrary controls}

Let $\xi$ denote the learner's random seed, containing all
its internal randomization and sampled independently of
the sign vector $(\sigma_s)_{s\in I}$.
For every $t\ge1$, define the $\sigma$-algebra: 
\[
\mathcal F_t^- :=
\sigma\!\left(\xi,(\sigma_s)_{s\in I,\ s<t}\right).
\]
In the constructed instance, the state and aggregate feedback
are deterministic, and all past cost feedback is determined
by the past signs and the learner's random seed.
Consequently, causality implies that $u_t^i$ is
$\mathcal F_t^-$-measurable.
For every $t\in I$, the current sign $\sigma_t$ is independent
of $\mathcal F_t^-$, so
\[
\mathbb P(\sigma_t=1\mid\mathcal F_t^-)
=
\mathbb P(\sigma_t=-1\mid\mathcal F_t^-)
=\frac12\,, 
\qquad\text{almost surely}.
\]

Using the $\mathcal F_t^-$-measurability of $u_t^i$
and the conditional distribution of $\sigma_t$, we obtain,
for every $t\in I$,
\begin{align}
\mathbb E[c_t^i(x_t,u_t^i)\mid\mathcal F_t^-]
&=\tfrac12\bigl[h(1+u_t^i)+h(1-u_t^i)\bigr]\notag\\
&\ge h\!\left(\frac{(1+u_t^i)+(1-u_t^i)}2\right)
=h(1)=\tfrac34.
\label{eq:lb-jensen}
\end{align}
For affine losses the classical proof uses exact sign cancellation;
for the admissible convex extension, Jensen's inequality gives a lower bound for every real input.

Summing~\eqref{eq:lb-jensen} gives
\begin{equation}
\mathbb E_{\sigma,\xi} C_I^{\mathcal A}\ge\tfrac34n,
\qquad
C_I^{\mathcal A}:=\sum_{t\in I}c_t^i(x_t,u_t^i).
\label{eq:lb-learner}
\end{equation}
The costs are nonnegative, so this remains valid if the expected learner
cost is infinite.

\subsubsection*{Step 4. Common upper bound on the comparator term for linear and DAC policies}

\paragraph{(a) Strongly stable linear comparators.}
For $s\in\{-1,1\}$ choose
\[
K_s= \frac s2.
\]
Then $|K_s|\le1$ and $A-B_iK_s=0$, so both controllers belong
to the $(1,1/2)$-strongly stable linear class.
For $t=0$, $u_0^{K_s} = - K_s x_0 = 0$ and for $t\ge1$,
\[
u_t^{K_s}=-K_sx_t^{K_s}= - \frac s2.
\]
By~\eqref{eq:lb-affine-slice},
\[
c_t^i(x_t^{K_s},u_t^{K_s})
=\tfrac34 - \tfrac{s\sigma_t}{2}\,, \qquad \forall t\in I.
\]
Consequently, with $S_I=\sum_{t\in I}\sigma_t$,
\begin{equation}
C_I^{K_s}=\tfrac34n - \tfrac{s}{2}S_I,
\qquad
\min_{s\in\{-1,1\}}C_I^{K_s}
=\tfrac34n-\tfrac12|S_I|.
\label{eq:lb-pair}
\end{equation}
The minimum over the full linear class is no larger than the minimum
over this pair.

\paragraph{(b) Fixed-baseline DAC comparators.}
For every $H\ge1$, set again $K_i=0$ and choose: 
\[
M_s^{[0]}=s/2,\qquad
M_s^{[p]}=0\,,\quad \forall 1\le p<H .
\]
These parameters satisfy the DAC feasibility conditions: $|M^{[p]}| \le 2 (\frac 12)^p$ for $0 \le p < H$. Indeed, at $\kappa=1$ and $\gamma=1/2$,
the first-block radius is at least $1$, whereas $|M_s^{[0]}|=1/2$, and every remaining block is zero. Since $w_t = 0$ for $t < 0$ and $x_0 = 0$, we also have $u_0^{M_s} = 0$. For $t \ge 1$, 
\[
u_t^{M_s}
=-K_ix_t^{M_s}
+\sum_{p=1}^{H}M_s^{[p-1]}w_{t-p}
= \frac s2 w_{t-1}
= \frac s2.
\]
Thus these DAC policies realize the same pair of comparison
trajectories and costs as the two linear controllers,
with the sign labels interchanged. 
In Setting~1, the effective disturbances also equal $w_t$ on this
instance, so the same calculation applies there.

It follows that, for either comparator class,
\begin{equation}
\min_{\pi\in\Pi}C_I^\pi
\le\tfrac34n-\tfrac12|S_I|,
\qquad
\Pi\in\{\Pi_i^{\mathrm{lin}},\Pi_i^{\mathrm{DAC}}\}.
\label{eq:lb-common-comparator}
\end{equation}

\subsubsection*{Step 5. Combining Step 3 (cumulative cost lower bound) and Step 4 (comparator upper bound)}

Define the common lower bound
\begin{equation}
D_I(\sigma,\xi)
:=
C_I^{\mathcal A}(\sigma,\xi)-\tfrac34n
+\tfrac12|S_I(\sigma)|\,,
\label{eq:lb-common-witness}
\end{equation}
where we make the dependence on randomness explicit using the notation $C_I^{\mathcal A}(\sigma,\xi):= C_I^{\mathcal A}$ and $S_I(\sigma) := S_I\,.$ 
By~\eqref{eq:lb-common-comparator}, we obtain for every realization,
\begin{equation}
R_i^I(\Pi_i^{\mathrm{lin}})\ge D_I(\sigma,\xi),
\qquad
R_i^I(\Pi_i^{\mathrm{DAC}})\ge D_I(\sigma,\xi).
\label{eq:lb-pathwise}
\end{equation}
Combining with~\eqref{eq:lb-learner} yields: 
\begin{equation}
\mathbb E_{\sigma,\xi}D_I(\sigma, \xi)
\ge\tfrac12\mathbb E_\sigma|S_I|.
\label{eq:lb-before-fluctuation}
\end{equation}
The common baseline $3n/4$ cancels, just as the baseline cancels
in the classical proof. What remains is the same random-sign fluctuation term $\tfrac12\mathbb E_\sigma|S_I|.$ 

\subsubsection*{Step 6. Elementary lower bound on random sign fluctuation: $E_\sigma|S_I| \ge \sqrt{n/3}$.}

This lower bound is standard in the literature and used as well in the classical online learning bound. We provide a full proof here for completeness. 

First, by independence of the random variables $\sigma_t$ which also has zero mean, and since $\sigma_t^2=1$, we have $\mathbb E S_I^2=n.$ 

The fourth moment of $S_I$ can be bounded as follows: 
\begin{equation}
\mathbb E S_I^4 
= \mathbb E\left[ \sum_{i,j,k,l} \sigma_i \sigma_j \sigma_k \sigma_l \right]
= \sum_{t \in I} \mathbb{E}[\sigma_t^4] + 6 \sum_{i,j \in I, i<j} \mathbb{E}[\sigma_i^2 \sigma_j^2]
=n+6\binom n2
=3n^2-2n
\le3n^2\,,
\label{eq:lb-fourth-moment}
\end{equation}
where the first identity follows from expanding $S_I^4$, the second one follows from observing that any term with an index appearing an odd number of times has expectation zero, and the remaining terms are $n$ terms with one index repeated four times and 6 $\binom n2$ terms corresponding to the terms where $\sigma_i \sigma_j \sigma_k \sigma_l = \sigma_i^2 \sigma_j^2$ for $i \neq j$, i.e., for each of the $\binom n2$ unordered pairs of distinct indices,
there are $\binom 4 2=6$ arrangements with each index appearing twice. Every such term has expectation one.

Then, H\"older's inequality with exponents $3/2$ and $3$ yields: 
\[
n=\mathbb E|S_I|^2
=\mathbb E\!\left[|S_I|^{2/3}|S_I|^{4/3}\right]
\le
(\mathbb E|S_I|)^{2/3}
(\mathbb E|S_I|^4)^{1/3}.
\]
Rearranging and using~\eqref{eq:lb-fourth-moment}, we obtain: 
\begin{equation}
\mathbb E|S_I|
\ge\frac{n^{3/2}}{\sqrt{\mathbb E S_I^4}}
\ge\sqrt{n/3}.
\label{eq:lb-fluctuation}
\end{equation}
Substitution into~\eqref{eq:lb-before-fluctuation} gives
\begin{equation}
\mathbb E_{\sigma,\xi}D_I(\sigma, \xi)
\ge\frac{\sqrt n}{2\sqrt3}.
\label{eq:lb-random-cost-bound}
\end{equation}

\subsubsection*{Step 7. Fixing a deterministic oblivious cost sequence}

The sign vector on $I$ has finite support.
By~\eqref{eq:lb-random-cost-bound}, at least one deterministic sign vector
$\sigma^\star$ satisfies
\begin{equation}
\mathbb E_\xi D_I(\sigma^\star,\xi)
\ge\frac{\sqrt n}{2\sqrt3}.
\label{eq:lb-fixed-signs}
\end{equation}
Otherwise, averaging these finitely many expectations would contradict
\eqref{eq:lb-random-cost-bound}.
These expectations are well defined because $D_I(\sigma, \xi) \ge-3n/4$.

Fixing $\sigma^\star$ in~\eqref{eq:lb-costs} gives a deterministic
oblivious cost sequence, chosen independently of the learner's realized
random seed.
Both inequalities in~\eqref{eq:lb-pathwise} hold for this same sequence,
so~\eqref{eq:lb-fixed-signs} proves the claimed expected regret bound
simultaneously for the linear and DAC comparator classes.
The sequence works for every opponent control sequence, since opponents
do not have any affect on costs and state dynamics for the chosen dynamical system instance.
\end{proof}

\newpage

\section{Proof of Theorem~\ref{thm:time-var-potential}}
\label{sec:proof-thm-time-var}

We first recall the theorem:

\thmtimevarpotential*

\noindent\textbf{Outline of the proof.}
The proof of Theorem~\ref{thm:time-var-potential} can be divided into three main steps that are recorded in the following three propositions: 
\begin{enumerate}[leftmargin=2em]
\item Proposition~\ref{prop:eqgap-to-2ndpath} upperbounds the sum of squared equilibrium gaps by the sum of policy parameter deviations across time and players. 
\item Proposition~\ref{prop:second-order-path-length} bounds these
policy parameter deviations by the cumulative decrease of the common potential.
\item Finally, Proposition~\ref{prop:sum-delta-loss-bound} bounds that decrease
by the initial potential value and the cost and disturbance variation.
\end{enumerate}

\begin{proposition}
\label{prop:eqgap-to-2ndpath}
For the updates~\eqref{eqt:update} and every $T\ge1$,
\begin{equation}
\label{eq:pot-gap-to-movement}
\sum_{t=1}^T g_t(M_t)^2
\le C_M\sum_{t=1}^T\sum_{i=1}^N
\|M_{i,t+1}-M_{i,t}\|_F^2,
\end{equation}
where $C_M$ is defined in~\eqref{eq:pot-gradient-constants}.
\end{proposition}

\begin{proposition}
\label{prop:second-order-path-length}
For the updates~\eqref{eqt:update}, with the positive smoothness
bound $L$ in~\eqref{eq:pot-gradient-constants}, and every $T\ge1$,
\begin{equation}
\label{eq:pot-movement-to-descent}
\sum_{t=1}^T\sum_{i=1}^N
\|M_{i,t+1}-M_{i,t}\|_F^2
\le \frac{2}{L}\sum_{t=1}^T
\bigl[F_t(M_t)-F_t(M_{t+1})\bigr].
\end{equation}
\end{proposition}

\begin{proposition}
\label{prop:sum-delta-loss-bound}
For every $T\ge1$ and every feasible sequence
$M_1,\ldots,M_{T+1}\in\mathcal M$,
\begin{equation}
\label{eq:eqgap}
\sum_{t=1}^T
\bigl[F_t(M_t)-F_t(M_{t+1})\bigr]
\le F_1(M_1)+V_c(T)+G_xQ\,V_w^+(T).
\end{equation}
This inequality does not require the sequence to follow the update rule.
\end{proposition}

\paragraph{Geometry and explicit constants.}
All matrix operator norms are denoted by $\|\cdot\|_{\mathrm{op}}$.
On a player's DAC parameter space and on the joint product space, we use
the block Frobenius norms
\[
\|M_i\|_F^2=\sum_{p=0}^{H-1}\|M_i^{[p]}\|_F^2,
\qquad
\|M\|_F^2=\sum_{i=1}^N\|M_i\|_F^2.
\]
Gradients, projections, and parameter inner products use this geometry.
Let $h=H+1$ and $q=1-\bar\gamma\in(0,1)$. Joint strong stability gives: 
\begin{equation}
\label{eq:pot-power-bound}
\|\bar A^r\|_{\mathrm{op}}\le\bar\kappa q^r\,,
\quad \forall r\ge0,
\qquad
\bar A=A-\sum_{i=1}^N B_iK_i.
\end{equation}
Write the DAC block radii as $a_p$, so that
\begin{equation}
\label{eq:pot-domain}
\mathcal M_i=
\left\{(M_i^{[0]},\ldots,M_i^{[H-1]}):
\|M_i^{[p]}\|_{\mathrm{op}}\le a_p,\ 0\le p<H\right\}\,, \quad a_p=2\bar\kappa^2q^{p}\,.
\end{equation}
The proof below only uses the actual radii
$a_p$ in~\eqref{eq:pot-domain}.
If $r_i=\min(k_i,d)$, define the diameter bound
\begin{equation}
\label{eq:pot-diameter}
D_i=2\left(r_i\sum_{p=0}^{H-1}a_p^2\right)^{1/2}.
\end{equation}
Then $\operatorname{diam}_F(\mathcal M_i)\le D_i$:
each block difference has operator norm at most $2a_p$ and Frobenius
norm at most $2\sqrt{r_i}\,a_p$. The domains are compact and convex.

Set
\begin{equation}
\label{eq:pot-state-constants}
\begin{gathered}
b_i=\|B_i\|_{\mathrm{op}},\qquad
B_\Sigma=\sum_{i=1}^N b_i,\qquad
S_h=\bar\kappa\sum_{r=0}^{h-1}q^r,\qquad
m_H=\sum_{p=0}^{H-1}a_p,\\
Q=S_h(1+B_\Sigma m_H),\qquad
D_x=WQ,\qquad G_x=GD_x.
\end{gathered}
\end{equation}
In particular, $S_h\le\bar\kappa/\bar\gamma$ and
$m_H\le2\bar\kappa^2/\bar\gamma$.
Define
\begin{equation}
\label{eq:pot-jacobian-constants}
J_i=W\sqrt H\,S_hb_i,\qquad
J=\left(\sum_{i=1}^N J_i^2\right)^{1/2},
\end{equation}
and choose
\begin{equation}
\label{eq:pot-gradient-constants}
L=\max\{1,\zeta J^2\},\qquad
G_{M,i}=G_xJ_i,\qquad
C_M=\sum_{i=1}^N(LD_i+G_{M,i})^2.
\end{equation}
Lemma~\ref{lem:smoothness-l_t} proves that $D_x$ bounds the rollout
states, $G_{M,i}$ bounds the coordinate parameter gradients, and $L$
is a common smoothness bound for $F_t$.
The choice of $L$ is positive even when $\zeta=0$ or $J=0$.
One may also use any other fixed $L>0$ with $L\ge\zeta J^2$ and
define $C_M$ using that same $L$.
All these constants are independent of $T$ for fixed memory and
problem parameters.

\noindent\textbf{Proof of Theorem~\ref{thm:time-var-potential} using the propositions.}
Combining the three propositions above gives: 
\begin{align*}
\sum_{t=1}^Tg_t(M_t)^2
&\le C_M\sum_{t=1}^T\sum_{i=1}^N
\|M_{i,t+1}-M_{i,t}\|_F^2\\
&\le\frac{2C_M}{L}
\sum_{t=1}^T\bigl[F_t(M_t)-F_t(M_{t+1})\bigr]\\
&\le\frac{2C_M}{L}
\bigl[F_1(M_1)+V_c(T)+G_xQ\,V_w^+(T)\bigr].
\end{align*}
Dividing by $T$ proves~\eqref{eqt:main-bound}.
We now establish the propositions in the same order as above.

\subsection{Proof of Proposition~\ref{prop:eqgap-to-2ndpath}}

First, recall that: 
\[
g_{i,t}(M_t)
=F_t(M_t)-\min_{Q_i\in\mathcal M_i}
F_t(Q_i,M_{-i,t}),
\qquad
g_t(M_t)=\max_{i\in[N]}g_{i,t}(M_t).
\]
The minima exist because the domains are compact and $F_t$ is
continuous. Taking $Q_i=M_{i,t}$ shows that $g_{i,t}(M_t)\ge0$.
Consequently,
\begin{equation}
\label{eq:sum-eq-gap-squared-cauchy}
\sum_{t=1}^Tg_t(M_t)^2
\le
\sum_{t=1}^T\left(\sum_{i=1}^Ng_{i,t}(M_t)\right)^2.
\end{equation}
We relate the individual improvements to the parameter movements
using the following proposition, proved in
Section~\ref{subsec:proof-prop-approx-nash-gap}.

\begin{proposition}
\label{prop:approx-nash-gap}
For every $i\in[N]$, $t\ge1$, and $Q_i\in\mathcal M_i$, the
update~\eqref{eqt:update} satisfies
\begin{equation}
\label{eq:pot-individual-improvement}
F_t(Q_i,M_{-i,t})-F_t(M_t)
\ge
-(LD_i+G_{M,i})\|M_{i,t+1}-M_{i,t}\|_F,
\end{equation}
with the constants in~\eqref{eq:pot-diameter}
and~\eqref{eq:pot-gradient-constants}.
\end{proposition}

Taking the maximum improvement over $Q_i$ yields
\[
0\le g_{i,t}(M_t)
\le (LD_i+G_{M,i})\|M_{i,t+1}-M_{i,t}\|_F.
\]
Summing over players and using Cauchy--Schwarz inequality gives
\begin{align*}
\left(\sum_{i=1}^Ng_{i,t}(M_t)\right)^2
&\le
\left(\sum_{i=1}^N
(LD_i+G_{M,i})\|M_{i,t+1}-M_{i,t}\|_F\right)^2\\
&\le
\left(\sum_{i=1}^N(LD_i+G_{M,i})^2\right)
\sum_{i=1}^N\|M_{i,t+1}-M_{i,t}\|_F^2\\
&=C_M\sum_{i=1}^N\|M_{i,t+1}-M_{i,t}\|_F^2.
\end{align*}
Summing this inequality over $t$ and using
\eqref{eq:sum-eq-gap-squared-cauchy} proves the proposition.

\subsection{Proof of Proposition~\ref{prop:second-order-path-length}}

We first establish the regularity of the common diagonal potential,
including the parameter-gradient bound used in
Proposition~\ref{prop:approx-nash-gap}.

\begin{lemma}[State bounds, parameter gradients, and smoothness]
\label{lem:smoothness-l_t}
For every $t\ge1$ and every $M\in\mathcal M$, the rollout
in~\eqref{eqt:rollout} satisfies $\|y_t(M)\|_2\le D_x$.
The loss $F_t$ is convex and differentiable, and
\[
\|\nabla_iF_t(M)\|_F\le G_{M,i},\qquad
\|\nabla F_t(M)-\nabla F_t(M')\|_F
\le\zeta J^2\|M-M'\|_F\,, 
\quad \forall M,M'\in\mathcal M.
\]
In particular, $F_t$ is $L$-smooth for the choice
in~\eqref{eq:pot-gradient-constants}.
\end{lemma}

\begin{proof}
Using the rollout formula \eqref{eqt:rollout}, the disturbance bound, and joint strong
stability, we obtain: 
\begin{align*}
\|y_t(M)\|_2
&\le
\sum_{r=1}^{h}\bar\kappa q^{r-1}
\left(W+\sum_{i=1}^Nb_i\sum_{p=1}^Ha_{p-1}W\right)\\
&=WS_h(1+B_\Sigma m_H)=D_x.
\end{align*}
Assumption~\ref{as:convexity}, specialized to the state-only cost,
therefore gives
$\|\nabla c_t(x)\|_2\le GD_x=G_x$ on $\|x\|_2\le D_x$.

The map $M\mapsto y_t(M)$ is affine: its disturbance-only term is
independent of $M$, and every remaining term is linear in the DAC
blocks. Its derivative in coordinate $i$ is the linear map
\begin{equation}
\label{eq:pot-coordinate-jacobian}
\mathcal J_{i,t}[\Delta M_i]
=\sum_{r=1}^{h}\bar A^{r-1}B_i
\sum_{p=1}^{H}\Delta M_i^{[p-1]}w_{t-r-p}.
\end{equation}
Since
\[
\sum_{p=0}^{H-1}\|\Delta M_i^{[p]}\|_{\mathrm{op}}
\le\sum_{p=0}^{H-1}\|\Delta M_i^{[p]}\|_F
\le\sqrt H\,\|\Delta M_i\|_F,
\]
we obtain
\[
\|\mathcal J_{i,t}[\Delta M_i]\|_2
\le W\sqrt H\,S_hb_i\|\Delta M_i\|_F
=J_i\|\Delta M_i\|_F.
\]
The full derivative $\mathcal J_t$ is the sum of the coordinate
derivatives, so another application of Cauchy--Schwarz yields
\[
\|\mathcal J_t[\Delta M]\|_2
\le\sum_{i=1}^NJ_i\|\Delta M_i\|_F
\le J\|\Delta M\|_F.
\]
Because $c_t$ is convex and differentiable, composition with the
affine map $y_t$ makes $F_t$ convex and differentiable.
The chain rule, with adjoints in the stated Euclidean and Frobenius
inner products, gives
\[
\nabla_iF_t(M)=\mathcal J_{i,t}^*\nabla c_t(y_t(M)).
\]
Thus
$\|\nabla_iF_t(M)\|_F\le J_iG_x=G_{M,i}$.
Finally, $\mathcal J_t$ does not depend on $M$.
Assumption~\ref{as:smoothness} gives
\begin{align*}
\|\nabla F_t(M)-\nabla F_t(M')\|_F
&\le J\,\|\nabla c_t(y_t(M))-\nabla c_t(y_t(M'))\|_2\\
&\le J\zeta\,\|y_t(M)-y_t(M')\|_2\\
&\le\zeta J^2\|M-M'\|_F.
\end{align*}
The choice $L\ge\zeta J^2$ proves the smoothness assertion.
\end{proof}

We now prove the descent inequality.
The induced stage game has the common payoff $F_t$, so each player's
coordinate gradient is $\nabla_iF_t$. Moreover, projection onto the product domain
separates in the block Frobenius geometry:
\[
\Pi_{\mathcal M}(Z)=
\bigl(\Pi_{\mathcal M_1}(Z_1),\ldots,
\Pi_{\mathcal M_N}(Z_N)\bigr)
\]
for every joint parameter vector $Z$.
The simultaneous updates~\eqref{eqt:update} are therefore exactly
\begin{equation}
\label{eq:update-rule-magpc}
M_{t+1}=\Pi_{\mathcal M}
\left(M_t-L^{-1}\nabla F_t(M_t)\right).
\end{equation}
By Lemma~\ref{lem:smoothness-l_t}, uniform smoothness of $F_t$ yields: 
\begin{equation}
\label{eq:smoothness-lt-taylor}
F_t(M_{t+1})\le{} 
F_t(M_t)+\langle\nabla F_t(M_t),M_{t+1}-M_t\rangle
+\frac L2\|M_{t+1}-M_t\|_F^2.
\end{equation}

Using the characterization of the projection operator $\Pi_{\mathcal M}$, we have: 
\begin{equation}
\forall M \in \mathcal{M}, \quad \langle M - M_{t+1}, M_t - \frac 1L \nabla F_t(M_t) - M_{t+1} \rangle \leq 0\,.
\end{equation}

Then, using the above inequality with the
feasible comparison point $M_t$ gives: 
\begin{equation}
\label{eq:proj-ineq-grad-desc}
\langle\nabla F_t(M_t),M_{t+1}-M_t\rangle
\le -L\|M_{t+1}-M_t\|_F^2.
\end{equation}
Combining~\eqref{eq:smoothness-lt-taylor}
and~\eqref{eq:proj-ineq-grad-desc},
\[
F_t(M_{t+1})
\le F_t(M_t)-\frac L2\|M_{t+1}-M_t\|_F^2.
\]
Rearranging and summing proves
\[
\sum_{t=1}^T\sum_{i=1}^N
\|M_{i,t+1}-M_{i,t}\|_F^2
\le\frac2L\sum_{t=1}^T
\bigl[F_t(M_t)-F_t(M_{t+1})\bigr],
\]
as claimed. With $\eta=1/L$, the coefficient is $2\eta$.

\subsection{Proof of Proposition~\ref{prop:sum-delta-loss-bound}}

\noindent\textbf{(i) Telescoping the potential decreases.}
First, the common cost is nonnegative under
Assumption~\ref{as:convexity}. Indeed, for every radius $D>0$,
that assumption gives
\[
|c_t(0)|\le\beta D^2,\qquad
\|\nabla c_t(0)\|_2\le GD.
\]
Letting $D$ decrease to zero shows that $c_t(0)=0$ and
$\nabla c_t(0)=0$. Convexity then implies
\[
c_t(x)\ge c_t(0)+\langle\nabla c_t(0),x\rangle=0.
\]
In particular, $F_t(M)\ge0$ for every feasible $M$.
For any feasible sequence $M_1,\ldots,M_{T+1}$, the exact telescoping
identity therefore gives
\begin{align}
\label{eq:decomp-sum-t-loss-diff}
\sum_{t=1}^T[F_t(M_t)-F_t(M_{t+1})]
={}&F_1(M_1)-F_{T}(M_{T+1})\nonumber\\
&+\sum_{t=1}^{T-1}
[F_{t+1}(M_{t+1})-F_t(M_{t+1})]\nonumber\\
\le{}&F_1(M_1)+\sum_{t=1}^{T-1}
[F_{t+1}(M_{t+1})-F_t(M_{t+1})].
\end{align}
For $T=1$, the temporal-variation sum is empty.

\medskip
\noindent\textbf{(ii) Separating cost variation and state variation.}
Let $M\in\mathcal M$ be a fixed joint profile. Then
\begin{equation*}
F_{t+1}(M)-F_t(M)
= c_{t+1}(y_{t+1}(M))-c_t(y_{t+1}(M)) 
+c_t(y_{t+1}(M))-c_t(y_t(M)).
\end{equation*}
The state bound in Lemma~\ref{lem:smoothness-l_t} gives
\begin{equation}
\label{eq:var-cost-term}
c_{t+1}(y_{t+1}(M))-c_t(y_{t+1}(M))
\le\sup_{\|x\|_2\le D_x}|c_{t+1}(x)-c_t(x)|.
\end{equation}
Both states lie in the convex ball of radius $D_x$, where the gradient has norm at most $G_x$. Thus
\begin{equation}
\label{eq:dif-cost-lip}
c_t(y_{t+1}(M))-c_t(y_t(M))
\le G_x\|y_{t+1}(M)-y_t(M)\|_2.
\end{equation}
Define
\[
\Delta_{t+1}^y
:=y_{t+1}(M_{t+1})-y_t(M_{t+1}).
\]
Both states in this difference use the same joint profile
$M_{t+1}$. Inserting the preceding bounds
into~\eqref{eq:decomp-sum-t-loss-diff} yields
\begin{equation}
\label{eq:sum-diff-losses-interm}
\sum_{t=1}^T[F_t(M_t)-F_t(M_{t+1})]
\le F_1(M_1)+V_c(T)
+G_x\sum_{t=1}^{T-1}\|\Delta_{t+1}^y\|_2.
\end{equation}

\medskip
\noindent\textbf{(iii) Bounding the state variation.}
For every integer $s$, let
$\delta_s=\|w_{s+1}-w_s\|_2$.
The zero-padding convention gives
\[
\delta_{-1}=\|w_0\|_2,\qquad
\delta_s=0\,, \forall s<-1,\qquad
V_w^+(T)=\sum_{s=-1}^{T-2}\delta_s.
\]
For the same fixed profile $M$, subtracting the two rollouts
in~\eqref{eqt:rollout} gives
\begin{align}
\label{eq:pot-fixed-profile-state-difference}
y_{t+1}(M)-y_t(M)
={}&\sum_{r=1}^{h}\bar A^{r-1}
(w_{t+1-r}-w_{t-r})\nonumber\\
&+\sum_{r=1}^{h}\sum_{i=1}^N\sum_{p=1}^{H}
\bar A^{r-1}B_iM_i^{[p-1]}
(w_{t+1-r-p}-w_{t-r-p}).
\end{align}
All DAC coefficients in this difference belong to that same profile.
By joint strong stability and the block bounds,
\begin{equation}
\label{eq:bound-sum-delta-y}
\|y_{t+1}(M)-y_t(M)\|_2
\le \sum_{r=1}^{h}\bar\kappa q^{r-1}\delta_{t-r} 
+\sum_{r=1}^{h}\sum_{i=1}^N\sum_{p=1}^{H}
\bar\kappa q^{r-1}b_i a_{p-1}\delta_{t-r-p}.
\end{equation}
This bound is uniform over $M\in\mathcal M$.
For each fixed $r,p$, every index in either shifted sum
\[
\sum_{t=1}^{T-1}\delta_{t-r},
\qquad
\sum_{t=1}^{T-1}\delta_{t-r-p}
\]
is at most $T-2$, and terms with indices below $-1$ are zero.
Each shifted sum is therefore at most $V_w^+(T)$.
Taking the supremum over $M$ in~\eqref{eq:bound-sum-delta-y}
and summing gives
\begin{align}
\label{eq:final-bound-sum-delta-y}
\sum_{t=1}^{T-1}\|\Delta_{t+1}^y\|_2
&\le
\sum_{t=1}^{T-1}\sup_{M\in\mathcal M}
\|y_{t+1}(M)-y_t(M)\|_2\nonumber\\
&\le
\left[
\sum_{r=1}^{h}\bar\kappa q^{r-1}
\left(1+\sum_{i=1}^Nb_i\sum_{p=1}^{H}a_{p-1}\right)
\right]V_w^+(T)\nonumber\\
&=S_h(1+B_\Sigma m_H)V_w^+(T)
=Q\,V_w^+(T).
\end{align}

\medskip
\noindent\textbf{(iv) Combining the bounds.}
Substituting~\eqref{eq:final-bound-sum-delta-y}
into~\eqref{eq:sum-diff-losses-interm} gives
\[
\sum_{t=1}^T[F_t(M_t)-F_t(M_{t+1})]
\le F_1(M_1)+V_c(T)+G_xQ\,V_w^+(T),
\]
which proves~\eqref{eq:eqgap}.

\subsection{Proof of Proposition~\ref{prop:approx-nash-gap}}
\label{subsec:proof-prop-approx-nash-gap}

Let $i\in[N]$, $t\ge1$, and $Q_i\in\mathcal M_i$.
The proof follows four steps.

\noindent\textbf{(i) Convexity.}
By Lemma~\ref{lem:smoothness-l_t}, the common diagonal loss is
convex. In particular,
\begin{align}
\label{eq:convexity}
F_t(Q_i,M_{-i,t})-F_t(M_t)
&\ge\langle\nabla_iF_t(M_t),Q_i-M_{i,t}\rangle\nonumber\\
&=\langle\nabla_iF_t(M_t),Q_i-M_{i,t+1}\rangle\nonumber\\
&\quad+
\langle\nabla_iF_t(M_t),M_{i,t+1}-M_{i,t}\rangle.
\end{align}

\medskip
\noindent\textbf{(ii) Lower bound for the first inner product.}
The update~\eqref{eqt:update} and projection optimality give: 
\[
\left\langle
Q_i-M_{i,t+1},
M_{i,t}-L^{-1}\nabla_iF_t(M_t)-M_{i,t+1}
\right\rangle\le0.
\]
Rearranging, using Cauchy--Schwarz inequality and the diameter bound, we obtain: 
\begin{align}
\label{eq:ineq2-inner-prod1}
\langle\nabla_iF_t(M_t),Q_i-M_{i,t+1}\rangle
&\ge
L\langle M_{i,t}-M_{i,t+1},Q_i-M_{i,t+1}\rangle\nonumber\\
&\ge
-L\|M_{i,t+1}-M_{i,t}\|_F \cdot 
\|Q_i-M_{i,t+1}\|_F\nonumber\\
&\ge -LD_i\|M_{i,t+1}-M_{i,t}\|_F.
\end{align}

\medskip
\noindent\textbf{(iii) Lower bound for the second inner product.}
Lemma~\ref{lem:smoothness-l_t} gives the parameter gradient
bound $\|\nabla_iF_t(M_t)\|_F\le G_{M,i}=G_xJ_i$.
Therefore,
\begin{align}
\label{eq:ineq3-inner-prod2}
\langle\nabla_iF_t(M_t),M_{i,t+1}-M_{i,t}\rangle
&\ge
-\|\nabla_iF_t(M_t)\|_F
\|M_{i,t+1}-M_{i,t}\|_F\nonumber\\
&\ge
-G_{M,i}\|M_{i,t+1}-M_{i,t}\|_F.
\end{align}

\medskip
\noindent\textbf{(iv) Combining the bounds.}
Substituting~\eqref{eq:ineq2-inner-prod1}
and~\eqref{eq:ineq3-inner-prod2} into~\eqref{eq:convexity} yields
\[
F_t(Q_i,M_{-i,t})-F_t(M_t)
\ge
-(LD_i+G_{M,i})\|M_{i,t+1}-M_{i,t}\|_F.
\]
This is the claimed inequality.

\subsection{Proof of Lemma~\ref{lem:convexity-loss-i}}
 
We return here to the general individual cost setting. 
Recall that the loss function~$\ell_t^i$ is defined for every~$M_i \in \mathcal{M}_i$ by 
\begin{equation}
\ell_t^i(M_i) = c_t^i(y_t^{i,K_i}(M_i),v_{t}^{i,K_i}(M_i))\,,
\end{equation}
where the counterfactual idealized state~$y_t^{i,K_i}(M_i)$ and action~$v_{t}^{i,K_i}(M_i)$ are defined in section~\ref{subsec:notation-counterfa-state-action}.  

The convexity
argument uses only Assumption~\ref{as:convexity}.
Fix a player $i$, a time $t$, and the realized sequence of opponents'
controls used in the unilateral comparison. Let
$d_s=w_s+\sum_{j\ne i}B_ju_s^j$ for $s\ge0$ denote the resulting external input.
The DAC policy uses the effective disturbance sequence $r_s=d_s$ under state-only feedback and $r_s=w_s$
under aggregate-input feedback. Set $d_s=r_s=0$ for $s<0$.
Both sequences are held fixed when varying the comparator parameter $M_i$.

Starting from the fixed zero state at the truncation boundary, the
unilateral rollout gives: 
\[
z_{s+1}(M_i)
=(A-B_iK_i)z_s(M_i)
+B_i\sum_{p=1}^{H}M_i^{[p-1]}r_{s-p}+d_s.
\]
Its dependence on $M_i$ is affine by induction.
Hence the terminal idealized state $y_t^{i,K_i}(M_i)$ and action
\[
v_t^{i,K_i}(M_i)
=-K_i y_t^{i,K_i}(M_i)
+\sum_{p=1}^{H}M_i^{[p-1]}r_{t-p}
\]
are affine in $M_i$.
Assumption~\ref{as:convexity} makes $c_t^i$ jointly convex, so
\[
\ell_t^i(M_i)
=c_t^i\bigl(y_t^{i,K_i}(M_i),v_t^{i,K_i}(M_i)\bigr)
\]
is convex as the composition of a convex function and an affine map.
This argument holds for any realized opponents' control sequence,
including one generated by DAC policies.

\section{Tools from Online Convex Optimization}
\label{sec:oco}

\subsection{Online convex optimization with memory}
\label{sec:oco-memory}

We use the following explicit interval form of the online gradient
descent with memory argument of~\cite{anava-hazan-mannor15online-learning-memory}. 
The updates begin at round one, even when regret is evaluated only on a later interval.

\begin{algorithm}[H]
\begin{algorithmic}[1]
\STATE \textbf{Input:} nonempty closed convex domain $\mathcal K$,
memory $h\ge1$, step size $\eta>0$, and $z_{\rm init}\in\mathcal K$.
\STATE Set $z_{1-h}=\cdots=z_1=z_{\rm init}$.
\FOR{$t=1,\ldots,T$}
\STATE Play $z_t$ and incur $\ell_t(z_{t-h},\ldots,z_t)$.
\STATE Observe the diagonal loss $f_t(z):=\ell_t(z,\ldots,z)$
or its required gradient.
\STATE Set $z_{t+1}=\Pi_{\mathcal K}(z_t-\eta\nabla f_t(z_t))$.
\ENDFOR
\end{algorithmic}
\caption{Online Gradient Descent with Memory}
\label{algo:ogdm}
\end{algorithm}

\begin{theorem}[Online gradient descent with memory]
\label{thm:ogdm-regret}
Let $\mathcal K$ be a nonempty closed convex subset of a
finite-dimensional Euclidean space with diameter at most $D_0>0$.
Let $\ell_t:\mathcal K^{h+1}\to\mathbb R$ be a sequence of losses
such that:
\begin{enumerate}[leftmargin=2em]
\item There exists~$L_0 > 0$ s.t. for any $z_0, \dots, z_h, \tilde z_j \in \mathcal{K}$, whenever two argument tuples differ only in coordinate $j\in\{0,\ldots,h\}$,
\[
\bigl|\ell_t(z_0,\ldots,z_j,\ldots,z_h)
-\ell_t(z_0,\ldots,\widetilde z_j,\ldots,z_h)\bigr|
\le L_0\|z_j-\widetilde z_j\|.
\]
\item The diagonal loss $f_t(z):=\ell_t(z,\ldots,z)$ is convex
and differentiable, and there is $G_0>0$ such that
$\|\nabla f_t(z)\|\le G_0$ for every $z\in\mathcal K$ and every $t$.
\end{enumerate}
Then Algorithm~\ref{algo:ogdm} satisfies, for every
$1\le a\le b\le T$,
\begin{equation}
\label{eq:j-memory-regret-interval}
\sum_{t=a}^{b}\ell_t(z_{t-h:t})
-\min_{z\in\mathcal K}\sum_{t=a}^{b}f_t(z)
\le\frac{D_0^2}{2\eta}
+\frac{\eta(b-a+1)}2
\left[G_0^2+L_0G_0h(h+1)\right].
\end{equation}
In particular, choosing
\[
\eta=\frac{D_0}{\sqrt{T[G_0^2+L_0G_0h(h+1)]}}
\]
bounds the left-hand side by
$D_0\sqrt{T[G_0^2+L_0G_0h(h+1)]}$.
\end{theorem}

\begin{proof}
For any fixed $z\in\mathcal K$, projection nonexpansiveness gives
\[
\|z_{t+1}-z\|^2
\le\|z_t-\eta\nabla f_t(z_t)-z\|^2.
\]
Expanding, rearranging, and using convexity and the gradient bound, we obtain: 
\[
f_t(z_t)-f_t(z)
\le\frac{\|z_t-z\|^2-\|z_{t+1}-z\|^2}{2\eta}
+\frac{\eta G_0^2}{2}.
\]
Summing from $a$ to $b$ and using the diameter bound yields: 
\begin{equation}
\label{eq:j-diagonal-regret}
\sum_{t=a}^{b}[f_t(z_t)-f_t(z)]
\le\frac{D_0^2}{2\eta}+\frac{\eta(b-a+1)G_0^2}{2}.
\end{equation}
Also, since $\Pi_{\mathcal K}(z_t)=z_t$, and using nonexpansiveness of the projection, we have for all $t \ge 1$, 
\[
\|z_{t+1}-z_t\| = \|\Pi_{\mathcal K}(z_t - \eta \nabla f_t(z_t)) -\Pi_{\mathcal K}(z_t)\|
\le\eta\|\nabla f_t(z_t)\|\le\eta G_0\,.
\]
Consequently, for every $1\le k\le h$,
\[
\|z_t-z_{t-k}\|\le k\eta G_0,
\]
including when $t-k\le0$. Changing the arguments of the memory
loss one at a time and using the coordinate-wise Lipschitzness assumption gives: 
\begin{equation}
\label{eq:j-memory-diagonal-difference}
\ell_t(z_{t-h:t})-f_t(z_t)
\le L_0\sum_{k=1}^{h}\|z_{t-k}-z_t\|
\le\frac{\eta L_0G_0h(h+1)}2.
\end{equation}
Combining~\eqref{eq:j-diagonal-regret} and
\eqref{eq:j-memory-diagonal-difference}, summing, and minimizing
over the fixed comparator proves~\eqref{eq:j-memory-regret-interval}.
The minimum exists because $\mathcal K$ is compact and the losses
are continuous. The stated step-size conclusion uses $b-a+1\le T$.
\end{proof}

In our control setting, the Euclidean geometry is the block
Frobenius geometry, $h=H+1$, $a=H+1$, and $b=T$. 
The constants $D_0,G_0,L_0$ respectively denote
parameter diameter, diagonal gradient bound, and memory-coordinate
Lipschitz constant, and are distinct from the state bound $D$ and
the cost-gradient growth constant $G$.

\subsection{Time regret decomposition}

\begin{lemma}
\label{lem:time-regret-decomp}
For every player $i \in [N]$, every fixed realized opponent control sequence $\{u_t^{-i}\}$,
every comparator class $\Pi_i$ for which the minima are defined,
and integers $1\le H<T$,
\[
\reg_i^{1:T}(\mathcal A_i,\{u_t^{-i}\},\Pi_i)
\le\reg_i^{1:H}(\mathcal A_i,\{u_t^{-i}\},\Pi_i)
+\reg_i^{H+1:T}(\mathcal A_i,\{u_t^{-i}\},\Pi_i).
\]
Every counterfactual trajectory in these interval regrets starts at
time zero from the same original initial state; no state is reset
at the interval boundary.
\end{lemma}

\begin{proof}
For any comparator policy $\pi^i$, write
\[
A_{\pi^i}:=\sum_{t=1}^{H}c_t^i(x_t^{\pi^i},u_t^{\pi^i}),
\qquad
B_{\pi^i}:=\sum_{t=H+1}^{T}c_t^i(x_t^{\pi^i},u_t^{\pi^i}).
\]
Then, we have: 
\[
\min_{\pi^i\in\Pi_i}(A_{\pi^i}+B_{\pi^i})
\ge\min_{\pi^i\in\Pi_i}A_{\pi^i}
+\min_{\pi^i\in\Pi_i}B_{\pi^i}.
\]
Subtracting this inequality from the learner's total cost and
splitting that cost over $1{:}H$ and $H+1{:}T$ proves the result
using~\eqref{eq:reg-def}.
\end{proof}

\newpage
\section{Simulations}
\label{sec:simulations}

While our contributions are mainly theoretical, we also performed simple simulations to illustrate the algorithm's behavior.

We implemented Algorithm~\ref{algo:MAGPC} in both settings (independent learning and aggregated control learning) for $N=3$ agents with distinct quadratic cost functions, using a two-dimensional instance of \eqref{eq:state-lds}. We considered three types of disturbance sequences: (1) a constant nonzero disturbance, (2) sinusoidal disturbances (with different frequencies for each coordinate), and (3) Gaussian noise. 
See sections below for details of the setup, vanishing per-agent time-average cost, control input curves, and 2D state trajectories.
The simulations were run on Google Colab and only take a few seconds to minutes to be completed. 

\subsection{Setting}

We consider a 2-dimensional ($d=2$)~\ref{eq:state-lds} with $N = 3$ agents and scalar control inputs ($k_i = 1$ for $i \in \{1,2,3\}$) with: 
\begin{equation}
A = \begin{bmatrix}
0.95 & 0.1 \\
0 & 0.9
\end{bmatrix}\,,\quad 
B_1 = B_3 =  \begin{bmatrix}
1 \\
0 
\end{bmatrix}\,, 
\quad 
B_2 = \begin{bmatrix}
0 \\
1 
\end{bmatrix}\,, \quad x_0 = \begin{bmatrix}
6 \\
6 
\end{bmatrix}\,.
\end{equation}
Each agent $i \in \{1,2,3\}$ has their quadratic cost function $c_t^i(x, u^i) = x^{\top} Q_i x + r_i (u^i)^2$ where: 
\begin{equation}
Q_i = \begin{bmatrix}
1 + 0.8 i & 0 \\
0 & 1 + 0.4 (N-1-i)
\end{bmatrix}\,,\quad r_i = 0.1 (1 + 0.4 i)\,,\quad u^i \in \R\,.
\end{equation}
The cost functions reflect different distances to the origin goal state $(0,0)$, see figures~\ref{fig:constant-disturb}, \ref{fig:sin-disturb}, \ref{fig:gaussian-disturb} below (bottom right subplots). 

We test Algorithm~\ref{algo:MAGPC} with three different kinds of disturbances $w_t \in \R^2$:
\begin{enumerate}[label=(\arabic*)]
\item Constant disturbance: $w_t = 0.7$ (see Fig.~\ref{fig:constant-disturb}),
\item Sinusoidal disturbance: $w_{t,1} = \sin(0.1 t), w_{t,2} = \sin(0.1 t)$ where $w_{t,1}, w_{t,2}$ are the coordinates of $w_t$ (see Fig.~\ref{fig:sin-disturb}),
\item Independent and identically distributed Gaussian: $w_t \sim \mathcal{N}(0, \sigma^2)$ with $\sigma = 0.5$ (see Fig.~\ref{fig:gaussian-disturb}).\footnote{The untruncated Gaussian experiments are illustrative stress tests outside the deterministic bounded disturbance assumption
of our regret theorems.}
\end{enumerate}

For the hyperparameters of Algorithm~\ref{algo:MAGPC}, we set $T=500$ for the time horizon, $\eta = 10^{-4}$ for the step size, $H=5$ for the memory parameter and DAC policy parameters are initialized with zero values.

\begin{figure}[htbp]
    \centering
    \includegraphics[width=0.75\textwidth]{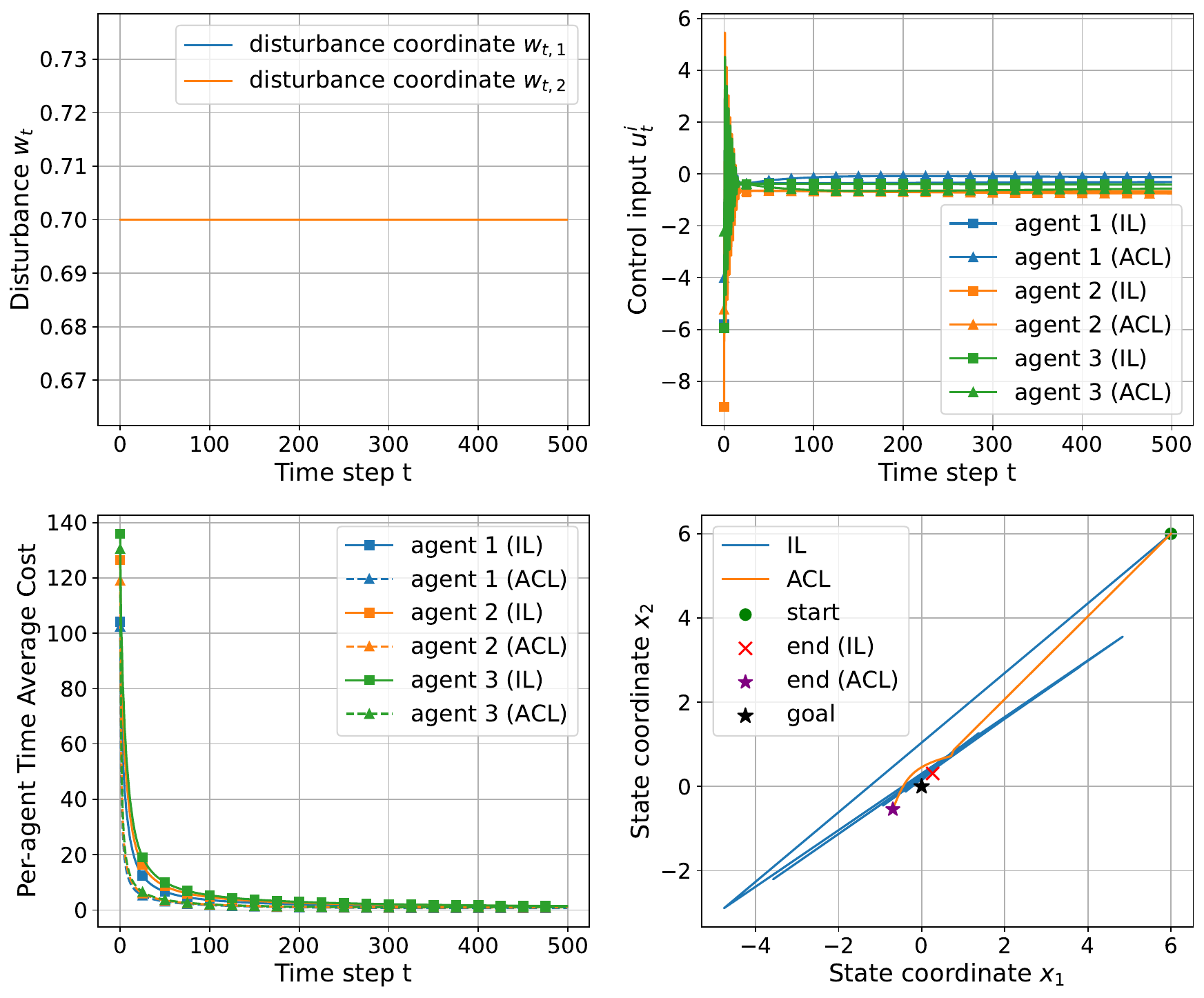}%
    \caption{Illustration of the performance of Algorithm~\ref{algo:MAGPC} on a simple multi-agent LDS with constant disturbance sequence. `IL' stands for Independent Learning (see Information Setting~\ref{cp0}), `ACL' for Aggregated Control Learning (see  Information Setting~\ref{cp+}), `start' refers to the initial state $x_0$, `end' to the state at the last time step for $t=T$.}
    \label{fig:constant-disturb}
\end{figure}

\begin{figure}[htbp]
    \centering
    \includegraphics[width=0.75\textwidth]{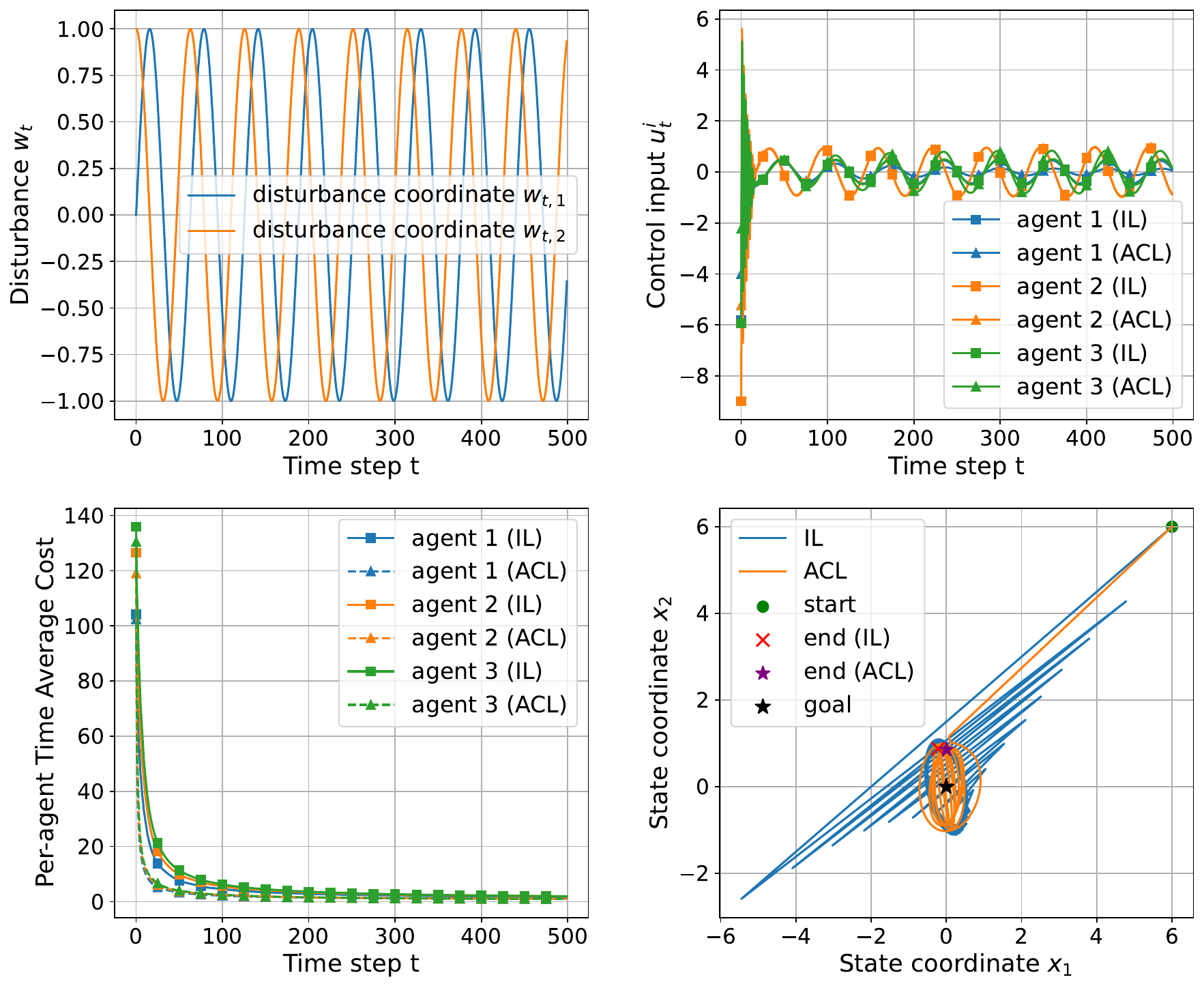}%
    \caption{Illustration of the performance of Algorithm~\ref{algo:MAGPC} on a simple multi-agent LDS with sinusoidal disturbance sequence. `IL' stands for Independent Learning (see Information Setting~\ref{cp0}, `ACL' for Aggregated Control Learning (see  Information Setting~\ref{cp+}), `start' refers to the initial state $x_0$, `end' to the state at the last time step for $t=T$.}
    \label{fig:sin-disturb}
\end{figure}

\begin{figure}[htbp]
    \centering
    \includegraphics[width=0.75\textwidth]{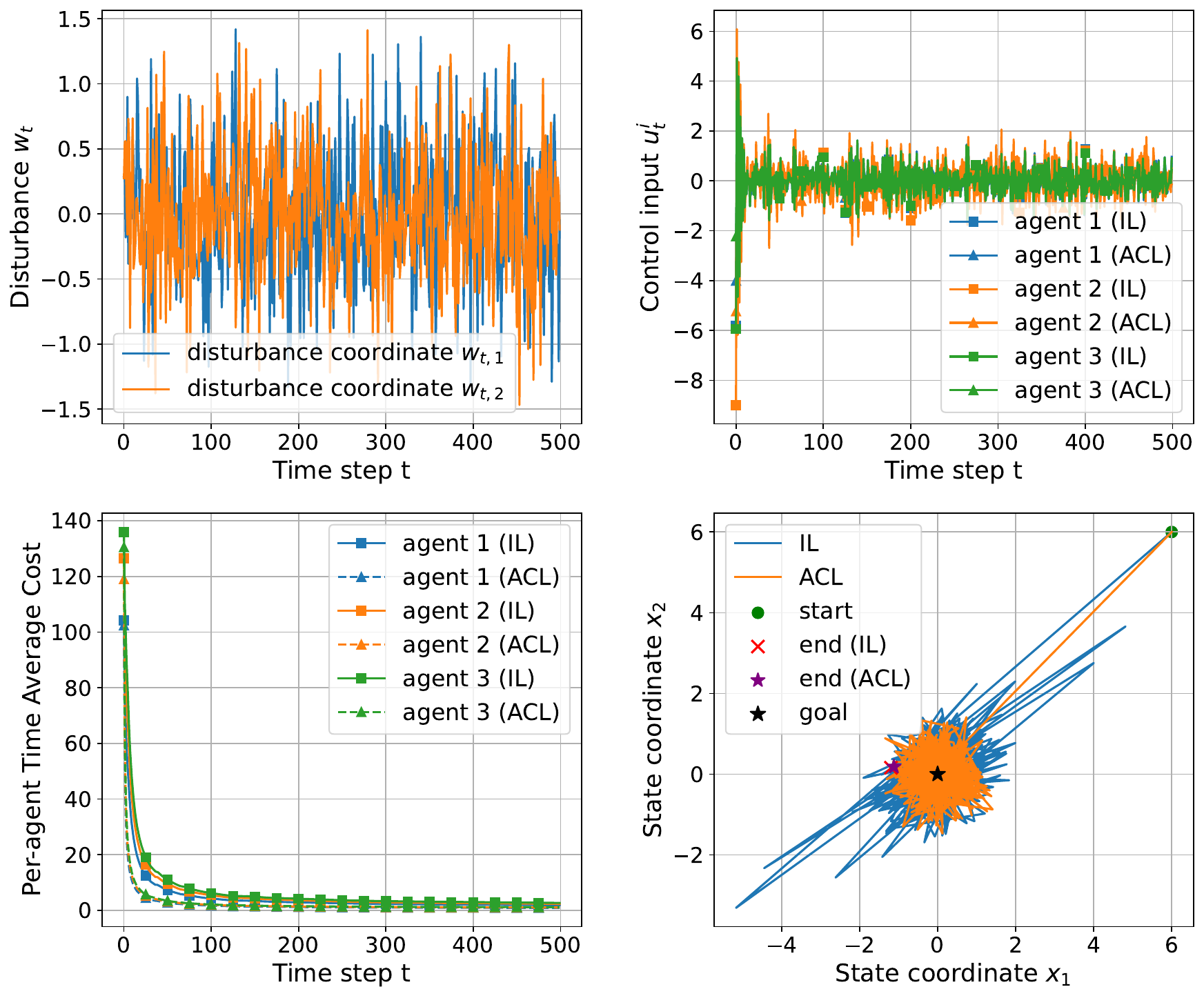}%
    \caption{Illustration of the performance of Algorithm~\ref{algo:MAGPC} on a simple multi-agent LDS with Gaussian disturbance sequence. `IL' stands for Independent Learning (see Information Setting~\ref{cp0}, `ACL' for Aggregated Control Learning (see  Information Setting~\ref{cp+}), `start' refers to the initial state $x_0$, `end' to the state at the last time step for $t=T$.}
    \label{fig:gaussian-disturb}
\end{figure}

\subsection{Comments}

We make a few remarks on the results of the simulations (see Figs.~\ref{fig:constant-disturb}, \ref{fig:sin-disturb}, \ref{fig:gaussian-disturb} above): 
\begin{itemize}[leftmargin=*]
\item Starting from the initial state $x_0$ (see bottom right subplots in all the figures), the state evolves quickly towards the goal origin state minimizing the costs by strong stability of the controllers. In particular, this quick phase corresponds to an application of higher control inputs by all agents in all three disturbance scenarios (see Figs.~\ref{fig:constant-disturb} to~\ref{fig:gaussian-disturb}, top right subplots). Then the control inputs have a similar shape to the disturbances themselves to stabilize the system (almost constant in the first case, sinusoidal in the second and random Gaussian in the third case).  

\item The plots show a reduction in per-agent time-average costs over the simulated horizon.

\item It can be seen in all the figures that there is a slight advantage to the ACL setting (which can infer the disturbance values) compared to the independent learning setting in our simple simulation setting. Compare for instance the dotted per-agent time-average cost curves to the plain ones in bottom left subplots of all three figures. 

\item We can also observe from all the state trajectories (bottom right subplots) that Algorithm~\ref{algo:MAGPC} in the ACL setting is more stable than in the IL setting. For instance, in the sinusoidal case (see Fig.~\ref{fig:sin-disturb}, bottom right subplot), there are less oscillations and their magnitude is smaller in the ACL setting as expected from our theory. In particular, the state trajectory converges to a neighborhood of the goal state defined by the amplitude of the disturbance sequence. The same observation can be made in the case of the Gaussian noise disturbance where the state trajectory in the ACL setting concentrates more around the origin than in the IL setting as expected. The region of concentration is controlled by the standard deviation of the Gaussian noise disturbance sequence.  

\end{itemize}

\clearpage
\subsection{Additional experiments regarding stability}

 We run additional simulations in dimension 8 with 5 agents under adversarial disturbances, and a sensitivity analysis over several values of a stability parameter $\rho$. These track both the time-averaged total cost and the state norm over 200 iterations; see Figures~\ref{fig:mix-dist}, \ref{fig:cosine-dist} and~\ref{fig:gaussian-dist}. The main takeaway is that performance degrades smoothly as the stability margin decreases and deteriorates near instability, consistent with the role of strong stability in our analysis.

\begin{figure}[h]
    \centering
    \includegraphics[width=0.88\textwidth]{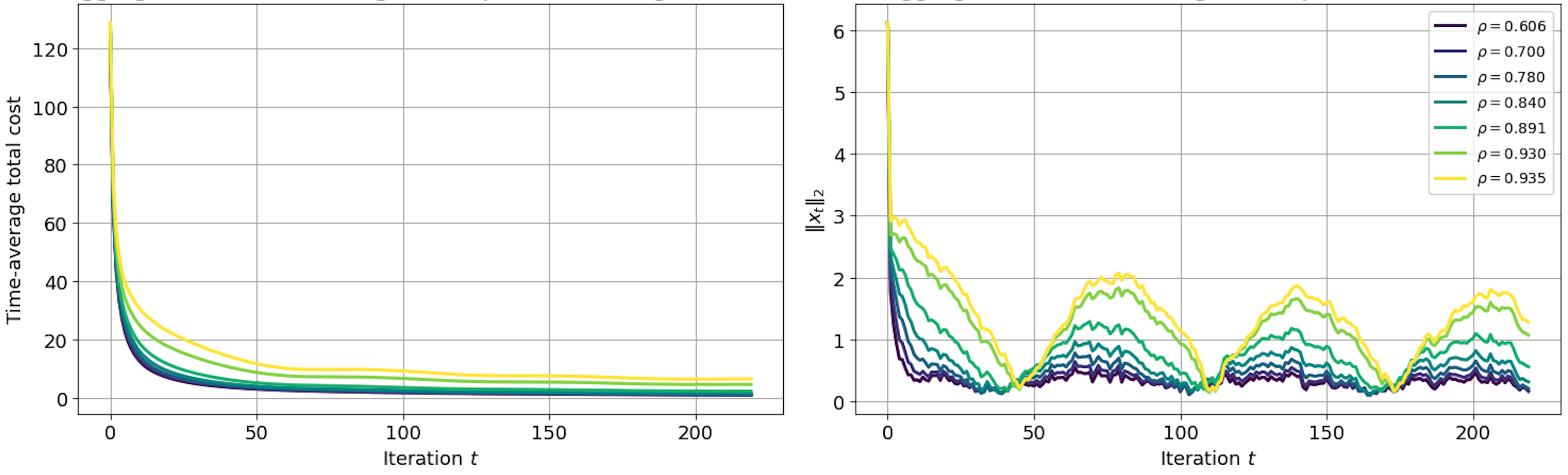}%
    \caption{Illustration of the performance of Algorithm~\ref{algo:MAGPC} on a multi-agent LDS with $N = 5, d=8$ and composite adversarial disturbance (randomly additively mixing sin, cosine and Gaussian perturbations with different amplitudes, phases and variance respectively).}
    \label{fig:mix-dist}
\end{figure}

\begin{figure}[h]
    \centering
    \includegraphics[width=0.88\textwidth]{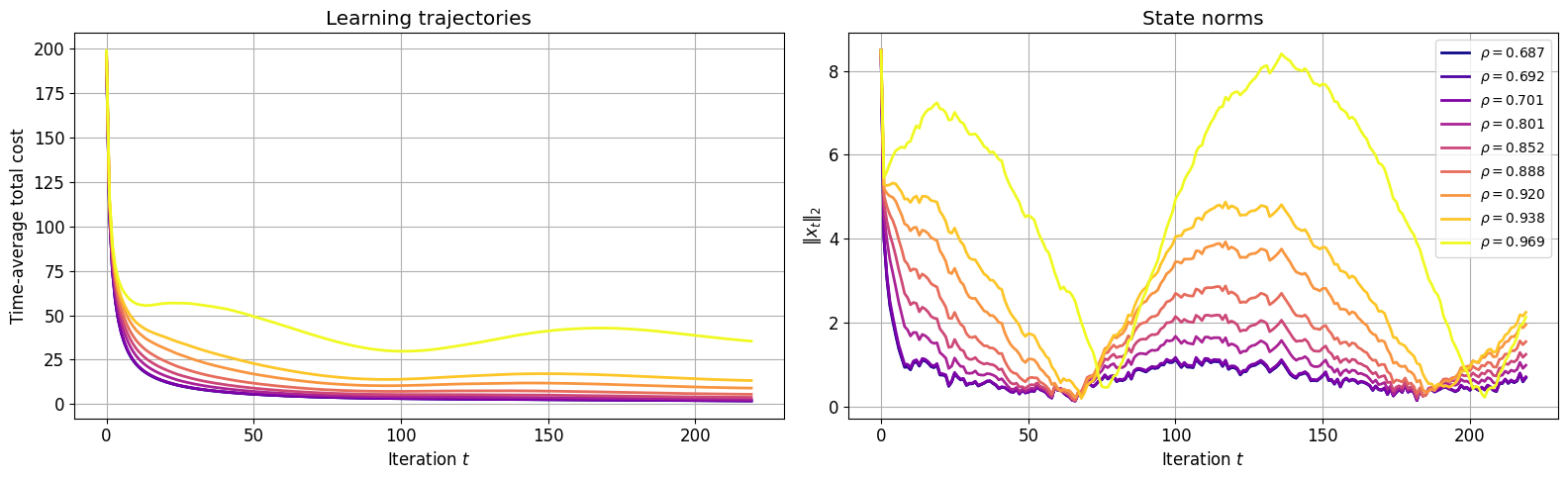}%
    \caption{Same setting with only cosine disturbance.}
    \label{fig:cosine-dist}
\end{figure}

\begin{figure}[h]
    \centering
    \includegraphics[width=0.88\textwidth]{rebuttal-fig3-updated.pdf}%
    \caption{Same setting with only Gaussian disturbances.}
    \label{fig:gaussian-dist}
\end{figure}

\end{document}